\documentclass{article}

\usepackage[accepted]{icml2021}

\usepackage{microtype}
\usepackage{graphicx}
\usepackage{subfigure}
\usepackage{booktabs} %

\usepackage{hyperref}

\usepackage{epsf}
\usepackage{graphicx}
\usepackage{amsmath}
\usepackage{amssymb}
\usepackage{algorithmic}
\usepackage{algorithm}
\usepackage{times}
\usepackage{mathtools}
\usepackage{microtype}
\usepackage{graphicx}
\usepackage{subfigure}
\usepackage{booktabs} %
\usepackage{natbib}

\usepackage{mathrsfs}
\usepackage[algo2e]{algorithm2e} 
\usepackage{algorithm,algorithmic}
\usepackage{enumitem}

\usepackage{hyperref}
\usepackage{cleveref}

\usepackage{amsmath,amsfonts,amsthm}
\usepackage{thmtools,thm-restate}

\newtheorem{theorem}{Theorem}
\newtheorem{lemma}{Lemma}

\newtheorem{assumption}{Assumption}

\newcommand{\BoldLambda}{\boldsymbol{\lambda}}
\newcommand{\BoldMu}{\boldsymbol{\mu}}
\newcommand{\Boldv}{\mathbf{v}}
\newcommand{\Boldz}{\mathbf{z}}
\newcommand{\Boldu}{\mathbf{u}}
\newcommand{\Boldx}{\mathbf{x}}
\newcommand{\Boldy}{\mathbf{y}}
\newcommand{\Boldw}{\mathbf{w}}
\newcommand{\BoldP}{\mathbf{P}}
\newcommand{\Boldb}{\mathbf{b}}

\newcommand{\Boldr}{\mathbf{r}}

\newcommand{\Boldq}{\mathbf{q}}

\newcommand{\CalS}{\mathcal{S}}
\newcommand{\CalA}{\mathcal{A}}

\DeclareMathAlphabet{\mathsfit}{\encodingdefault}{\sfdefault}{m}{sl}
\SetMathAlphabet{\mathsfit}{bold}{\encodingdefault}{\sfdefault}{bx}{n}

\newcommand{\E}{\mathbb{E}}

\newcommand{\R}{\mathbb{R}}

\newcommand{\softmax}{\mathrm{softmax}}
\newcommand{\lse}{\mathrm{logsumexp}}

\DeclareMathOperator*{\argmax}{arg\,max}
\DeclareMathOperator*{\argmin}{arg\,min}

\newcommand{\st}{\emph{s.t.}}

\newcommand{\ie}{\emph{i.e.}}

\renewcommand{\eqref}[1]{Eq. (\ref{#1})}

\def\Trans{\mathcal{T}}

\def\st{~\text{s.t.}~}

\def\softmax{\mathrm{softmax}}

\newtheorem{corollary}{Corollary}
\newtheorem{definition}{Definition}

\newtheorem{remark}{Remark}

\theoremstyle{definition}

\newcommand{\aldotwo}[1]{\textcolor{red}{}}

\icmltitlerunning{Near Optimal Policy Optimization via REPS}

\begin{document}

\twocolumn[
\icmltitle{Near Optimal Policy Optimization via REPS}

\begin{icmlauthorlist}
\icmlauthor{Aldo Pacchiano}{berkeley}
\icmlauthor{Jonathan Lee}{stanford}
\icmlauthor{Peter Bartlett}{berkeley}
\icmlauthor{Ofir Nachum}{google}
\end{icmlauthorlist}

\icmlaffiliation{stanford}{Stanford}
\icmlaffiliation{google}{Google Research}
\icmlaffiliation{berkeley}{UC Berkeley}

\icmlcorrespondingauthor{Aldo Pacchiano}{pacchiano@berkeley.edu}

\icmlkeywords{Machine Learning, ICML}

\vskip 0.3in
]

\printAffiliationsAndNotice{} %

\begin{abstract}
Since its introduction a decade ago, \emph{relative entropy policy search} (REPS) has demonstrated successful policy learning on a number of simulated and real-world robotic domains, not to mention providing algorithmic components used by many recently proposed reinforcement learning (RL) algorithms. While REPS is commonly known in the community, there exist no guarantees on its performance when using stochastic and gradient-based solvers. In this paper we aim to fill this gap by providing guarantees and convergence rates for the sub-optimality of a policy learned using first-order optimization methods applied to the REPS objective. We first consider the setting in which we are given access to exact gradients and demonstrate how near-optimality of the objective translates to near-optimality of the policy. 
We then consider the practical setting of stochastic gradients, and introduce a technique that uses \emph{generative} access to the underlying Markov decision process to compute parameter updates that maintain favorable convergence to the optimal regularized policy.
\end{abstract}

\section{Introduction}
\label{sec:intro}
Introduced by~\citet{peters2010relative}, \emph{relative entropy policy search} (REPS) is an algorithm for learning agent policies in a reinforcement learning (RL) context. 
REPS has demonstrated successful policy learning in a variety of challenging simulated and real-world robotic tasks, encompassing table tennis~\citep{peters2010relative}, tether ball~\citep{daniel2012hierarchical}, beer pong~\citep{abdolmaleki2015model}, and ball-in-a-cup~\citep{boularias2011relative}, among others.
Beyond these direct applications of REPS, the mathematical tools and algorithmic components underlying REPS have inspired and been utilized as a foundation for a number of later algorithms, with their own collection of practical successes~\citep{fox2017taming,schulman2015trust,nachum2017bridging,neu2017unified,haarnoja2018soft,abdolmaleki2018maximum,kostrikov2019imitation,nachum2019algaedice}.

At its core, the REPS algorithm is derived via an application of convex duality~\citep{neu2017unified,nachum2020reinforcement}, in which a  Kullback Leibler (KL)-regularized version of the max-return objective in terms of state-action distributions is transformed into an $\lse$ objective in terms of state-action \emph{advantages} (\ie, the difference of the value of the state-action pair compared to the value of the state alone, with respect to some learned state value function).
If this dual objective is optimized, then the optimal policy of the original primal problem may be derived as a $\softmax$ of the state-action advantages.
This basic derivation may be generalized, using any number of entropic regularizers on the original primal to yield a dual problem in the form of a convex function of advantages, whose optimal values may be transformed back to optimal regularized policies~\citep{belousov2017f}.

While the motivation for the REPS objective through the lens of convex duality is attractive, it leaves two main questions unanswered regarding the theoretical soundness of using such an approach.
First, in practice, the dual objective in terms of advantages is likely not optimized fully. Rather, standard gradient-based solvers only provide guarantees on the \emph{near-optimality} of a returned candidate solution.
While convex duality asserts a relationship between primal and dual variables at the \emph{exact} optimum, it is far from clear whether a near-optimal dual solution will be guaranteed to yield a near-optimal primal solution, and this is further complicated by the fact that the primal candidate solution must be transformed to yield an agent policy.

The second of the two main practical difficulties is due to the form of the dual objective. Specifically, the form of the dual objective as a convex function of advantages frustrates the use of gradient-based solvers in stochastic settings. That is, the advantage of a state-action pair consists of an expectation over next states -- an expectation over the transition function associated with the underlying Markov decision process (MDP).
In practical settings, one does not have explicit knowledge of this transition function. Rather, one only has access to stochastic samples from this transition function, and so calculation of unbiased gradients of the REPS objective is not directly feasible. 

In this paper, we provide solutions to these two main difficulties. 
To the first issue, we present guarantees on the near-optimality of a derived policy from dual variables optimized via a first-order gradient method, relying on a key property of the REPS objective that ensures near-optimality in terms of gradient norms. To the second issue, we propose and analyze a stochastic gradient descent procedure that makes use of a plug-in estimator of the REPS objective gradients. 
Under some mild assumptions on the MDP, our estimators need only sample transitions from a behavior policy rather than full access to a generative model (where one can uniformly sample transitions).
We combine these results to yield high-probability convergence rates of REPS to a near-optimal policy. In this way, we show that REPS enjoy not only favorable practical performance but also strong theoretical guarantees.
\section{Related Work}

As REPS is a popular and influential work, there exist a number of previous papers that have studied its performance guarantees.
These previous works predominantly study REPS as an iterative algorithm, where each step comprises of an exact optimization of the REPS objective and then the derived policy is used as the reference distribution for the KL regularization of the next step.
This iterative scheme may be interpreted as a form of mirror descent or similar proximal algorithms~\citep{beck2003mirror}, and this interpretation can provide guarantees on convergence to a near-optimal policy~\citep{zimin2013online,neu2017unified}.
However, because this approach assumes the ability to optimize the REPS objective exactly, it still suffers from the practical limitations discussed above; specifically (1) translation of near-optimality of advantages to near-optimality of the policy and (2) ability to compute unbiased gradients when one does not have explicit knowledge of the MDP dynamics. Our analysis attacks these issues head-on, providing guarantees on first-order optimization methods applied to the REPS objective. To maintain focus we do not consider iterative application of REPS, although extending our guarantees to the iterative setting is a promising direction for future research.  %

In a somewhat related vein, a number of works use REPS-inspired derivations to yield \emph{dynamic programming} algorithms~\citep{fox2017taming,geist2019theory,vieillard2020leverage} and subsequently provide guarantees on the convergence of approximate dynamic programming in these settings.
Our results focus on the use of REPS in a \emph{convex programming} context, and optimizing these programs via standard gradient-based solvers.

The use of convex programming for RL in this way has recently received considerable interest.
Works in this area typically propose to learn near-optimal policies through \emph{saddle-point} optimization~\citep{chen2016stochastic,wang2017randomized,chen2018scalable,bas2019faster,cheng2020reduction,jin2020efficiently}.
Rather than solving either the primal or dual max-return problem directly, these works optimize the Lagrangian in the form of a $\min$-$\max$ bilinear problem.
The Lagrangian form helps to mitigate the two main issues we identify with advantage learning, since (1) the candidate primal solution can be use to derive a policy in a significantly more direct fashion than using the candidate dual solution, and (2) the bilinear form of the Lagrangian is immediately amenable to stochastic gradient computation. 
In contrast to these works, our analysis focuses on learning exclusively in the dual (advantage) space. 
The first part of our results is most comparable to the work of \cite{bas2019faster}, which proposes a saddle-point optimization with runtime $O(1/\epsilon)$, assuming access to known dynamics. While our results yield a $O(1/\epsilon^2)$ rate, we show that it can be achieved via optimizing the dual objective alone. 

More similar to our work is the analysis of~\citet{basserrano2020logistic}, which considers an objective similar to REPS, but which is in terms of $Q$-values as opposed to state ($V$) values. Beyond these structural differences, our proof techniques also differ. For example, our result on the suboptimality of the policy derived from dual variables (Lemma~\ref{lemma::bounding_primal_value_candidate_solution}), is arguably simpler from the analogous result in~\citet{basserrano2020logistic}, which uses a two-step process to first connect suboptimality of the dual variables to constraint violation of the primal variables, and then connects this to suboptimality of the policy.

\section{Contributions}
The main contributions of this paper are the following:%
\begin{enumerate}
    \item We prove several structural results regarding entropy regularized objectives for reinforcement learning  and leverage them to prove convergence guarantees for Accelerated Gradient Descent on the dual (REPS) objective under mild assumptions on the MDP (see Theorem~\ref{theorem::main_accelerated_result}). For discounted MDPs we show that an $\epsilon$-optimal policy can be found after $\mathcal{O}(1/(1-\gamma)^2\epsilon^2)$ steps and an $\epsilon-$optimal regularized policy can be found in $\mathcal{O}(1/(1-\gamma)^2\epsilon)$ steps. 
    
    \item Similarly we show that a simple version of stochastic gradient descent using biased plug-in gradient estimators can be used to find an $\epsilon-$optimal policy after $\mathcal{O}(1/(1-\gamma)^8\epsilon^8)$ iterations (see Theorem~\ref{theorem::main_result}) and an $\epsilon$-optimal regularized policy  in $\mathcal{O}(1/(1-\gamma)^8\epsilon^4)$ steps. Although our rates are short of the ones achievable by alternating optimization methods, we are the first to show meaningful convergence guarantees for a purely dual approach based on on-policy access to samples from the underlying MDP. 
    \item In Appendix~\ref{section::extended_results_tsallis} we extend our results beyond the REPS objective and consider the use of Tsallis Entropy regularizers. Similar to our results for the REPS objective we show that for discounted MDPs an $\epsilon-$optimal policy can be found after $\mathcal{O}(1/(1-\gamma)^2\epsilon^2 )$ steps and an $\epsilon-$optimal regularized policy can be found in $\mathcal{O}(1/(1-\gamma)^2\epsilon)$ steps.   

\end{enumerate}

\section{Background}

In this section we review the basics of Markov decision processes and their Linear Programming primal and dual formulations (see section~\ref{section::RL_as_LP}) and some facts about the geometry of convex functions. 

\subsection{RL as an LP}\label{section::RL_as_LP}

We consider a discounted Markov decision process (MDP) described by a tuple $\mathcal{M} = (\mathcal{S}, \mathcal{A}, P, \Boldr, \BoldMu, \gamma)$, where $\mathcal{S}$ is a finite state space, $\mathcal{A}$ is a finite action space, $P$ is a transition probability matrix, $\Boldr$ is a reward vector, $\BoldMu$ is an initial state distribution, and $\gamma \in (0,1)$ is a discount factor. 
We make the following assumption regarding the reward values $\{ \Boldr_{s,a} \}$. 
\begin{assumption}[Unit rewards]\label{assumption::bounded_rewards}
For all $s,a \in \CalS \times \CalA$, the rewards satisfy,
\begin{equation*}
    \Boldr_{s,a}\in[0,1].
\end{equation*}
\vspace{-.5cm}
\aldotwo{We can make this assumption more general.}
\end{assumption}
The agent interacts with $\mathcal{M}$ via a policy $\pi : \CalS \rightarrow \Delta_{\CalA}$. The agent is initialized at a state $s_0$ sampled from an initial state distribution $\BoldMu$ and at time $k=0,1,\dots$ it uses its policy to sample an action $a_k\sim \pi(s_k)$. The MDP provides an immediate reward $\Boldr_{s_k,a_k}$ and transitions randomly to a next state $s_{k+1}$ according to probabilities $\mathbf{P}_a(s_{k+1}|s_k)$.
Given a policy $\pi$ we define its infinite-horizon discounted reward as:
\begin{equation*}
 V_\pi :=   \mathbb{E}^\pi\left[\sum_{k=0}^\infty \gamma^k \Boldr_{s_k, a_k} \right],
\end{equation*}
where we use $\mathbb{E}^{\pi}$ to denote the expectation over trajectories induced by the MDP $\mathcal{M}$ and policy $\pi$. In RL, the agent's objective is to find an optimal policy $\pi_\star$; that is, find a policy maximizing $V_\pi$ over all policy mappings $\pi : \mathcal{S} \rightarrow \Delta_{\mathcal{A}}$. We denote the optimal policy as:
\begin{equation*}
    \pi_\star := \argmax_{\pi} V_\pi.
\end{equation*}

We now review the definitions of state value functions:
\begin{definition}
We define the value vector $\mathbf{v}^\pi \in \mathbb{R}^{|\CalS|}$ of a policy $\pi$ as:
\begin{equation*}
    \mathbf{v}^\pi_s := \mathbb{E}^{\pi}\left[\sum_{k=0}^\infty \gamma^k r_{s_k, a_k} | s_0 = s \right].
\end{equation*}
\end{definition}

We now review the definition of visitation distributions:
\begin{definition}
Given a policy $\pi$ we define its state-action visitation distribution $\BoldLambda^{\pi} \in \mathbb{R}^{|\CalS|\times |\CalA|}$ as,
\begin{equation*}
    \BoldLambda^{\pi}_{s,a} := (1-\gamma)\mathbb{E}^{\pi}\left[\sum_{k=0}^\infty \gamma^k \mathbf{1}(s_k = s, a_k = a) \right].
\end{equation*}
Notice that by definition $\sum_{s,a} \BoldLambda_{s,a} = 1$. 
\end{definition}
We note that any vector of nonnegative entries $\BoldLambda$ may be used to define a policy $\pi_{\BoldLambda}$ as:
\begin{equation}\label{equation::lambda_to_policy}
    \pi_{\BoldLambda}(a | s) := \frac{\BoldLambda_{s, a}}{\sum_{a' \in \mathcal{A}} \BoldLambda_{s, a'} } .
\end{equation}
Note that $\pi_{\BoldLambda^\pi} = \pi$, while the visitation distribution $\BoldLambda^{\pi_{\BoldLambda}}$ of $\pi_{\BoldLambda}$ is not necessarily $\BoldLambda$ for an arbitrary vector $\BoldLambda$. %
\begin{definition}
Given a policy $\pi$ we define its state visitation distribution as,
\begin{equation*}
    \BoldLambda^{\pi}_{s} := (1-\gamma)\mathbb{E}^{\pi}\left[\sum_{k=0}^\infty \gamma^k \mathbf{1}(s_k = s) \right].
\end{equation*}
Notice that $\BoldLambda^{\pi}_s = \sum_a \BoldLambda^{\pi}_{s,a}$ and $\BoldLambda^\pi_{s,a} = \BoldLambda^\pi_s \cdot \pi(a|s)$. 
\end{definition}

The optimal visitation distribution $\BoldLambda^*$ is defined as
\begin{equation*}
    \BoldLambda^* := \argmax_{\BoldLambda^\pi} \sum_{s,a} \BoldLambda^{\pi}_{s,a}\Boldr_{s,a} .
\end{equation*}
It can be shown \citep{puterman2014markov,chen2016stochastic} that solving for the optimal visitation distribution is equivalent to the following linear program:
\begin{align}\label{equation::primal_visitation_unregularized}
    & \max_{ \BoldLambda_{s,a} \in \Delta_{\CalS\times \CalA}} ~  \sum_{s, a} \BoldLambda_{s, a} \Boldr_{s,a} \tag{Primal-$\BoldLambda$}\\
    & \st \sum_{a} \BoldLambda_{s, a} = (1 - \gamma) \BoldMu_s + \gamma \sum_{s', a} \BoldP_a(s | s') \BoldLambda_{s', a} 
    \quad \forall s\in \CalS. \notag
\end{align}
Where we write $\BoldP \in \mathbb{R}^{|S||A| \times |S| }$ to denote the transition operator. Specifically, the $|\CalS|$ constraints of~\ref{equation::primal_visitation_unregularized} restrict any feasible $\BoldLambda$ to be the state-action visitations for some policy $\pi$ (given by $\pi_{\BoldLambda}$). The dual of this LP is given by,
\begin{align}\label{equation::dual_value_unregularized}
\min_{\Boldv}~& (1-\gamma) \sum_{s \in \mathcal{S}} \BoldMu_s \Boldv_s \tag{Dual-$\Boldv$} \\
\st & ~0 \geq \mathbf{A}^\Boldv_{s,a} \quad \forall s\in \CalS, a\in \CalA, \notag
\end{align}
where $\mathbf{A}^\Boldv_{s,a} = \Boldr_{s,a} - \Boldv_s + \gamma \sum_{s'} \mathbf{P}_a(s'|s)\Boldv_{s'}$ is the advantage evaluated at $s,a \in \mathcal{S} \times \mathcal{A}$. It can be shown~\citep{puterman2014markov,chen2016stochastic} that the unique primal solution $\BoldLambda^*$ is exactly $\BoldLambda^{\pi_*}$ and the unique dual solution $\Boldv^*$ is $\Boldv^{\pi_*}$.

We finalize this section by defining the notion of suboptimality satisfied by the final policy produced by the algorithms that we propose. 

\begin{definition}
Let $\epsilon > 0$. We say that policy $\pi$ is $\epsilon$-optimal if
\begin{equation*}
    \max_{s \in \mathcal{S}} |  \Boldv_s^{\pi} - \Boldv_s^{\pi_\star}   | \leq \epsilon.
\end{equation*}
\end{definition}

 Our objective is to design algorithms such that for any parameter $\epsilon > 0$, can return an $\epsilon-$optimal policy.

\subsection{Regularized Policy Search}

Following~\citet{belousov2017f}, we consider regularizing \ref{equation::primal_visitation_unregularized} with 
a convex function $F:\Delta_{|\CalS|\times|\CalA|}\to \R\cup\{\infty\}$.
The resulting regularized LP is given by,
\begin{align}\label{equation::primal_visitation_regularized}
         & \max_{\BoldLambda_{s,a} \in \Delta_{\CalS\times \CalA}} ~  \sum_{s, a} \BoldLambda_{s, a} \Boldr_{s,a} - F(\BoldLambda) \tag{PrimalReg-$\BoldLambda$} := J_P(\BoldLambda) \\
         & \st  \sum_{a} \BoldLambda_{s, a} = (1 - \gamma) \BoldMu_s + \gamma \sum_{s', a} \BoldP_a(s | s') \BoldLambda_{s', a}  \quad \forall s\in \CalS. \notag
\end{align}
Henceforth we denote the primal objective function as $J_P(\BoldLambda) = \sum_{s, a} \BoldLambda_{s, a} \Boldr_{s,a} - F(\BoldLambda) $. Note that any feasible $\BoldLambda$ that satisfies the $|\CalS|$ constraints in this regularized LP is the (true) state-action visitation distribution for some policy $\pi$; therefore, the optimal $\BoldLambda^*$ of this problem can be used to derive an optimal $F$-regularized max-return policy $\pi_{F,*} := \pi_{\BoldLambda^*}$. To simplify the subsequent derivations, we introduce the definition of the convex conjugate of a convex function, oftentimes referred to as the Fenchel conjugate:

\begin{definition}[Fenchel Conjugate]
Let $F: \mathcal{D} \rightarrow \mathbb{R}$ be a convex function over a convex domain $\mathcal{D}\subseteq \mathbb{R}^d$. We denote its $\mathcal{D}-$constrained Fenchel conjugate as $F^*: \mathbb{R}^n \rightarrow \mathbb{R}$  defined as:
\begin{equation*}
    F^*(\Boldu) = \max_{\Boldx\in \mathcal{D}} ~\langle \Boldx, \Boldu \rangle - F(\Boldx).
\end{equation*}
\end{definition}

The dual $J_D$ of the regularized problem is given by the following optimization problem~\citep{belousov2017f,nachum2020reinforcement}:
\begin{equation}
    \min_{\Boldv} J_D(\Boldv) := (1 - \gamma) \sum_{s} \Boldv_s \BoldMu_s + F^*\left(\mathbf{A}^\Boldv\right), \label{equation::dual_function_RL}
\end{equation}
where $F^*$ is the $\Delta_{\CalS\times \CalA}$-constrained Fenchel conjugate of $F$.
The vector quantity inside $F^*$ is known as the \emph{advantage}. That is, it quantifies the advantage (the difference in estimated value) of taking an action $a$ at $s$, with respect to some state value function $\Boldv$.

Using Fenchel-Rockafellar duality, the optimal solution $\Boldv^*$ of the dual function $J_D$ may be used to derive an optimal primal solution $\BoldLambda^*$ as: 
\begin{equation}
    \label{eq:lambda-star}
    \BoldLambda^* \in \nabla F^*\left( \mathbf{A}^{\Boldv^\star} \right).
\end{equation}

\begin{algorithm}[H]
\textbf{Input: } Initial iterate $\Boldv_0$, accuracy level $\epsilon > 0$, gradient optimization algorithm $\mathcal{O}$. 
\begin{enumerate}
    \item Optimize the objective in~\ref{equation::dual_function_RL} using $\mathcal{O}$ to yield a candidate dual solution $\hat{\Boldv}^*$ where $F$ satisfies Equation~\ref{equation::definition_KL_F}.
    \item Use the candidate dual solution to derive a candidate primal solution $\hat{\BoldLambda}^*$ using~\ref{eq:lambda-star}.
    \item Extract a candidate policy $\pi_{\hat{\BoldLambda}^*}$ from $\hat{\BoldLambda}^*$ via Equation~\ref{equation::lambda_to_policy}. 
\end{enumerate}
\textbf{Return:} $\pi_{\hat{\BoldLambda}^*}$.
\caption{Relative Entropy Policy Search [Sketch].}
\label{alg:advantage_learning}
\end{algorithm}

Relative Entropy Policy Search (REPS) is derived by setting $F(\BoldLambda):= D_{\mathrm{KL}}(\BoldLambda \| \Boldq)$, the KL-divergence of $\BoldLambda$ from some reference distribution $\Boldq \in \Delta_{| \CalS | |\CalA|}$. The reader should think of $\Boldq$ as the visitation distribution of a behavior policy. As we can see, the derivation we provide here further generalizes to arbitrary regularizers $F$. We focus on a specific $F$ given by 
\begin{equation}\label{equation::definition_KL_F}
    F(\BoldLambda) := \frac{1}{\eta}\sum_{s,a} \BoldLambda_{s,a} \left(\log\left(\frac{\BoldLambda_{s,a}}{\Boldq_{s,a}}\right) - 1\right),
\end{equation}
for some scalar $\eta>0$. In this case $F^* : \mathbb{R}^{|S|\times |A|}\rightarrow \mathbb{R}$ equals: 
\begin{equation*}
    F^*(\mathbf{u}) = \frac{1}{\eta}\log\left(\sum_{s,a} \exp\left( \eta \mathbf{u}_{s,a} \right)
\Boldq_{s,a} \right) + \frac{1}{\eta}.
\end{equation*}
\begin{equation*}
    \left[ \nabla F^*(\mathbf{u}) \right]_{s,a} = \frac{ \exp(\eta \mathbf{u}_{s,a}) \Boldq_{s,a}}{\sum_{s', a'} \exp(\eta \mathbf{u}_{s',a'}) \Boldq_{s',a'}   }. 
\end{equation*}
And therefore the dual function equals:
\begin{align}\label{equation::dual_visitation_regularized}\tag{DualReg-$\Boldv$}
    & J_D(\Boldv) :=
     (1 - \gamma) \sum_{s} \Boldv_s \BoldMu_s  \\
     & + \frac{1}{\eta} \log\left( \sum_{s,a} \exp\left( \eta  \mathbf{A}^{\Boldv}_{s,a}  \right)\Boldq_{s,a} \right) + \frac{1}{\eta}, 
\end{align}
And the dual problem equals the unconstrained minimization problem:
\begin{equation}\label{equation::dual_problem_RL}
 \min_{\Boldv} J_D(\Boldv)
\end{equation}
The objective of REPS is to find the minimizer $\Boldv^\star$ of~\ref{equation::dual_visitation_regularized} (with regularization level $\eta$). 

Algorithm~\ref{alg:advantage_learning} raises two practical issues discussed in Section~\ref{sec:intro}. Specifically, optimization algorithms applied to REPS will typically only give guarantees on the near-optimality of $\hat{\Boldv}^*$. We will need to translate near-optimality of $\hat{\Boldv}^*$ to near-primal-optimality (w.r.t. $J_P(\BoldLambda^\star)$) of $\hat{\BoldLambda}^*$, and then translate that to near-optimality of the final returned policy $\pi_{\hat{\BoldLambda}^*}$.
Secondly, first-order optimization of the REPS objective requires access to a gradient $\nabla_{\Boldv} J_D(\Boldv)$, which involved computing $\nabla F^*(\mathbf{A}^\Boldv)$. Exact computation of this quantity is often infeasible in practical scenarios where one does not have access to $\BoldP$, but rather only stochastic \emph{generative} access to samples from $\BoldP$. We show how to compute approximate (biased) gradients of $J_D(\Boldv)$ using samples from a distribution $\Boldq_{s,a}$ (here thought of as a behavior policy) and how to use them to derive convergence rates for Relative Entropy Policy Search.

\section{Relative Entropy Policy Search}
We start by deriving some general results regarding the geometry of regularized linear programs. Our first result (Lemma~\ref{lemma::dual_smoothness_regularized_LP}) characterizes the smoothness properties of a regularized LP.  This will prove crucial in later sections where we make use of this result to derive convergence rates for the REPS objective. We start by recalling the definitions of both strong convexity and smoothness of a function.  %
\begin{definition}\label{definition::smoothness}
A function $f: \mathbb{R}^n \rightarrow \mathbb{R}$ is $\beta-$strong convex w.r.t norm $\| \cdot \|$ if:
\begin{equation*}
f(\Boldx) \geq f(\Boldy) + \langle \nabla f(\Boldy), \Boldx-\Boldy \rangle + \frac{\beta}{2}\| \Boldx - \Boldy \|^2 
\end{equation*}
\end{definition}
\vspace{-.5cm}
Let's also define smoothness:
\begin{definition}\label{definition::strong_convexity}
A function $h$ is $\alpha-$smooth\footnote{Smoothness is independent of the convexity properties of $h$.} w.r.t. norm $\| \cdot \|_*$ if:
\begin{equation}\label{equation::definition_smoothness}
    h(\Boldu) \leq h(\Boldw) + \langle \nabla h(\Boldw), \Boldu-\Boldw \rangle + \frac{\alpha}{2} \| \Boldu-\Boldw \|_*^2
\end{equation}
\end{definition}
\vspace{-.3cm}
We will now characterize the smoothness properties of the dual of a regularized linear program. Let's start by considering the generic linear program:
\begin{align*}
    \max_{ \BoldLambda \in \mathcal{D} }  ~&\langle \Boldr, \BoldLambda \rangle, \quad 
   \text{s.t. } \mathbf{E}\BoldLambda = \mathbf{b},
\end{align*}
 where $\mathbf{r}\in\R^n$, $\mathbf{E}\in\R^{m\times n}$, and $\mathbf{b}\in\R^m$ and $\mathcal{D}$ is a convex domain.
Let's regularize this objective using a function $F$ that is $\beta$-strongly convex with respect to norm $\|\cdot\|$:
\begin{align}
    \max_{\BoldLambda \in \mathcal{D}}~&\langle \Boldr, \BoldLambda \rangle - F(\BoldLambda), \quad
    \text{s.t. }  \mathbf{E}\BoldLambda = \mathbf{b}.\label{equation::regularized_LP} \tag{RegLP}
\end{align}
The Lagrangian of problem \ref{equation::regularized_LP} is given by
\begin{equation*}
    g_L(\BoldLambda, \Boldv) = \langle \Boldr, \BoldLambda \rangle -  F(\BoldLambda ) + \sum_{i=1}^m \Boldv_i\left(   \Boldb_i - (\mathbf{E}\BoldLambda)_i\right).
\end{equation*}
Therefore, the dual function $g_D : \mathbb{R}^m \rightarrow \mathbb{R}$ with respect to the original primal regularized LP is,
\begin{align*}
    g_D(\Boldv) &= \langle \Boldv, \Boldb \rangle + \max_{\BoldLambda \in \mathcal{D}} \langle \BoldLambda, \Boldr- \Boldv^\top \mathbf{E} \rangle - F(\BoldLambda) \\
    &= \langle \Boldv, \Boldb \rangle + F^*( \Boldr- \Boldv^\top \mathbf{E} ),
\end{align*}
where the last equality follows from the definition of the Fenchel conjugate of $F$. It is possible to relate the smoothness properties of $F^*$ with the strong convexity of $F$. A crucial result that we will use in our results is the following:
\begin{restatable}{lemma}{equivalencesmoothnessstrongconvexity}
\label{lemma::equivalence_smoothness_strong_convexity}
If $F$ is $\beta$-strongly convex  w.r.t. norm $\| \cdot \|$ over $\mathcal{D}$ then $F^*$ is $\frac{1}{\beta}$-smooth w.r.t the dual norm $\| \cdot \|_*$.
\end{restatable}
The proof of this lemma is in Appendix~\ref{section::appendix_geometry_regularized_linear}. Definitions~\ref{definition::smoothness} and~\ref{definition::strong_convexity} are stated in terms of a generic norm $\| \cdot \|$ and its dual $\| \cdot \|_\star$. When applied to the REPS objective in Equation~\ref{equation::dual_function_RL}, using these general norm definitions of smoothness and strong convexity allow us to obtain guarantees with a milder dependence on $\mathcal{S}$ and $\mathcal{A}$ than would be possible if we were to use their $\ell_2$ norm characterization instead. We can use the result of Lemma~\ref{lemma::equivalence_smoothness_strong_convexity} to characterize the smoothness properties of the dual function $J_D$ of a generic regularized LP.
\begin{restatable}{lemma}{dualsmoothnessregularizedLP}
\label{lemma::dual_smoothness_regularized_LP}
Consider the regularized LP~\ref{equation::regularized_LP} with $\Boldr \in \mathbb{R}^n$, $\mathbf{E} \in \mathbb{R}^{m\times n}$, $\Boldb\in\mathbb{R}^m$, and where $F$ is $\beta-$strongly convex w.r.t. norm $\| \cdot \|$. The dual function $g_D:\mathbb{R}^m \rightarrow \mathbb{R}$ of this regularized LP
is $\frac{\| \mathbf{E} \|^2_{\cdot, *}}{\beta}$-smooth w.r.t. to the dual norm $\| \cdot \|_*$, where we use $\| \mathbf{E} \|_{\cdot, *} $ to denote the $\| \cdot \|$ norm over the $\| \cdot \|_*$ norm of $\mathbf{E}'$s rows. 
\end{restatable}
As a simple consequence of~Lemma~\ref{lemma::dual_smoothness_regularized_LP} we can characterize the smoothness parameter of $J_D$ in the REPS objective:

\begin{restatable}{lemma}{RLsmoothnessdual}\label{lemma::RL_smoothness_dual}
The dual function $J_D(\Boldv)$ is $(|\mathcal{S}|+1)\eta$-smooth in the $\| \cdot \|_\infty$ norm.
\aldotwo{Actually we may want to check this because we are not using just the entropy... there is the $\Boldq$ distribution.}
\end{restatable}
A detailed proof of this result can be found in Appendix~\ref{subsection::proof_of_RL_smoothness_dual}.

\subsection{Structural results for the REPS objective}\label{section::structural_REPS}

Armed with Lemma~\ref{lemma::dual_smoothness_regularized_LP} we are ready to derive some useful structural properties of the REPS objective. In this section we present two main results. First we show that under some mild assumptions it is possible to relate the gradient magnitude of any candidate solution to $J_D$ with its suboptimality gap and second, we show an $l_\infty$ bound for the norm of the optimal dual solution $\Boldv^\star$. For most of the analysis we make the following assumptions:

\begin{assumption}\label{assumption::lower_bound_q}
There is $ \beta >0$ such that:
\begin{equation*}
      \Boldq_{s,a}\geq \beta \quad \forall s, a \in \mathcal{S} \times \mathcal{A}.
\end{equation*}
\vspace{-.5cm}
\end{assumption}
We introduce the following assumption on the discounted state visitation distribution of arbitrary policies $\pi$ in the MDP, paraphrased from~\citet{wang2017primal}:
\begin{assumption}
\label{ass:uniform}
There exists $\rho>0$ such that for any policy $\pi$, the discounted state visitation distribution $\BoldLambda^\pi$ defined as $\BoldLambda_s^\pi = \sum_a \BoldLambda_{s,a}^\pi$ satisfies 
\begin{equation}
 \BoldLambda^\pi_s \geq \rho
\end{equation}
for all states $s \in \CalS$.%
\end{assumption}

Suppose we have a candidate dual solution $\widetilde{\Boldv}$ for$J_D(\Boldv)$ in~\ref{equation::dual_visitation_regularized} with its corresponding candidate primal solution $$\widetilde{\BoldLambda} = \frac{\exp\left(   \eta \mathbf{A}^{\tilde{\Boldv}} \right)\boldsymbol{\cdot}\Boldq}{\widetilde{Z}}$$ where the operators $\exp$ and $\boldsymbol{\cdot}$ act pointwise and $$\widetilde{Z} = \sum_{a, s} \exp( \eta \mathbf{A}^{\tilde{\Boldv}} ) \Boldq_{s,a}. $$ We denote the corresponding candidate policy (computed using Equation~\ref{equation::lambda_to_policy}) associated with $\widetilde{\Boldv}$ as $\widetilde{\pi}(a| s)$. This candidate policy induces a discounted visitation distribution $\BoldLambda^{\widetilde{\pi}}$ that may be substantially different from $\widetilde{\BoldLambda}$. We now show that it is possible to control the deviation of primal objective value of $\BoldLambda^{\widetilde{\pi}}$ from $J_P(\BoldLambda)$ in terms of $\| \nabla J_D(\widetilde{\Boldv})\|_1$:

\begin{restatable}{lemma}{boundingprimalvaluecandidatesolution}
\label{lemma::bounding_primal_value_candidate_solution}
 Let $\widetilde{\Boldv} \in \mathbb{R}^{|\CalS|}$ be arbitrary and let $\widetilde{\BoldLambda}$ be its corresponding candidate primal variable. If $\|\nabla_\Boldv J_D(\widetilde{\Boldv})\|_1 \leq \epsilon$ and Assumptions~\ref{assumption::lower_bound_q} and~\ref{ass:uniform} hold then whenever $|\CalS| \geq 2$:
\begin{equation*}
      J_P( \BoldLambda^{\widetilde{\pi}} ) \geq J_P(\BoldLambda_\eta^\star) -\epsilon \left( \frac{1 +c }{1-\gamma}   + \| \widetilde{\Boldv } \|_\infty \right),
\end{equation*}
where $c =  \frac{ 1 + \log(\frac{1}{\rho^3\beta})}{\eta} $ and $\BoldLambda_\eta^{\star}$ is the $J_P$ optimum. 
\end{restatable}

The proof of Lemma~\ref{lemma::bounding_primal_value_candidate_solution} is in Appendix~\ref{section::proof_of_primal_value_candidate_solution}.

We finish this section by proving a bound on the norm of the dual variables. This bound will inform our optimization algorithms as it will allow us to set up the right constraints.

\begin{restatable}{lemma}{dualvariablesboundofir}
\label{lemma::dual_variables_bound_ofir}
Under Assumptions~\ref{assumption::bounded_rewards},  \ref{assumption::lower_bound_q} and~\ref{ass:uniform}, the optimal dual variables are bounded as
\begin{equation}\label{equation::max_radius}
\| \Boldv^*  \|_\infty \le \frac{1}{1-\gamma} \left( 1 + \frac{\log\frac{|S||A| }{\beta \rho}}{\eta} \right):= D.
\end{equation}
\end{restatable}

The proof of Lemma~\ref{lemma::dual_variables_bound_ofir} can be found in Appendix~\ref{section::dual_variables_bound_ofir}. From now on we use the notation $D$ to refer to the quantity on the RHS of Equation~\ref{equation::max_radius}.  

\subsection{Convergence rates}

As a warm up in this section we derive convergence rates for the case when we have access to exact knowledge of the transition dynamics $\mathbf{P}$ and therefore exact gradients. We analyze the effects of running Accelerated Gradient Descent on the REPS objective $J_D(\Boldv)$.  First we require to define a distance generating function:

\begin{definition}[Distance generating function] We say that $w : \mathcal{D} \rightarrow \mathbb{R}$ is a distance generating function (DGF) if $w$ is $1-$strongly convex w.r.t to the $\| \cdot \|_\star$ norm. Accordingly, the Bregman divergence is given as:
$$D_w(\Boldx, \Boldy) = w(\Boldy) - \langle \nabla w(\Boldx), \Boldy-\Boldx \rangle - w(\Boldx), \quad \Boldx \in \mathcal{D}, \forall \Boldy \in \mathcal{D}$$ 
The strong convexity of $w$ implies that $D_w$ satisfies $D_w(\Boldx,\Boldx) = 0$ and $D_w(\Boldx, \Boldy) \geq \frac{1}{2}\| \Boldx-\Boldy \|_\star^2\geq 0$. 
\end{definition}

\begin{algorithm}
\textbf{Input} Initial point $\Boldx_0$, domain $\mathcal{D}$, distance generating function $w$.\\
   $\Boldy_0 \leftarrow \Boldx_0, \quad \Boldz_0 \leftarrow \Boldx_0$.\\
   \For{$t=0, \cdots , T$}{
   $\eta_{t+1} = \frac{t+2}{2\alpha}$ and $\tau_t = \frac{2}{t+2}$.\\
   \begin{align*}
   \mathbf{x}_{t+1} &\leftarrow (1-\tau_{t}) \mathbf{y}_{t} + \tau_{t} \mathbf{z}_{t}\\
\mathbf{y}_{t+1 } &\leftarrow \argmin_{\mathbf{y} \in \mathcal{D}}  \frac{1}{\alpha}\langle \nabla h(\mathbf{x}_t) , \mathbf{y} - \mathbf{x}_t \rangle+\frac{\| \mathbf{y} -\mathbf{x}_t\|_\star^2}{2} .\\
z_{t+1} &\leftarrow \argmin_{\mathbf{z} \in \mathcal{D}} \eta_t \langle  \nabla h(\Boldx_t) , \Boldz- \Boldz_t\rangle + D_w(\Boldz_t, \Boldz).
\end{align*}
}
For some stepsize parameter sequence $\eta_t$.
\caption{Accelerated Gradient Descent}
\label{algorithm::accelerated_gradient_descent}
\aldotwo{Verify this algorithm.}
\end{algorithm} 

Algorithm~\ref{algorithm::accelerated_gradient_descent} satisfies the following convergence guarantee:

\begin{theorem}[Accelerated Gradient Descent for general norms. Theorem 4.1 in~\citet{allen2014linear}]\label{theorem::accelerated_gradient_descent}  Let $w$ be a distance generating function and let $D_\star$ be an upper bound to $D_w(\Boldx_0, \Boldx_\star)$. Given an $\alpha-$smooth function $h$  w.r.t. the $\| \cdot \|_\star$ norm over domain $\mathcal{D}$, then $T$ iterations of Algorithm~\ref{algorithm::accelerated_gradient_descent} ensure:%
 \begin{equation*}
    h(\mathbf{y}_t) - h(\mathbf{x}^\star) \leq  \frac{4 \alpha D_*  }{T^2}.
\end{equation*}

\end{theorem}

We care about recovering almost optimal solutions (in function value). Let's define an $\epsilon-$optimal solution:

\begin{definition}
Let $\epsilon>0$. We say that $\mathbf{x}$ is an $\epsilon-$optimal solution of an $\alpha-$smooth function $h: \mathbb{R}^d\rightarrow \mathbb{R}$ if:
\begin{equation*}
    h(\mathbf{x}) - h(\mathbf{x}^\star) \leq \epsilon
\end{equation*}
Where $h(\mathbf{x}^\star) = \min_{\mathbf{x} \in \mathbb{R}^d} h(\mathbf{x})$.
\end{definition}

We can also show the following bound on the gradient norm for any $\epsilon-$optimal solutions of $h$.

\begin{restatable}{lemma}{boundinggradients}
\label{lemma::bounding_gradients}
If $\Boldx$ is an $\epsilon-$optimal solution for the $\alpha-$smooth function $h: \mathbb{R}^d \rightarrow \mathbb{R}$ w.r.t. norm $\| \cdot \|_\star$ then the gradient of $h$ at $\Boldx$ satisfies:
\begin{equation*}
    \| \nabla h(\Boldx) \| \leq \sqrt{ 2\alpha \epsilon}.
\end{equation*}
\end{restatable}

The proof of this lemma can be found in Appendix~\ref{section::convergence_rates_REPS}.

When $h=J_D$ the \ref{equation::dual_visitation_regularized} function in the reinforcement learning setting, we set $\| \cdot \|_* = \| \cdot \|_\infty$ and $\| \cdot \| = \| \cdot\|_1$. We are ready to prove convergence guarantees for Algorithm~\ref{algorithm::accelerated_gradient_descent} when applied to the objective $J_D$. %

\begin{lemma}\label{lemma::grad_descent_guarantee} Let Assumptions~\ref{assumption::bounded_rewards},~\ref{assumption::lower_bound_q} and \ref{ass:uniform} hold. Let $\mathcal{D} = \left\{ \Boldv \text{ s.t. } \| \Boldv \|_\infty \le D\right\}$, and define the distance generating function to be $w(\Boldx) = \| \Boldx \|_2^2$. After $T$ steps of Algorithm~\ref{algorithm::accelerated_gradient_descent}, the objective function $J_D$ evaluated at the iterate $\Boldv_T = y_T$ satisfies:
\vspace{-.3cm}
\begin{equation*}
   J_D(\Boldv_T ) - J_D(\Boldv^*)\leq 4\eta (|\CalS|+1)^2\frac{\left( 1 + c' \right)^2}{(1-\gamma)^2T^2}.
\end{equation*}
Where $c'  = \frac{\log\frac{|S||A| }{\beta \rho}}{\eta}$.
\end{lemma}

\begin{proof}
This results follows simply by invoking the guarantees of Theorem \ref{theorem::accelerated_gradient_descent}, making use of the fact that $J_D$ is $(| \CalS| + 1)\eta-$smooth as proven by Lemma~\ref{lemma::RL_smoothness_dual}, observing that as a consequence of Lemma~\ref{lemma::dual_variables_bound_ofir}, $\Boldv^\star \in \mathcal{D}$ and using the inequality $\| \Boldx \|_2^2 \leq |\CalS|\| \Boldx\|_\infty^2$ for $\Boldx \in \mathbb{R}^{| \CalS|}$. 
\end{proof}
Lemma~\ref{lemma::grad_descent_guarantee} can be easily turned into the following guarantee regarding the dual function value of the final iterate:%
\begin{corollary}\label{corollary::lower_bound_T_accelerated_gradient}
Let $\epsilon > 0$. If Algorithm~\ref{algorithm::accelerated_gradient_descent} is ran for at least $T$ rounds
\begin{equation*}
    T \geq 2\eta^{1/2} (|\mathcal{S}|+1)\frac{\left( 1 + c' \right)}{(1-\gamma) \sqrt{\epsilon}}
\end{equation*}
then $\Boldv_{T}$ is an $\epsilon-$optimal solution for the dual objective $J_D$. %
\end{corollary}
If $T$ satisfies the conditions of Corollary~\ref{corollary::lower_bound_T_accelerated_gradient} a simple use of Lemma~\ref{lemma::bounding_gradients} allows us to bound the $\| \cdot \|_1$ norm of the dual function's gradient at $\Boldv_{T}$:
\begin{equation*}
    \| \nabla J_D(\Boldv_{T}) \|_1 \leq \sqrt{ 2 (|\CalS|+1)\eta \epsilon}
\end{equation*}
If we denote as $\pi_{T}$  to be the policy induced by $\BoldLambda^{\Boldv_{T}}$, and $\BoldLambda_{\eta}^\star$ is the candidate dual solution corresponding to $\Boldv^\star$. A simple application of Lemma~\ref{lemma::bounding_primal_value_candidate_solution} yields:
\begin{equation*}
    J_P( \BoldLambda^{\pi_{T}} ) \geq J_P(\BoldLambda_\eta^\star) -\frac{\sqrt{ 2 (|\CalS|+1)\eta \epsilon} }{1-\gamma} \left( 2 + \frac{1+\log\frac{|\CalS||\CalA| }{\beta^2 \rho^4}}{\eta} \right)
\end{equation*}
The following is the equivalent version of optimality for regularized objectives:
\begin{definition}
Let $\epsilon >0$. We say $\tilde{\pi}$ is an $\epsilon-$optimal regularized policy if $J_P( \BoldLambda^{\widetilde{\pi}} ) \geq J_P(\BoldLambda_\eta^\star) -\epsilon$.
\end{definition}  

This leads us to the main result of this section:
\begin{corollary}\label{corollary::regularized_result}
 For any $\xi  > 0$, and let $c'' = \frac{1+\log\frac{|\CalS||\CalA| }{\beta^2 \rho^4}}{\eta}$. If $T \geq 4\eta\left(| \CalS| + 1 \right)^{3/2} \frac{\left( 2 + c'' \right)^2 }{(1-\gamma)^2 \xi} $ then:
\begin{equation*}
     J_P( \BoldLambda^{\pi_{T}} )\geq J_P(\BoldLambda^{\star}_\eta)  -\xi.
\end{equation*}
\end{corollary}
Thus Algorithm~\ref{algorithm::accelerated_gradient_descent} achieves an $\mathcal{O}(1/(1-\gamma)^2\epsilon)$ rate of convergence to an $\epsilon-$optimal regularized policy. We now proceed to show that an appropriate choice for $\eta$ can be leveraged to obtain an $\epsilon-$optimal policy.%

\begin{restatable}{theorem}{mainacceleratedresult}\label{theorem::main_accelerated_result}
For any $\epsilon > 0$, let $\eta = \frac{1}{2\epsilon \log(\frac{|\CalS||\CalA|}{\beta})}$. If $T \geq (|\CalS|+1)^{3/2}\frac{(2+c'')^2}{(1-\gamma)^2\epsilon^2}$, then $\pi_T$ is an $\epsilon-$optimal policy. %
\end{restatable}

The proof of this result can be found in  Appendix~\ref{section::accelerated_gradient_descent_guarantees}. The main difficulty in deriving the guarantees of Theorem~\ref{theorem::main_accelerated_result} lies in the need to translate the function value optimality guarantees of Accelerated Gradient Descent into $\epsilon$-optimality guarantees for the candidate policy $\pi_T$. This is where our results from Lemma~\ref{lemma::bounding_primal_value_candidate_solution} have proven fundamental. It remains to show that it is possible to obtain an $\epsilon-$optimal policy access to the true model is only via samples.

\section{Stochastic Gradients}\label{section::stochastic_gradients}
In this section we show how to obtain stochastic (albeit biased) gradient estimators $\widehat{\nabla}_\Boldv J_D(\Boldv)$ for $\nabla_\Boldv J_D(\Boldv)$ (see Algorithm~\ref{algorithm::biased_gradient}).  We use $\widehat{\nabla}_\Boldv J_D(\Boldv)$ to perform biased stochastic gradient descent steps on $J_D(\Boldv)$ (see Algorithm~\ref{algorithm::biased_gradient_descent}). In Lemma~\ref{lemma::biased_gradient_guarantee} we prove guarantees for the bias and variance of this estimator and show rates for convergence in function value to the optimum of $J_D(\Boldv)$ in Lemma~\ref{lemma::SGD_result_simplified}. We turn these results into guarantees for $\epsilon-$optimality of the final candidate policy in Theorem~\ref{theorem::main_result}.   
Let's start by noting that:
\begin{align*}
    &\left( \nabla_\Boldv J_D(\Boldv) \right)_s = 
     (1 - \gamma)  \BoldMu_s + \\
     &\mathbb{E}_{(s', a, s'') \sim \Boldq \times \mathbf{P}_{a}(\cdot |s')} \Big[ \mathbf{B}_{s',a}^\Boldv \left( \gamma \mathbf{1}(s'' = s) - \mathbf{1}(s' = s)\right) \Big],  
\end{align*}
Where $\mathbf{B}^{\Boldv}_{s,a} =\frac{ \exp(\eta \mathbf{A}_{s, a}^\Boldv)}{\mathbf{Z}}$ and $\mathrm{Z} = \sum_{s,a} \exp\left(\eta  \mathbf{A}^{\Boldv}_{s,a} \right)\Boldq_{s,a}$. We will make use of this characterization to devise a plug-in estimator for this quantity:

\begin{algorithm}[H]
\textbf{Input} Number of samples $t$.\\
Collect samples $\{(s_\ell, a_\ell, s_\ell')\}_{\ell=1}^t$ such that $(s_\ell, a_\ell) \sim \Boldq$ while $s_\ell' \sim \mathbf{P}_{a_\ell}(\cdot | s_\ell)$\\
\For{$(s,a) \in \mathcal{S}\times \mathcal{A}$}{
Build empirical estimators $\widehat{\mathbf{A}}^{\Boldv}(t) \in \mathbb{R}^{|\mathcal{S}|\times |\mathcal{A}|}$ and $\widehat{\Boldq}(t) \in \mathbb{R}^{|\mathcal{S}|\times |\mathcal{A}|}$.\\
Compute estimators $\widehat{\mathbf{B}}_{s,a}^{\Boldv}(t) = \frac{\exp(\eta \widehat{\mathbf{A}}^\Boldv_{s, a}(t) )}{ \widehat{\mathbf{Z}}(t)}$. \\
Where $\widehat{\mathbf{Z}}(t) = \sum_{s,a} \exp(\eta \widehat{\mathbf{A}}^{\Boldv}_{s,a}(t) ) \widehat{\Boldq}_{s,a}(t)$. \\
}
Produce a final sample $(s_{t+1}, a_{t+1}) \sim \Boldq$ and $s_{t+1}' \sim \mathbf{P}_{a_{t+1}}( \cdot | s_{t+1})$. \\
Compute $\widehat{\nabla}_{\Boldv} J_D(\Boldv) $ such that:
\begin{small}
\begin{align*}
 \left(\widehat{\nabla}_{\Boldv} J_D(\Boldv) \right)_s &= (1-\gamma)\boldsymbol{\mu}_s + \\
 &\widehat{\mathbf{B}}_{s_{t+1}, a_{t+1}}(t)\left(\gamma   \mathbf{1}(s_{t+1}' = s) -  \mathbf{1}(s_{t+1} = s)\right). 
\end{align*}
\end{small}
\textbf{Output:} $\widehat{\nabla}_{\Boldv} J_D(\Boldv)$.

\caption{Biased Gradient Estimator}
\label{algorithm::biased_gradient}
\end{algorithm}

We now proceed to bound the bias of this estimator:
\begin{restatable}{lemma}{biasedgradientguarantee}
\label{lemma::biased_gradient_guarantee}
Let $\delta, \xi \in (0,1)$ with $\xi \leq \min(\beta, \frac{1}{4})$. With probability at least $1-\delta$ for all $t \in \mathbb{N}$ such that $$\frac{t}{\ln\ln(2t) } \geq \frac{120(\ln\frac{41.6|\mathcal{S}||\mathcal{A}|}{\delta} +1)}{\beta \xi^2}\max\left( 480 \eta^2 \gamma^2 \| \Boldv \|_\infty^2  , 1 \right)$$
the plugin estimator $\widehat{\nabla}_{\Boldv} J_D(\Boldv)$ satisfies:
\begin{align*}
    \max_{u \in \{ 1,2,\infty\}}  \| \hat{\mathbf{g}} - \mathbb{E}_{t+1}[\hat{\mathbf{g}}] \|_u&\leq  \frac{8}{\beta},   \\
    \max_{u \in \{1,2,\infty\}} \| \mathbb{E}_{t+1}[ \hat{\mathbf{g}}] - \mathbf{g} \|_u &\leq 8 \xi,  \\
    \mathbb{E}\left[\| \hat{\mathbf{g}}- \mathbb{E}_{t+1}[\hat{\mathbf{g}}] \|_2^2  \Big| \widehat{\mathbf{B}}^\Boldv(t)  \right] &\leq  \frac{8}{\beta},
\end{align*}
where $\hat{\mathbf{g}} = \widehat{\nabla}_\Boldv J_D(\Boldv) $, $\mathbf{g}= \nabla_\Boldv J_D(\Boldv)$, and $\mathbb{E}_{t+1}[\cdot ] = \mathbb{E}_{s_{t+1}, a_{t+1}, s'_{t+1} }[\cdot | \widehat{\mathbf{B}}^{\Boldv}(t)]$. 
\end{restatable}
The proof of this lemma can be found in Appendix~\ref{section::biased_stochastic_gradients_appendix}. 

We will now make use of Lemma~\ref{lemma::biased_gradient_guarantee} along with the following guarantee for projected Stochastic Gradient Descent to prove convergence guarantees for Algorithm~\ref{algorithm::biased_gradient_descent}.

\begin{algorithm}[H]
\textbf{Input} Desired accuracy $\epsilon$, learning rates $\{\tau_t\}_{t=1}^\infty$, and number-of-samples function $n: \mathbb{N} \rightarrow \mathbb{N}$ .\\
Initialize $\Boldv_0 = \mathbf{0}$
\For{$t=1, \cdots, T$}{
Get $\widehat{\nabla}_{\Boldv} J_D(\Boldv)$ with $n(t)$ samples via Algorithm~\ref{algorithm::biased_gradient}.\\
Perform update:
\begin{align*}
    \Boldv'_t \leftarrow \Boldv_t - \tau_t \widehat{\nabla}_{\Boldv} J_D(\Boldv).  \\
    \Boldv_{t} \leftarrow \Pi_{\mathcal{D}} ( \Boldv'_{t}).
\end{align*}
}
 $\Pi_{\mathcal{D}}$ denotes the projection to $\mathcal{D}= \left\{ \Boldv \text{ s.t. } \| \Boldv \|_\infty \le D\right\}$. \\
\textbf{Output:} $\Boldv_T$.

\caption{Biased Stochastic Gradient Descent}
\label{algorithm::biased_gradient_descent}
\end{algorithm}

The following holds:

\begin{restatable}{lemma}{helperprojectedsgd}
\label{lemma::helper_projected_sgd}
Let $f: \mathbb{R}^d \rightarrow \mathbb{R}$ be an $L-$smooth function. We consider the following update:
\begin{align*}
    \Boldx_{t+1}' &= \Boldx_t - \tau  \left( \nabla f(\Boldx_{t}) +  \boldsymbol{\epsilon}_t + \Boldb_t\right)\\
    \Boldx_{t+1} &= \Pi_{\mathcal{D}}( \Boldx_{t+1}').
\end{align*}
If $\tau \leq \frac{2}{L}$ then:
\begin{align*}
       f(\Boldx_{t+1}) - f(\Boldx_\star) &\leq \frac{ \| \Boldx_t - \Boldx_\star \|^2 - \| \Boldx_{t+1} - \Boldx_{\star}\|^2}{2\tau}  +\\
       &2\tau \| \nabla f(\Boldx_t)  \|^2  + 5\tau\| \Boldb_t\|^2 + 5\tau \| \boldsymbol{\epsilon}_t\|^2  + \\
       &\| \Boldb_t\|_1\|\Boldx_t - \Boldx_\star\|_\infty - 
       \langle \boldsymbol{\epsilon}_t, \Boldx_t -\Boldx_\star\rangle.
\end{align*}

\end{restatable}

The proof of Lemma~\ref{lemma::helper_projected_sgd} is in Appendix~\ref{section::appendix_SGD}. Lemma~\ref{lemma::biased_gradient_guarantee} implies the following guarantee for the following projected stochastic gradient algorithm with biased gradients $\widehat{\nabla}_{\Boldv} J_D(\Boldv)R$:

\begin{restatable}{lemma}{SGDresultsimplified}
\label{lemma::SGD_result_simplified}
We assume $\eta \geq \frac{4}{\beta}$. Set $\xi_t = \frac{8|\mathcal{S}|\eta D}{\sqrt{t}}$ and $\tau_t = \frac{1}{16|\mathcal{S}|\eta\sqrt{t}}$. If we take $t$ gradient steps using $n(t)$ samples from $\Boldq \times \mathbf{P}$ (possibly reusing the samples for multiple gradient computations) with $n(t)$ satisfying:

\begin{equation*}
    n(t) \geq \frac{525 t  \left( \ln \frac{100 |\mathcal{S}||\mathcal{A}|t^2}{\delta} + 1\right)^3}{\beta  |\mathcal{S}|^2  }
\end{equation*}
Then for all $t \geq 1$ we have that with probability at least $1-3\delta$ and simultaneously for all $t \in \mathbb{N}$ such that $t \geq \frac{64|\mathcal{S}|^2\eta^2 D^2}{\beta}$:
 \begin{align*}
 J_D\left(\frac{1}{t}\sum_{\ell=1}^t \Boldv_\ell\right) &\leq  J_D(\Boldv_\star) + \widetilde{\mathcal{O}}\left( \frac{D^2 |\mathcal{S}|\eta}{\sqrt{t} }\right). 
 \end{align*} 
\end{restatable}

The proof of Lemma~\ref{lemma::SGD_result_simplified} is in Appendix~\ref{section::biased_stochastic_gradients_appendix}. 
 Lemma~\ref{lemma::SGD_result_simplified} implies that making use of $N$ samples it is possible to find a candidate $\bar{\mathbf{v}}_N$ such that $J_D(\bar{\Boldv}_N) \leq J_D(\Boldv_\star)  + \widetilde{\mathcal{O}}\left( \frac{D^2 \eta}{\beta\sqrt{N} }\right)$.  This in turn implies by a simple use of Lemma~\ref{lemma::bounding_gradients} that $\| \nabla J_D(\bar{\Boldv}_N)\|_1 \leq  \widetilde{\mathcal{O}}\left( \frac{|\mathcal{S}|^{1/2}D \eta}{\sqrt{\beta}N^{1/4} }\right)$. If we denote as $\bar{\pi}_N$ to the policy induced by $\BoldLambda^{\bar{\Boldv}_N}$, a simple application of Lemma~\ref{lemma::bounding_primal_value_candidate_solution} yields:
\vspace{-.3cm}
\begin{equation*}
    J_P( \BoldLambda^{\bar{\pi}_{N}} ) \geq J_P(\BoldLambda_\eta^\star) - \widetilde{\mathcal{O}}\left( \frac{|\mathcal{S}|^{1/2}D \eta}{(1-\gamma)\sqrt{\beta}N^{1/4} } \right) 
\end{equation*}
Thus Algorithm~\ref{algorithm::biased_gradient_descent} achieves an $\mathcal{O}(1/(1-\gamma)^8\epsilon^4)$ rate of convergence to an $\epsilon-$optimal regularized policy. We proceed to show that an appropriate setting for $\eta$ can be leveraged to obtain an $\epsilon-$optimal policy:
\begin{theorem}[Informal]\label{theorem::main_result}
For any $\epsilon > 0$ let $\eta = \frac{1}{2\epsilon \log(\frac{|\CalS||\CalA|}{\beta})}$. If $N \geq \widetilde{\mathcal{O}}\left( \frac{1}{\epsilon^8 (1-\gamma)^8 \beta^2} \right)$, then with probability at least $1-\delta$ it is possible to find a candidate $\bar{\Boldv}_N$ such that $\bar{\pi}_N$ is an $\epsilon-$optimal policy. 
\end{theorem}

\section{Conclusion}

This work presents an analysis of first-order optimization methods for the REPS objective in reinforcement learning. We prove convergence rates of $O(1/\epsilon^2)$ for accelerated gradient descent on the dual of the KL-regularized max-return LP in the case of a known transition function with convergence rate. For the unknown case, we propose a biased stochastic gradient descent method relying on samples from behavior policy and show that it converges to an optimal policy with rate $O(1/\epsilon^8)$. There are several interesting questions that remain open. First, while directly optimizing the dual via gradient methods is convenient from an algorithmic perspective, prior unregularized saddle-point methods have been shown to achieve a faster $O(1/\epsilon)$ convergence \citep{bas2019faster}. An important open direction is thus to understand if faster rates are possible in order to bridge this gap, or if optimizing the regularized dual directly is fundamentally limited. Second, we only considered MDPs with finite state and action spaces. It is therefore of interest to see if these ideas readily extend to infinite or very large spaces through function approximation.

\bibliography{refs}
\bibliographystyle{icml2021}

\appendix

\onecolumn

\renewcommand{\contentsname}{Contents of main article and appendix}
\tableofcontents
\addtocontents{toc}{\protect\setcounter{tocdepth}{3}}
\clearpage

\section{Geometry of regularized Linear Programs}\label{section::appendix_geometry_regularized_linear}

We start by fleshing out the connection between strong convexity and smoothness charted in Lemma~\ref{lemma::equivalence_smoothness_strong_convexity}:

\equivalencesmoothnessstrongconvexity*

\begin{proof}

Let $\Boldu, \Boldw \in \mathbb{R}^n$ and  $\Boldx,\Boldy \in\mathcal{D}$ be such that $\nabla F^*(\Boldu) = \Boldx$ and $\nabla F^*(\Boldw) = \Boldy$. By definition this also implies that:
\begin{align*}
    \langle \nabla F(\Boldx) - \Boldu, \Boldz_1 - \Boldx \rangle \geq 0, \quad \forall \Boldz \in \mathcal{D} \\
\langle \nabla F(\Boldy) - \Boldw, \Boldz_2 - \Boldy \rangle \geq 0, \quad \forall \Boldz \in \mathcal{D} 
\end{align*}
Setting $\Boldz_1 = \Boldy$ and $\Boldz_2 = \Boldx$ along with the definition of $\Boldx, \Boldy$ and summing the two inequalities:
\begin{equation}\label{equation::inequality_boundary}
    \langle \nabla F(\Boldx) - \nabla F(\Boldy) , \Boldy - \Boldx \rangle \geq \langle \nabla F^*(\Boldw) - \nabla F^*(\Boldu), \Boldu - \Boldw \rangle.  
\end{equation}
By strong convexity of $F$ over domain $\mathcal{D}$ we see that:
\begin{align*}
    F(\Boldx) \geq F(\Boldy) + \langle \nabla F(\Boldy), \Boldx-\Boldy \rangle + \frac{\beta}{2} \| \Boldx-\Boldy \|^2\\
    F(\Boldy) \geq  F(\Boldx) + \langle \nabla F(\Boldx), \Boldy-\Boldx \rangle + \frac{\beta}{2} \| \Boldx-\Boldy \|^2
\end{align*}
Summing both inequalities yields:
\begin{equation*}
    \beta \| \Boldx-\Boldy\|^2 \leq   \langle \nabla F(\Boldx) - \nabla F(\Boldy), \Boldx-\Boldy \rangle
\end{equation*}

Plugging in the definition of $\Boldu$ and $\Boldw$ along with inequality \ref{equation::inequality_boundary}:

\begin{equation*}
      \beta \| \nabla F^*(\Boldu) - \nabla F^*(\Boldw)  \|^2 \leq \langle \Boldu - \Boldw, \nabla F^*(\Boldu) - \nabla F^*(\Boldw) \rangle \stackrel{(i)}{\leq} \| \Boldu-\Boldw  \|_* \| \nabla F^*(\Boldu) - \nabla F^*(\Boldw)\|  .
\end{equation*}
Where inequality $(i)$ holds by Cauchy-Schwartz and consequently:
\begin{equation*}
    \| \nabla F^*(\Boldu) - \nabla F^*(\Boldw)   \| \leq \frac{1}{\beta} \| \Boldu-\Boldw  \|_*
\end{equation*}

By the mean value theorem there exists $\Boldz \in [\Boldu, \Boldw]$:

\begin{align*}
    F^*(\Boldu) &= F^*(\Boldw) + \langle \nabla F^*(\Boldz), \Boldw-\Boldu\rangle   \\
    &= F^*(\Boldw) + \langle \nabla F^*(\Boldw), \Boldw-\Boldu\rangle  + \langle \nabla F^*(\Boldz)-\nabla F^*(\Boldw), \Boldw-\Boldu\rangle  \\
    &\leq F^*(\Boldw) + \langle \nabla F^*(\Boldw), \Boldw-\Boldu\rangle +  \| \nabla F^*(\Boldz)-\nabla F^*(\Boldw)\| \| \Boldw-\Boldu\|_* \\
    &\leq F^*(\Boldw) + \langle \nabla F^*(\Boldw), \Boldw-\Boldu\rangle +  \frac{1}{\beta}\|\Boldz-\Boldw\|_* \| \Boldw-\Boldu\|_* \\
    &\leq F^*(\Boldw) + \langle \nabla F^*(\Boldw), \Boldw-\Boldu\rangle +  \frac{1}{\beta} \| \Boldw-\Boldu\|^2_*
\end{align*}

Which concludes the proof.

\end{proof}

The proof of Lemma~\ref{lemma::equivalence_smoothness_strong_convexity} yields the following useful result that characterizes the smoothness properties of the dual function in a regularized LP:

\subsection{Proof of Lemma~\ref{lemma::dual_smoothness_regularized_LP}}

\dualsmoothnessregularizedLP*

\begin{proof}
Recall that:
\begin{equation*}
    g_D(v) = \langle v, b \rangle + F^*(r-v^\top E).
\end{equation*}

Notice that:

\begin{equation*}
    \nabla_{v} g_D(v) = b + E \nabla F^*(r - v^\top E).
\end{equation*}

And therefore for any two $v_1, v_2$:
\begin{align*}
    \|  \nabla g_D(v_1) - \nabla g_D(v_2) \| &= \| E \left( \nabla F^*(r - v_1^\top E) - \nabla F^*(r - v_2^\top E)   \right) \| \\
    &\stackrel{(i)}{\leq} \| E \|_{\cdot, *} \| \nabla F^*(r - v_1^\top E) - \nabla F^*(r - v_2^\top E) \| \\
    &\stackrel{(ii)}{\leq} \| E \|_{\cdot, *} \frac{1}{\beta} \| v_1^\top E - v_2^\top E  \|_* \\
    &\stackrel{(ii)}{\leq} \frac{ \| E\|_{\cdot, *}^2 }{\beta} \| v_1 - v_2 \|_*
\end{align*}

The result follows. %

\end{proof}

We can apply Lemma~\ref{lemma::dual_smoothness_regularized_LP} to problem~\ref{equation::primal_visitation_regularized} and thus characterize the smoothness properties of the dual function $J_D$.

\subsection{Proof of Lemma~\ref{lemma::RL_smoothness_dual}}\label{subsection::proof_of_RL_smoothness_dual}

\RLsmoothnessdual*

\begin{proof}

Recall that \ref{equation::primal_visitation_regularized} can be written as \ref{equation::regularized_LP}:
\begin{align*}
    \max_{\BoldLambda \in \mathcal{D}} \langle \Boldr, \BoldLambda \rangle - F(\BoldLambda)  \\
    \text{s.t. } \mathbf{E}\BoldLambda = b. \notag
\end{align*}

Where the regularizer ($   F(\BoldLambda) := \frac{1}{\eta}\sum_{s,a} \BoldLambda_{s,a} \left(\log\left(\frac{\BoldLambda_{s,a}}{\Boldq_{s,a}}\right) - 1\right)$) is $\frac{1}{\eta}-\| \cdot \|_1$ strongly convex. In this problem $\Boldr$ corresponds to the reward vector, the vector $\Boldb = (1-\gamma) \BoldMu \in \mathbb{R}^{|\mathcal{S}|}$ and matrix $\mathbf{E} \in \mathbb{R}^{|\mathcal{S}| \times |\mathcal{S}|\times |\mathcal{A}|}$ takes the form:
\begin{equation*}
    \mathbf{E}[s, s',a] = \begin{cases}
                \gamma \mathbf{P}_a(s | s')     & \text{if } s \neq s'\\
                  1-\gamma \mathbf{P}_{a}(s|s)  & \text{o.w.}
                \end{cases}
\end{equation*}
Therefore
\begin{equation*}
    \| \mathbf{E} \|_{1, \infty} \leq S+1 
\end{equation*}

The result follows as a corollary of Lemma~\ref{lemma::equivalence_smoothness_strong_convexity}.

\end{proof}

\section{Proof of Lemma~\ref{lemma::bounding_primal_value_candidate_solution}}\label{section::proof_of_primal_value_candidate_solution}

The objective of this section is to show that a candidate dual variable $\widetilde{\Boldv}$ having small gradient gives rise to a policy whose true visitation distribution has large primal value $J_P$. 

\boundingprimalvaluecandidatesolution*

\begin{proof}

For any $\BoldLambda$ and $\Boldv$ let the lagrangian $J_L(\BoldLambda, \Boldv)$ be defined as,
\begin{equation*}
J_L( \BoldLambda, \Boldv) = (1-\gamma) \langle \BoldMu, \Boldv \rangle + \left\langle \BoldLambda, \mathbf{A}^{\Boldv} - \frac{1}{\eta}\left( \log\left(\frac{\BoldLambda}{\Boldq}\right) - 1 \right)  \right\rangle
\end{equation*}

Note that $J_D(\widetilde{\Boldv}) = J_L(\widetilde{\BoldLambda}, \widetilde{\Boldv})$ and that in fact $J_L$ is linear in $\bar{\Boldv}$; \ie, 
$$J_L(\widetilde{\BoldLambda}, \bar{\Boldv}) = J_L(\widetilde{\BoldLambda}, \widetilde{\Boldv}) + \langle \nabla_\Boldv J_L(\widetilde{\BoldLambda}, \widetilde{\Boldv}),\bar{\Boldv} - \widetilde{\Boldv}     \rangle.$$

Using Holder's inequality we have:
\begin{equation*}
J_L(\widetilde{\BoldLambda}, \bar{\Boldv} ) \geq J_L(\widetilde{\BoldLambda}, \widetilde{\Boldv}) - \|\nabla_\Boldv J_L(\widetilde{\BoldLambda}, \widetilde{\Boldv})\|_1\cdot \| \bar{\Boldv} - \widetilde{\Boldv}  \|_\infty = J_D(\widetilde{\Boldv}) - \|\nabla_\Boldv J_L(\widetilde{\BoldLambda}, \widetilde{\Boldv})\|_1\cdot \| \bar{\Boldv} - \widetilde{\Boldv}  \|_\infty. 
\end{equation*}
Let $\BoldLambda_\star$ be the candidate primal solution to the optimal dual solution $\Boldv_\star = \argmin_{\Boldv} J_D(\Boldv)$. By weak duality we have that $J_D(\widetilde{\Boldv}) \geq J_P(\BoldLambda^\star) = J_D(\Boldv_\star)$, and since by assumption $\|\nabla_\Boldv J_L(\widetilde{\BoldLambda}, \widetilde{\Boldv})\|_1 \leq \epsilon$:
\begin{equation}\label{equation::lowr_bounding_J_lambda_star}
    J_L(\widetilde{\BoldLambda}, \bar{\Boldv} ) \geq J_P(\BoldLambda^\star) -\epsilon \| \bar{\Boldv} - \widetilde{\Boldv}  \|_\infty. 
\end{equation}
In order to use this inequality to lower bound the value of $J_P(\BoldLambda^{\widetilde{\pi}})$, we will need to choose an appropriate $\bar{\Boldv}$ such that the LHS reduces to $J_P(\BoldLambda^{\widetilde{\pi}})$ while keeping the $\ell_\infty$ norm on the RHS small. Thus we consider setting $\bar{\Boldv}$ as:
\begin{equation*}
    \bar{\Boldv}_s = \mathbb{E}_{a,s' \sim \widetilde{\pi} \times \Trans }\left[ \Boldz_s + \Boldr_{s,a}        - \frac{1}{\eta}\left( \log\left(\frac{\BoldLambda^{\widetilde{\pi}}_{s,a}}{\Boldq_{s,a}}\right) - 1 \right) + \gamma \bar{\Boldv}_{s'} \right]
\end{equation*}
Where $\Boldz \in \mathbb{R}^{|S|}$ is some function to be determined later. It is clear that an appropriate $\Boldz$ exists as long as $\Boldz, \Boldr, \frac{1}{\eta}\left( \log\left(\frac{\BoldLambda^{\widetilde{\pi}}_{s,a}}{\Boldq_{s,a}}\right) - 1 \right)$ are uniformly bounded. Furthermore:
\begin{equation}\label{equation::z_infinity_bound}
    \| \bar{\Boldv} \|_\infty \leq \frac{\max_{s,a} \left| \Boldz_s + \Boldr_{s,a} - \frac{1}{\eta}\left( \log\left(\frac{\BoldLambda_{s,a}^{\widetilde{\pi}}}{\Boldq_{s,a}} \right) - 1 \right) \right|}{1-\gamma} \leq \frac{\| \Boldz \|_\infty + \| \Boldr \|_{\infty}  + \frac{1}{\eta}\left\| \log\left(  \frac{\BoldLambda_{s,a}^{\widetilde{\pi}}}{\Boldq_{s,a}}   \right) - 1\right\|_{\infty}  }{1-\gamma}
\end{equation}
Notice that by Assumptions~\ref{assumption::lower_bound_q} and~\ref{ass:uniform}, we have that $\rho, \beta \leq \frac{1}{2}$. This is because for all $\pi$, Assumption~\ref{ass:uniform} implies that:

\begin{equation*}
   0\leq 2\rho \leq | \CalS|  \rho \leq \sum_{s } \BoldLambda^{\pi}_s  =1
\end{equation*}
The proof for $\beta \leq \frac{1}{2}$ is symmetric. Due to Assumption~\ref{assumption::lower_bound_q} the $\| \cdot \|_\infty$ norm of $\log( \frac{\BoldLambda^{\widetilde{\pi}}}{\Boldq})  - \mathbf{1}_{|\CalS||\CalA|}  $ satisfies:
\begin{equation*}
    \left\| \log\left( \frac{\BoldLambda^{\widetilde{\pi}}}{\Boldq}\right)  - \mathbf{1}_{|\CalS||\CalA|} \right\|_\infty \leq 1 +   \left\| \log\left( \frac{\BoldLambda^{\widetilde{\pi}}}{\Boldq}\right)  \right\|_\infty \leq 1 + \max(|\log(\rho/\beta) |, \log(1/\beta) ) \leq 1 + \log(1/\rho)  + \log(1/\beta).
\end{equation*}

Notice the following relationships hold:
\begin{small}
\begin{align}
\left\langle \widetilde{\BoldLambda}, \mathbf{A}^{\bar{\Boldv}} - \frac{1}{\eta}\left( \log\left(\frac{\widetilde{\BoldLambda}}{\Boldq}\right) - 1 \right)  \right\rangle &=    \sum_{s} \widetilde{\BoldLambda}_{s} \left(\mathbb{E}_{a, s' \sim \widetilde{\pi} \times \BoldP } \left[  \Boldr_{s,a} + \gamma \bar{\Boldv}_{s'} - \bar{\Boldv}_s - \frac{1}{\eta} \left( \log\left( \frac{\widetilde{\BoldLambda}_{s,a}}{\Boldq_{s,a}}\right)-1 \right)   \right] \right) \notag\\
&=    \sum_{s} \widetilde{\BoldLambda}_{s} \left(\mathbb{E}_{a, s' \sim \widetilde{\pi} \times \BoldP } \left[ \frac{1}{\eta}\left(\log\left( \frac{\BoldLambda_{s,a}^{\widetilde{\pi}}}{\Boldq_{s,a}}\right) -1 \right)  - \frac{1}{\eta} \left( \log\left( \frac{\widetilde{\BoldLambda}_{s,a}}{\Boldq_{s,a}}\right)-1 \right) -\Boldz_s  \right] \right)\notag\\
&=  \sum_{s} \widetilde{\BoldLambda}_{s} \left(\mathbb{E}_{a, s' \sim \widetilde{\pi} \times \BoldP } \left[ \frac{1}{\eta}\log\left( \BoldLambda_{s,a}^{\widetilde{\pi}}\right)    - \frac{1}{\eta} \log\left( \widetilde{\BoldLambda}_{s,a}\right)  -\Boldz_s  \right] \right)\notag\\
&=  \sum_{s} \widetilde{\BoldLambda}_{s} \left(  \frac{1}{\eta}\log\left( \BoldLambda_{s}^{\widetilde{\pi}}\right)    - \frac{1}{\eta} \log\left( \widetilde{\BoldLambda}_{s}\right) - \Boldz_s\right) \label{equation::support_equation_1} 
\end{align}
\end{small}
Where $\widetilde{\BoldLambda}_{s} = \sum_{a} \widetilde{\BoldLambda}_{s,a}$ and $\BoldLambda^{\widetilde{\pi}}_s = \sum_a \BoldLambda^{\widetilde{\pi}}_{s,a}$. Note that by definition:

\begin{equation}
    (1-\gamma) \langle \BoldMu, \bar{\Boldv} \rangle = \left\langle \BoldLambda^{\widetilde{\pi}}, \Boldz + \Boldr - \frac{1}{\eta} \left( \log\left( \frac{\BoldLambda^{\widetilde{\pi}}}{\Boldq}\right)-1 \right)  \right\rangle = J_P(\BoldLambda^{\widetilde{\pi}}) + \langle \BoldLambda^{\widetilde{\pi}}, \Boldz \rangle. \label{equation::support_equation_2}
\end{equation}

\aldotwo{This may require some condition. Verify this comment,... I forgot why it was here.}

Let's expand the definition of $J_L( \widetilde{\BoldLambda}, \bar{\Boldv} )$ using Equations~\ref{equation::support_equation_1} and \ref{equation::support_equation_2}:
\begin{align*}
    J_L( \widetilde{\BoldLambda}, \bar{\Boldv} ) &= (1-\gamma) \langle \BoldMu, \bar{\Boldv} \rangle + \left\langle \widetilde{\BoldLambda}, \mathbf{A}^{\bar{\Boldv}}- \frac{1}{\eta}\left( \log\left(\frac{\widetilde{\BoldLambda}}{\Boldq}\right) - 1 \right)  \right\rangle\\
    &= J_P(\BoldLambda^{\widetilde{\pi}}) + \langle \BoldLambda^{\widetilde{\pi}}, \Boldz \rangle + \sum_{s} \widetilde{\BoldLambda}_{s} \left(  \frac{1}{\eta}\log\left( \BoldLambda_{s}^{\widetilde{\pi}}\right)    - \frac{1}{\eta} \log\left( \widetilde{\BoldLambda}_{s}\right) - \Boldz_s\right)\\
    &= J_P(\BoldLambda^{\widetilde{\pi}})  + \sum_s \left( \Boldz_s(\BoldLambda^{\widetilde{\pi}}_s - \widetilde{\BoldLambda}_s ) + \frac{1}{\eta}\widetilde{\BoldLambda}_s \log\left(\frac{\BoldLambda_s^{\widetilde{\pi}}}{\widetilde{\BoldLambda}_s} \right) \right)
\end{align*}
Since we want this expression to equal $J_P( \BoldLambda^{\widetilde{\pi}} )$, we need to choose $\Boldz$ such that:
\begin{equation*}
    \Boldz_s = \frac{\frac{1}{\eta} \log\left( \frac{\BoldLambda_s^{\widetilde{\pi}}}{\widetilde{\BoldLambda}_s}    \right)}{1-\frac{\BoldLambda_s^{\widetilde{\pi}}}{\widetilde{\BoldLambda}_s}}.
\end{equation*}
By Assumption~\ref{ass:uniform} we have that for all $s$:
\begin{equation*}
    \frac{\BoldLambda_s^{\widetilde{\pi}}}{\widetilde{\BoldLambda}_s} \geq \rho 
\end{equation*}
 
Now we bound $\| \Boldz_s \|_\infty$. Note that the function $h(\phi) = \frac{\log \phi}{1-\phi}$ is non decreasing and negative, and therefore the maximum of its absolute value is achieved at the lower end of its domain. This implies:
\begin{equation*}
    \left| \Boldz_s \right| \leq \frac{|h(\rho)|}{\eta} = \frac{\left|\log(\rho)\right|}{\eta(1-\rho)} \leq \frac{2\log(1/\rho)}{\eta}, \quad \forall s \in \CalS.
\end{equation*}
And therefore Equation~\ref{equation::z_infinity_bound} implies:
\begin{equation*}
    \| \bar{\Boldv} \|_{\infty} \leq \frac{ \frac{2\log(1/\rho)}{\eta} + 1 + \frac{1 + \log(1/\rho)  + \log(1/\beta)}{\eta}  }{1-\gamma}= \frac{1 + \frac{ 1 + \log(\frac{1}{\rho^3\beta})}{\eta} }{1-\gamma}
\end{equation*}
Putting these together we obtain the following version of equation~\ref{equation::lowr_bounding_J_lambda_star}:
\begin{equation*}
      J_L(\widetilde{\BoldLambda}, \bar{\Boldv} ) \geq J_P(\BoldLambda^\star) -\epsilon \left( \frac{1 + \frac{ 1 + \log(\frac{1}{\rho^3\beta})}{\eta} }{1-\gamma} + \| \widetilde{\Boldv } \|_\infty   \right)
\end{equation*}

As desired.
\end{proof}

\section{Proof of Lemma~\ref{lemma::dual_variables_bound_ofir}}\label{section::dual_variables_bound_ofir}

In this section we derive an upper bound for the $l_\infty$ norm of the optimal solution $\Boldv^\star$.

\dualvariablesboundofir*

\begin{proof}
Recall the Lagrangian form,
\begin{equation*}
    \min_{\Boldv}, \max_{\BoldLambda_{s,a} \in \Delta_{S\times A}}~ J_L(\BoldLambda, \Boldv) :=  (1-\gamma) \langle \Boldv, \BoldMu \rangle + \left\langle \BoldLambda, \mathbf{A}^{\Boldv} - \frac{1}{\eta}\left( \log\left(\frac{\BoldLambda_{s,a}}{\Boldq_{s,a}}\right) - 1 \right)\right\rangle.     
\end{equation*}
The KKT conditions of $\BoldLambda^*,\Boldv^*$ imply that for any $s,a$, either (1) $\BoldLambda^*_{s,a} = 0$ and $\frac{\partial}{\partial\BoldLambda_{s,a}}J_L(\BoldLambda^*,v^*) \le 0$ or (2) $\frac{\partial}{\partial\BoldLambda_{s,a}}J_L(\BoldLambda^*,\Boldv^*) = 0$. The partial derivative of $J_L$ is given by,
\begin{equation}
    \frac{\partial}{\partial\BoldLambda_{s,a}}J_L(\BoldLambda^*,\Boldv^*) = \Boldr_{s,a} -\frac{1}{\eta}\log\left(  \frac{\BoldLambda^*_{s,a} }{\Boldq_{s,a}}\right) + \gamma\sum_{s'} P_{a}(s'|s) \Boldv^*_{s'} - \Boldv^*_{s}. 
\end{equation}
Thus, for any $s,a$, either
\begin{equation}
    \BoldLambda^*_{s,a} = 0 ~\text{and}~ \Boldv^*_{s} \ge \Boldr_{s,a} -\frac{1}{\eta}\log\left( \frac{\BoldLambda^*_{s,a}}{\Boldq_{s,a}} \right) + \gamma\sum_{s'} P_{a}(s'|s) \Boldv^*_{s'}, 
\end{equation}
or,
\begin{equation}
    \BoldLambda^*_{s,a} > 0 ~\text{and}~ \Boldv^*_{s} = \Boldr_{s,a} -\frac{1}{\eta}\log\left( \frac{\BoldLambda^*_{s,a}}{\Boldq_{s,a}} \right)+ \gamma\sum_{s'} P_{a}(s'|s) \Boldv^*_{s'}.
\end{equation}
Recall that $\BoldLambda^*$ is the discounted state-action visitations of some policy $\pi_\star$; \ie, $\BoldLambda^*_{s,a} = \BoldLambda^{\pi_\star}_s \cdot \pi_\star(a|s)$ for some $\pi_\star$. Note that by Assumption~\ref{ass:uniform}, any policy $\pi$ has $\BoldLambda^{\pi_\star}_{s} > 0$ for all $s$. Accordingly, the KKT conditions imply,
\begin{equation}
    \pi_\star(a|s) = 0 ~\text{and}~ \Boldv^*_{s} \ge \Boldr_{s,a} -\frac{1}{\eta}\log \left(\frac{\BoldLambda^*_{s,a} }{\Boldq_{s,a}}\right) + \gamma\sum_{s'} P_{a}(s'|s) \Boldv^*_{s'}, 
\end{equation}
or,
\begin{equation}
    \pi_\star(a|s) > 0 ~\text{and}~ \Boldv^*_{s} = \Boldr_{s,a} -\frac{1}{\eta}\log\left( \frac{\BoldLambda^*_{s,a} }{\Boldq_{s,a}}\right) + \gamma\sum_{s'} P_{a}(s'|s) \Boldv^*_{s'}.
\end{equation}
Equivalently,
\begin{align}
    \Boldv^*_{s} &= \E_{a\sim\pi_\star(s)}\left[\Boldr_{s,a} - \frac{1}{\eta}\log\left(\frac{\BoldLambda^*_{s,a} }{\Boldq_{s,a}} \right)+ \gamma\sum_{s'} P_{a}(s'|s) \Boldv^*_{s'}\right] \\
    &= \frac{1}{\eta}\E_{a\sim\pi_\star(s)}\left[-\log\left(\frac{\pi(a|s)}{\Boldq_{a|s}}\right)\right] + \E_{a\sim\pi(s)}\left[r_{s,a} - \frac{1}{\eta}\log\left( \frac{\BoldLambda^{\pi_\star}_{s}}{\Boldq_s} \right) + \gamma\sum_{s'} P_{a}(s'|s) \Boldv^*_{s'}\right].
\end{align}
We may express these conditions as a Bellman recurrence for $\Boldv^*_s$:%
\begin{equation}
    \Boldv^*_{s} = 
    \frac{1}{\eta}\E_{a\sim\pi_\star(s)}\left[-\log\left(\frac{\pi(a|s)}{\Boldq_{a|s}}\right)\right] + \E_{a\sim\pi_\star(s)}\left[\Boldr_{s,a} - \frac{1}{\eta}\log\left(\frac{ \BoldLambda^{\pi_\star}_{s}}{\Boldq_s} \right)+ \gamma\sum_{s'} P_{a}(s'|s) \Boldv^*_{s'}\right].
\end{equation}
The solution to these Bellman equations is bounded when $\E_{a\sim\pi_\star(s)}\left[-\log\left(\frac{\pi_\star(a|s)}{\Boldq_{a|s}}\right)\right]$, $\Boldr_{s,a}$, and $\log\left( \frac{\BoldLambda^\pi_{s}}{\Boldq_s}\right)$ are bounded~\citep{puterman2014markov}. And indeed, by Assumptions~\ref{ass:uniform} and~\ref{assumption::bounded_rewards}, each of these is bounded by within $[\log \beta, \log |A|]$, $[0, 1]$, and $[\log \rho, -\log \beta]$, respectively.
We may thus bound the solution as,
\begin{equation}
    \|\Boldv^*\|_\infty \le \frac{1}{1-\gamma} \left( 1 + \frac{\log\frac{|S||A| }{\beta \rho}}{\eta} \right).
\end{equation}
\end{proof}

\section{Convergence rates for REPS}\label{section::convergence_rates_REPS}

We start with the proof of Lemma~\ref{lemma::bounding_gradients} which we restate for convenience:

\begin{lemma}%
If $\Boldx$ is an $\epsilon-$optimal solution for the $\alpha-$smooth function $h: \mathbb{R}^d \rightarrow \mathbb{R}$ w.r.t. norm $\| \cdot \|_\star$ then the gradient of $h$ at $\Boldx$ satisfies:
\begin{equation*}
    \| \nabla h(\Boldx) \| \leq \sqrt{ 2\alpha \epsilon}.
\end{equation*}
\end{lemma}

\begin{proof}
Let $\mathbf{x} \in \mathbb{R}^d$ be an arbitrary point and let $\mathbf{x}'$ equal the point resulting of the update
\begin{equation}%
   \Boldx' = \argmin_{\Boldy \in \mathcal{D}} \frac{1}{\alpha}\langle \nabla h(\Boldx), \Boldy - \Boldx\rangle + \frac{\| \Boldy - \Boldx \|_\star^2}{2} 
\end{equation}
 Notice that by smoothness of $h$:
\begin{align}
    h(\Boldx') &\leq h(\Boldx) + \langle \nabla h(\Boldx), \Boldx' - \Boldx \rangle + \frac{\alpha}{2} \| \Boldx' - \Boldx\|^2_\star = h(\Boldx)  -\frac{1}{2\alpha} \| \nabla h(\Boldx) \|^2 \label{equation::support_equation_lemma_epsilon_optimal_1}
\end{align}
Since $h(\Boldx^\star) \leq h(\Boldx')$ and $\Boldx$ is $\epsilon-$optimal:
\begin{equation*}
    \frac{1}{2\alpha} \| \nabla h(\Boldx) \|^2  + h(\Boldx^\star) \stackrel{(i)}{\leq}  \frac{1}{2\alpha} \| \nabla h(\Boldx) \|^2 + h(\Boldx') \stackrel{(ii)}{\leq} h(\Boldx) \stackrel{(iii)}{\leq} h(\Boldx^\star) + \epsilon 
\end{equation*}
Inequality $(i)$ holds because $h(\Boldx^\star) \leq h(\Boldx')$, inequality $(ii)$ by Equation~\ref{equation::support_equation_lemma_epsilon_optimal_1} and $(iii)$ by $\epsilon-$optimality of $\Boldx$. Therefore:
\begin{equation*}
    \frac{1}{2\alpha} \| \nabla h(\Boldx) \|^2 \leq \epsilon.
\end{equation*}
The result follows.
\end{proof}

We also show that the gradient norm of a smooth function over a bounded domain containing the optimum can be bounded:

\begin{lemma}\label{lemma::supporting_lemma_bound_gradient}
If $h$ is an $a\alpha-$smooth function w.r.t. norm $\| \cdot \|_\star$, and $\Boldx^\star$ is such that $\nabla h(\Boldx^\star) = \mathbf{0}$ then:
\begin{equation*}
    \| \nabla h(\Boldx) \| \leq \alpha \|\Boldx - \Boldx^\star\|_\star. 
\end{equation*}
And therefore whenever $\|\Boldx - \Boldx^\star\|_\star \leq D$ we have that:
\begin{equation*}
    \| \nabla h(\Boldx) \| \leq \alpha D. 
\end{equation*}

\end{lemma}

\begin{proof}
Since $h$ is $\alpha-$smooth:

\begin{equation*}
    h(\Boldx) \leq h( \Boldx^\star) + \langle \nabla h(\Boldx^\star), \Boldx - \Boldx^\star \rangle + \frac{\alpha}{2}\| \Boldx - \Boldx^\star \|^2_\star = h(\Boldx^\star) + \frac{\alpha}{2}\| \Boldx - \Boldx^\star \|^2_\star
\end{equation*}

Therefore:

\begin{equation*}
    h(\Boldx) - h(\Boldx^\star) \leq \frac{\alpha}{2}\| \Boldx - \Boldx^\star \|^2_\star.
\end{equation*}

Therefore, as a consequence of Lemma~\ref{lemma::bounding_gradients}:

\begin{equation*}
    \| \nabla h(\Boldx) \| \leq \alpha D.
\end{equation*}
The result follows.

\end{proof}

\subsection{Proof of Theorem~\ref{theorem::main_accelerated_result}} \label{section::accelerated_gradient_descent_guarantees}

We can now prove the estimation guarantees whenever exact gradients are available.

\begin{theorem}\label{theorem::main_accelerated_result_appendix}
For any $\epsilon > 0$, let $\eta = \frac{1}{2\epsilon \log(\frac{|\CalS||\CalA|}{\beta})}$. If $T \geq (|\CalS|+1)^{3/2}\frac{(2+c'')^2}{(1-\gamma)^2\epsilon^2}$, then $\pi_T$ is an $\epsilon-$optimal policy. %
\end{theorem}

\begin{proof}
As a consequence of Corollary~\ref{corollary::regularized_result}, we can conclude that:
\begin{equation*}
    J_P(\BoldLambda^{\pi_T}) \geq J_P(\BoldLambda^{\star, \eta}) - \frac{\epsilon}{2}.
\end{equation*}
Where $\BoldLambda_\eta^\star$ is the regularized optimum. Recall that:
\begin{equation*}
    J_P(\BoldLambda) = \sum_{s,a} \BoldLambda_{s,a} \Boldr_{s,a} -  \frac{1}{\eta}\sum_{s,a} \BoldLambda_{s,a} \left(\log\left(\frac{\BoldLambda_{s,a}}{\Boldq_{s,a}}\right) - 1\right). 
\end{equation*}
Since $\BoldLambda^{\star, \eta}$ is the maximizer of the regularized objective, it satisfies $J_P(\BoldLambda^{\star, \eta}) \geq J_P(\BoldLambda^*)$ where $\BoldLambda^\star$ is the visitation frequency of the optimal policy corresponding to the unregularized objective. We can conclude that:
\begin{align*}
    \sum_{s,a} \BoldLambda_{s,a}^{\pi_T} \Boldr_{s,a} &\geq \sum_{s,a} \BoldLambda^\star_{s,a} \Boldr_{s,a} + \frac{1}{\eta}\left(  \sum_{s,a} \BoldLambda^{\pi_T}_{s,a} \left( \log\left( \frac{ \BoldLambda^{\pi_T}_{s,a}}{\Boldq_{s,a}}\right)  - 1\right) -  \sum_{s,a} \BoldLambda^\star_{s,a} \left( \log\left( \frac{ \BoldLambda^\star_{s,a}}{\Boldq_{s,a}}\right)  - 1\right) \right) -\frac{\epsilon}{2} \\
    &= \sum_{s,a} \BoldLambda^\star_{s,a} \Boldr_{s,a} + \frac{1}{\eta}\left(  \sum_{s,a} \BoldLambda^{\pi_T}_{s,a} \left( \log\left( \frac{ \BoldLambda^{\pi_T}_{s,a}}{\Boldq_{s,a}}\right)  \right) -  \sum_{s,a} \BoldLambda^\star_{s,a} \left( \log\left( \frac{ \BoldLambda^\star_{s,a}}{\Boldq_{s,a}}\right)  \right) \right) - \frac{\epsilon}{2}\\
    &\geq  \sum_{s,a} \BoldLambda^\star_{s,a} \Boldr_{s,a} - \frac{2}{\eta}\log(\frac{|\CalS||\CalA|}{\beta}) - \frac{\epsilon}{2}
\end{align*}
And therefore if $\eta = \frac{1}{4\epsilon \log(\frac{|\CalS||\CalA|}{\beta})}$, we can conclude that:
\begin{equation*}
    \sum_{s,a} \BoldLambda_{s,a}^{\pi_T} \Boldr_{s,a} \geq \sum_{s,a} \BoldLambda_{s,a}^{\star} \Boldr_{s,a} - \epsilon.
\end{equation*}

\end{proof}

\section{Stochastic Gradient Descent}\label{section::appendix_SGD}

In this section we will have all the proofs and results corresponding to Section~\ref{section::stochastic_gradients} in the main. We start by showing the proof of Lemma~\ref{lemma::helper_projected_sgd}.  

\helperprojectedsgd*

\begin{proof}
Through the proof we use the notation $\| \cdot \|$ to denote the $L_2$ norm. By smoothness the following holds:

\begin{equation*}
    f(\Boldx_{t+1}) \leq f(\Boldx_{t}) + \langle \nabla f(\Boldx_t) , \Boldx_{t+1} - \Boldx_t \rangle  + \frac{L}{2}\| \Boldx_{t+1} - \Boldx_t \|_\infty^2 \leq f(\Boldx_{t}) + \langle \nabla f(\Boldx_t) , \Boldx_{t+1} - \Boldx_t \rangle  + \frac{L}{2}\| \Boldx_{t+1} - \Boldx_t \|^2 
\end{equation*}

Since $\Boldx_{t+1} = \Pi_{\mathcal{D}}( \Boldx_{t+1}')$ and by properties of a convex projection:

\begin{equation*}
    \langle \Boldx_{t+1}' - \Boldx_{t+1}, \Boldx_t - \Boldx_{t+1} \rangle \leq 0.
\end{equation*}

And therefore:

\begin{equation*}
    \langle \Boldx_{t} - \tau \left( \nabla f(\Boldx_t) + \Boldb_t + \boldsymbol{\epsilon}_t \right) - \Boldx_{t+1}, \Boldx_t - \Boldx_{t+1} \rangle \leq 0. 
\end{equation*}

Which in turn implies that :

\begin{equation*}
    \| \Boldx_{t} - \Boldx_{t+1} \|^2 \leq \tau \langle \nabla f(\Boldx_t) + \Boldb_t + \boldsymbol{\epsilon}_t, \Boldx_t - \Boldx_{t+1}\rangle.
\end{equation*}

We can conclude that:

\begin{equation}\label{equation::first_equation_proj_sgd}
    f(\Boldx_{t+1} ) \leq f(\Boldx_{t})  - \frac{\| \Boldx_t - \Boldx_{t+1}\|^2 }{\tau} + \langle \Boldb_t + \boldsymbol{\epsilon}_t, \Boldx_t - \Boldx_{t+1}\rangle +\frac{L}{2}\| \Boldx_{t+1} - \Boldx_t \|^2.
\end{equation}

By convexity:

\begin{equation*}
    f(\Boldx_\star) \geq f(\Boldx_t) + \langle \nabla f(\Boldx_t) , \Boldx_\star - \Boldx_t \rangle.
\end{equation*}

And therefore $f(\Boldx_t) \leq f(\Boldx_\star) + \langle \nabla f(\Boldx_t) , \Boldx_t - \Boldx_\star\rangle$.

Combining this last result with Equation~\ref{equation::first_equation_proj_sgd}:

\begin{equation}\label{equation::second_equation_proj_sgd}
    f(\Boldx_{t+1}) \leq f(\Boldx_\star) + \langle \nabla f(\Boldx_t) , \Boldx_t - \Boldx_\star\rangle + \left(\frac{L}{2} - \frac{1}{\tau} \right) \| \Boldx_{t+1}- \Boldx_t \|^2 + \langle \Boldb_t + \boldsymbol{\epsilon}_t, \Boldx_t - \Boldx_{t+1}\rangle. 
\end{equation}

Now observe that as a consequence of the contraction property of projections %

\begin{align*}
    \| \Boldx_{t+1} - \Boldx_\star \|^2 &\leq \| \Boldx_t - \tau \left( \nabla f(\Boldx_t) + \Boldb_t + \boldsymbol{\epsilon}_t \right) - \Boldx_\star \|^2 \\
    &= \| \Boldx_t - \Boldx_\star\|^2 + \tau^2 \| \nabla f(\Boldx_t) + \Boldb_t + \boldsymbol{\epsilon}_t \|^2 - 2\tau \langle \nabla f(\Boldx_t) + \Boldb_t + \boldsymbol{\epsilon}_t , \Boldx_t - \Boldx_\star \rangle.
\end{align*}
And therefore:

\begin{equation*}
    \langle \nabla f(\Boldx_t ) , \Boldx_t - \Boldx_\star \rangle \leq \frac{ \| \Boldx_t - \Boldx_\star \|^2 - \| \Boldx_{t+1} - \Boldx_{\star}\|^2}{2\tau} + \frac{\tau}{2} \| \nabla f(\Boldx_t) + \Boldb_t + \boldsymbol{\epsilon}_t \|^2 - \langle \Boldb_t + \boldsymbol{\epsilon}_t, \Boldx_t - \Boldx_\star \rangle.
\end{equation*}

Substituting this last inequality into Equation~\ref{equation::second_equation_proj_sgd}:

\begin{align}
     f(\Boldx_{t+1}) - f(\Boldx_\star) &\leq  \frac{ \| \Boldx_t - \Boldx_\star \|^2 - \| \Boldx_{t+1} - \Boldx_{\star}\|^2}{2\tau} + \frac{\tau}{2} \| \nabla f(\Boldx_t) + \Boldb_t + \boldsymbol{\epsilon}_t \|^2 - \langle \Boldb_t + \boldsymbol{\epsilon}_t, \Boldx_t - \Boldx_\star \rangle + \\
     & \quad \left(\frac{L}{2} - \frac{1}{\tau} \right) \| \Boldx_{t+1}- \Boldx_t \|^2 + \langle \Boldb_t + \boldsymbol{\epsilon}_t, \Boldx_t - \Boldx_{t+1}\rangle \label{equation::third_equation_proj_sgd}
\end{align}

Notice that as a consequence of the contraction property of projections:

\begin{align*}
    \| \Boldx_{t+1} - \Boldx_{t} \|^2 &\leq \| \Boldx_t - \tau \left( \nabla f(\Boldx_t) + \Boldb_t +\boldsymbol{\epsilon}_t \right)  - \Boldx_t \|  \\
    &= \tau\| \nabla f(\Boldx_t) + \Boldb_t + \boldsymbol{\epsilon}_t\|
\end{align*}

And therefore 

$$ \langle \Boldb_t + \boldsymbol{\epsilon}_t, \Boldx_t - \Boldx_{t+1} \rangle \leq \| \Boldb_t + \boldsymbol{\epsilon}_t\| \| \Boldx_t - \Boldx_{t+1} \| \leq \tau \| \Boldb_t + \boldsymbol{\epsilon}_t\| \| \nabla f(\Boldx_t) + \Boldb_t + \boldsymbol{\epsilon}_t\|$$ :

Substituting this back into~\ref{equation::third_equation_proj_sgd} and assuming $\frac{L}{2} \leq \frac{1}{\tau}$: 

\begin{align*}
     f(\Boldx_{t+1}) - f(\Boldx_\star) &\leq  \frac{ \| \Boldx_t - \Boldx_\star \|^2 - \| \Boldx_{t+1} - \Boldx_{\star}\|^2}{2\tau} + \frac{\tau}{2} \| \nabla f(\Boldx_t) + \Boldb_t + \boldsymbol{\epsilon}_t \|^2 - \langle \Boldb_t + \boldsymbol{\epsilon}_t, \Boldx_t - \Boldx_\star \rangle + \\
     & \quad  \tau \| \Boldb_t + \boldsymbol{\epsilon}_t\| \| \nabla f(\Boldx_t) + \Boldb_t + \boldsymbol{\epsilon}_t\|\\
     &\leq  \frac{ \| \Boldx_t - \Boldx_\star \|^2 - \| \Boldx_{t+1} - \Boldx_{\star}\|^2}{2\tau}  + \tau \| \nabla f(\Boldx_t) + \Boldb_t + \boldsymbol{\epsilon}_t \|^2  + \frac{\tau}{2}\| \Boldb_t+\boldsymbol{\epsilon}_t\|^2  - \langle \Boldb_t + \boldsymbol{\epsilon}_t, \Boldx_t - \Boldx_\star \rangle \\
     &\stackrel{(i)}{\leq}  \frac{ \| \Boldx_t - \Boldx_\star \|^2 - \| \Boldx_{t+1} - \Boldx_{\star}\|^2}{2\tau}  + 2\tau \| \nabla f(\Boldx_t)  \|^2  + 5\tau\| \Boldb_t\|^2 + 5\tau \| \boldsymbol{\epsilon}_t\|^2  - \langle \Boldb_t + \boldsymbol{\epsilon}_t, \Boldx_t - \Boldx_\star \rangle \\
     &\leq \frac{ \| \Boldx_t - \Boldx_\star \|^2 - \| \Boldx_{t+1} - \Boldx_{\star}\|^2}{2\tau}  + 2\tau \| \nabla f(\Boldx_t)  \|^2  + 5\tau\| \Boldb_t\|^2 + 5\tau \| \boldsymbol{\epsilon}_t\|^2  + \| \Boldb_t\|_1\|\Boldx_t - \Boldx_\star\|_\infty - \langle \boldsymbol{\epsilon}_t, \Boldx_t - \Boldx_\star\rangle
\end{align*}

Inequality $(i)$ is a result of a repeated use of Young's inequality. The last inequality is a result of Cauchy-Schwartz.

\end{proof}

\section{Stochastic Gradients Analysis}

We will make use of the following concentration inequality:

\begin{lemma}[Uniform empirical Bernstein bound]
\label{lem:uniform_emp_bernstein}
In the terminology of \citet{howard2018uniform}, let $S_t = \sum_{i=1}^t Y_i$ be a sub-$\psi_P$ process with parameter $c > 0$ and variance process $W_t$. Then with probability at least $1 - \delta$ for all $t \in \mathbb{N}$
\begin{align*}
    S_t &\leq  1.44 \sqrt{(W_t \vee m) \left( 1.4 \ln \ln \left(2 \left(\frac{W_t}{m} \vee 1\right)\right) + \ln \frac{5.2}{\delta}\right)}\\
   & \qquad + 0.41 c  \left( 1.4 \ln \ln \left(2 \left(\frac{W_t}{m} \vee 1\right)\right) + \ln \frac{5.2}{\delta}\right)
\end{align*}
where $m > 0$ is arbitrary but fixed.
\end{lemma}
\begin{proof}
Setting $s = 1.4$ and $\eta = 2$ in the polynomial stitched  boundary in Equation~(10) of \citet{howard2018uniform} shows that $u_{c, \delta}(v)$ is a sub-$\psi_G$ boundary for constant $c$ and level $\delta$ where 
\begin{align*}
    u_{c, \delta}(v) &= 1.44 \sqrt{(v \vee 1) \left( 1.4 \ln \ln \left(2 (v \vee 1)\right) + \ln \frac{5.2}{\delta}\right)}\\
   &\quad  + 1.21 c  \left( 1.4 \ln \ln \left( 2 (v \vee 1)\right) + \ln \frac{5.2}{\delta}\right).
\end{align*}
By the boundary conversions in Table~1 in \citet{howard2018uniform} $u_{c/3, \delta}$ is also a sub-$\psi_P$ boundary for constant $c$ and level $\delta$. The desired bound then follows from Theorem~1 by \citet{howard2018uniform}.
\end{proof}

The following estimation bound holds:

\begin{lemma}\label{lemma::upper_lower_bounds_N_t_s_a}
Let $\{(s_\ell, a_\ell, s_\ell')\}_{\ell=1}^\infty$ be samples generated as above. Let $N_t(s,a) = \sum_{\ell = 1}^t \mathbf{1}(s_\ell, a_\ell = s,a)$. Let $\delta \in (0,1)$. With probability at least $1-(2|\mathcal{S}||\mathcal{A}|\delta)$ for all $t$ such that $ \ln(2t) + \ln \frac{5.2}{\delta} \leq \frac{t\beta}{6}$ and for all $s,a \in\mathcal{S}\times \mathcal{A}$ simultaneously:
\begin{equation*}
    N_t(s,a) \in \left[   \frac{t\Boldq_{s,a}}{4}   , \frac{7t\Boldq_{s,a}}{4} \right]
\end{equation*}

Additionally define $\widehat{\Boldq}_{s,a} = \frac{N_t(s,a)}{t}$. For any $\epsilon \in (0,1)$ with probability at least $1-(2|\mathcal{S}||\mathcal{A}|\delta)$ and for all $t$ such that $\frac{t}{\ln\ln(2t)} \geq \frac{1+\ln \frac{5.2}{\delta}}{\beta \epsilon^2}$:
\begin{equation*}
 \left|   \widehat{\Boldq}_{s,a} - \Boldq_{s,a} \right|\leq  3.69 \epsilon \Boldq_{s,a}.
\end{equation*}

\end{lemma}

\begin{proof}
We start by producing a lower bound for $N_t(s,a)$. Consider the martingale sequence $Z_{s,a}(\ell) = \mathbf{1}(s_\ell = s, a_\ell = a) - \Boldq_{s,a}$ with the variance process $V_t = \sum_{\ell=1}^t \mathbb{E}\left[ Z_{s,a}^2(\ell)  | \mathcal{F}_{\ell-1}\right]$ satisfying $\mathbb{E}[ Z_{s,a}^2(\ell)| \mathcal{F}_{\ell-1} ]\leq \Boldq_{s,a} $. The martingale process $Z_{s,a}(\ell)$ satisfies the sub-$\psi_P$ condition of~\cite{howard2018uniform} with constant $c =1 $ (see Bennet case in Table 3 of~\cite{howard2018uniform}). By Lemma~\ref{lem:uniform_emp_bernstein}, and setting $m = \Boldq_{s,a}$ we conclude that with probability at least $1-\delta$ for all $t \in \mathbb{N}$ :
\begin{align}
    N_t(s,a) &\geq t \Boldq_{s,a} - 1.44\sqrt{ \Boldq_{s,a} t\left(    \ln \ln( 2t) + \ln\frac{5.2}{\delta} \right)} - 0.41\left( 1.4 \ln \ln(2t) + \ln \frac{5.2}{\delta}  \right) \label{equation::lower_bound_N_t}\\
    &\stackrel{(i)}{\geq} t\Boldq_{s,a} - \frac{t\Boldq_{s,a}}{2} - \frac{3}{2}\left( \ln\ln(2t) + \ln\frac{5.2}{\delta}\right)\notag\\
    &= \frac{t\Boldq_{s,a}}{2} - \frac{3}{2}\left( \ln\ln(2t) + \ln\frac{5.2}{\delta}\right)\notag
\end{align}
Inequality $(i)$ holds because $\sqrt{\Boldq_{s,a} t \left(\ln\ln(2t) + \ln\frac{5.2}{\delta}     \right) } \leq \frac{\Boldq_{s,a}t}{2} + \frac{\ln\ln(2t) + \ln\frac{5.2}{\delta}  }{2}$. As a consequence of Assumption~\ref{assumption::lower_bound_q} we can infer that with probability at least $1-\delta$ for all $t$ such that $ \ln\ln(2t) + \ln\frac{5.2}{\delta} \leq  \frac{t\beta}{6} \leq  \frac{t\Boldq_{s,a}}{6}$:

\begin{equation*}
    N_t(s,a) \geq \frac{t\Boldq_{s,a}}{4} 
\end{equation*}
The same sequence of inequalities but inverted implies the upper bound result. The last result is a simple consequence of the union bound. To obtain the stronger bound we start by noting that since $\frac{t}{\ln\ln(2t)} \geq \frac{1+\ln \frac{5.2}{\delta}}{\beta \epsilon^2} \geq \frac{1+\ln \frac{5.2}{\delta}}{\Boldq_{s,a} \epsilon^2} $ for all $(s,a)$ we can transform Equation~\ref{equation::lower_bound_N_t} as:

\begin{align*}
     N_t(s,a) &\geq t \Boldq_{s,a} - 1.44\sqrt{ \Boldq_{s,a} t\left(    \ln \ln( 2t) + \ln\frac{5.2}{\delta} \right)} - 0.41\left( 1.4 \ln \ln(2t) + \ln \frac{5.2}{\delta}  \right) \\
     &\geq t \Boldq_{s,a} - 2.88\sqrt{ \Boldq_{s,a} t   \ln \ln( 2t) (1+ \ln\frac{5.2}{\delta} )} - 0.81 \ln \ln(2t) (1+\ln \frac{5.2}{\delta} )    \\
     &\geq t \Boldq_{s,a} - 3.69\sqrt{ \Boldq_{s,a} t  \ln \ln( 2t) (1+ \ln\frac{5.2}{\delta} )}  \\
     &\geq t\Boldq_{s,a} - 3.69 \Boldq_{s,a} \epsilon
\end{align*}

The same sequence of inequalities but inverted implie the upper bound. This finishes the proof.

\end{proof}

The gradients of $J_D(\Boldv)$ can be written as:
\begin{small}
\begin{align*}
    \left( \nabla_\Boldv J_D(\Boldv) \right)_s &= (1 - \gamma)  \BoldMu_s + \gamma \sum_{s',a} \frac{ \exp\left( \eta  \mathbf{A}^{\Boldv}_{s',a}   \right)\Boldq_{s',a} }{\mathbf{Z}} P_a(s|s') -\\
    &\sum_a \frac{ \exp\left( \eta  \mathbf{A}^{\Boldv}_{s,a}   \right)\Boldq_{s,a} }{\mathbf{Z}} , 
\end{align*}
\end{small}
Where $\mathrm{Z} = \sum_{s,a} \exp\left(\eta  \mathbf{A}^{\Boldv}_{s,a} \right)\Boldq_{s,a}$. We will work under the assumption that $\Boldq_{s,a} \propto \exp( \eta \mathbf{A}^{\Boldv'}_{s,a} ) $ for some value vector $\Boldv'$. Given a value vector $\Boldv$ we denote its induced policy $\pi^\Boldv$ as:

\begin{equation*}
    \pi^{\Boldv}(a | s) = \frac{ \exp\left( \eta  \mathbf{A}^{\Boldv}_{s,a}\right) \Boldq_{s,a}}{ \mathbf{Z}_{s}}
\end{equation*}

Where $\mathbf{Z}_{s} = \sum_a \exp\left( \eta \mathbf{A}^{\Boldv}_{s,a} \right)\Boldq_{s,a}$. If we define $\Boldq_s = \sum_a \Boldq_{s,a}$, and we define $\Boldq_{a|s} = \frac{\Boldq_{s,a}}{\Boldq_s}$ then we can write:

\begin{equation*}
    \pi^{\Boldv}(a | s) = \frac{ \exp\left( \eta  \mathbf{A}^{\Boldv}_{s,a}\right) \Boldq_{a|s}}{ \mathbf{Z}_{a|s}}
\end{equation*}

Where $\mathbf{Z}_{s} = \sum_a \exp\left( \eta \mathbf{A}^{\Boldv}_{s,a} \right)\Boldq_{a|s}$. We work under the assumption that $\Boldq_{a|s}$ is a policy, and therefore known to the learner. We start by showing how to maintain a good estimator $\hat{\mathbf{A}}^{\Boldv}_{s,a}$ using stochastic gradient descent over a quadratic objective. Let $\mathbf{W}^{\Boldv}_{s,a} = \sum_{s'} P_a(s'|s) \Boldv_{s'}$ so that $\mathbf{A}^{\Boldv}_{s,a} = \Boldr_{s,a} - \Boldv_s + \gamma \mathbf{W}_{s,a}^{\Boldv}$ where both $\mathbf{W}^{\Boldv}$ and $\widehat{\mathbf{W}}^{\Boldv}$ are seen as vectors in $\mathbb{R}^{|S|\times |A|}$. 

If we had access to an estimator $\widehat{\mathbf{W}}^\Boldv$ of $\mathbf{W}^\Boldv$ such that for some $\epsilon \in (0,1)$:
\begin{equation}\label{equation::upper_bound_norm_W}
 \| \mathbf{W}^{\Boldv} - \widehat{\mathbf{W}}^{\Boldv} \|_{\infty} \leq \epsilon.
\end{equation}
We can use $\widehat{\mathbf{W}}^{\Boldv}$ to produce an estimator of $\mathbf{A}_{s,a}^\Boldv$ via $\widehat{\mathbf{A}}_{s,a}^\Boldv = \Boldr_{s,a}- \Boldv_s + \gamma \widehat{\mathbf{W}}_{s,a}^\Boldv$ such that:
\begin{equation*}
    \| \widehat{\mathbf{A}}^{\Boldv} - \mathbf{A}^{\Boldv}\|_\infty \leq \gamma \epsilon.
\end{equation*}

We now consider the problem of estimating $\mathbf{W}^{\Boldv}$ from samples. We assume the following stochastic setting: 

\begin{enumerate}
    \item The learner receives samples $\{(s_\ell, a_\ell, s_\ell')\}_{\ell=1}^\infty$ such that $(s_\ell, a_\ell) \sim \Boldq$ while $s_\ell' \sim P_{a_\ell}(\cdot |  s_\ell)$. Let $N_t(s,a) = \sum_{\ell=1}^t \mathbf{1}(s_\ell, a_\ell = s, a)$.
    \item Define $\widehat{\mathbf{W}}^{\Boldv}_{s,a}(t) = \frac{1}{N_t(s,a)} \sum_{\ell=1}^T \mathbf{1}(s_\ell, a_\ell = s, a)\Boldv_{s_\ell'}$. Notice that for all $s,a \in \CalS\times \CalA$, the estimator's noise $\xi_{s,a}(t) = \widehat{\mathbf{W}}_{s,a}^\Boldv(t) - \mathbf{W}^\Boldv_{s,a}$ satisfies $\mathbb{E}[ \xi_{s,a}(t) | \mathcal{F}_{t-1} ] =0$ and $|\xi_{s,a}(t)| \leq 2\| \Boldv'\|_\infty$. Where $\mathcal{F}_{t-1}$ is the sigma algebra corresponding to all the algorithmic choices up to round $t-1$.
\end{enumerate}

\begin{lemma}
Let $\{ (s_\ell, a_\ell, s_\ell')\}_{\ell=1}^\infty$ samples generated as above. Let $\widehat{\mathbf{W}}^{\Boldv}(t)$ be the empirical estimator of $\mathbf{W}^{\Boldv}$ defined as:
\begin{equation*}
    \widehat{\mathbf{W}}^{\Boldv}_{s,a}(t) = \frac{1}{N_t(s,a)} \sum_{\ell=1}^t \mathbf{1}(s_\ell, a_\ell = s, a) \Boldv_{s_\ell'}.
\end{equation*}
Where $N_t(s,a) = \sum_{\ell=1}^t \mathbf{1}(s_\ell, a_\ell = s, a)$. Let $\delta \in (0,1)$. With probability at least $1-(2|S||A|)\delta$ for all $t\in \mathbb{N}$ such that $\ln\ln(2t) + \ln\frac{5.2}{\delta} \leq  \frac{t\beta}{6}$ and for all $(s,a) \in \CalS $ simultaneously:
\begin{align*}
    |   \mathbf{W}^{\Boldv}_{s,a}-  \widehat{\mathbf{W}}^{\Boldv}_{s,a}(t)| \leq 8 \| \Boldv\|_\infty \left(\sqrt{ \frac{\ln \ln(2t) + \ln \frac{10.4}{\delta}}{t \beta} } + \frac{ \ln \ln(2t) + \ln \frac{10.4}{\delta}}{t \beta} \right) .
\end{align*}

\end{lemma}

\begin{proof}
Consider the martingale difference sequence $X_{s,a}(\ell) = \mathbf{1}( s_\ell, a_\ell = s, a) \left( \mathbf{W}^{\Boldv}_{s,a}- \Boldv_{s_\ell'}  \right)$. Notice that for all $s,a \in \CalS \times \mathcal{A}$ $| X_{s,a}(t) | \leq 2\| \Boldv'\|_\infty $ The process $S_t =\sum_{\ell=1}^t X_{s,a}(\ell) $ with variance process $W_t = \sum_{\ell=1}^t \mathbb{E}\left[X^2_{s,a}(\ell) |\mathcal{F}_{\ell-1}\right]$ satisfies the sub-$\psi_P$ condition of \cite{howard2018uniform} with constant $c=2\|\Boldv'\|_\infty$ (see Bennet case in Table 3 of \cite{howard2018uniform}). By Lemma~\ref{lem:uniform_emp_bernstein} the bound:
\begin{equation*}
    S_t \leq 1.44 \sqrt{(W_t \vee m ) \left(1.4 \ln \ln \left( 2(W_t / m \vee 1)    \right) + \ln \frac{5.2}{\delta} \right) } + 0.81 \|\Boldv\|_\infty  \left( 1.4 \ln \ln \left(2 \left(\frac{W_t}{m} \vee 1\right)\right) + \ln \frac{5.2}{\delta}\right)
\end{equation*}
holds for all $t \in \mathbb{N}$ with probability at least $1-\delta$. Notice that $\mathbb{E}[ X^2_{s,a}(\ell) | \mathcal{F}_{\ell-1} ] \leq 4\| \Boldv\|_\infty^2 \mathrm{Var}_{\Boldq}(\mathbf{1}_{s,a} ) = 4\| \Boldv\|_\infty^2 \Boldq_{s,a}(1-\Boldq_{s,a}) \leq \Boldq_{s,a}\| \Boldv\|_\infty^2$ and therefore $W_t \leq t \Boldq_{s,a}\| \Boldv\|_\infty^2$. We set $m =\Boldq_{s,a} \| \Boldv\|_\infty^2$. And obtain that with probability $1-\delta$ and for all $t \in \mathbb{N}$:
\begin{align}
    \left| \underbrace{\frac{1}{N_t(s,a) }\sum_{\ell=1}^t \mathbf{1}(s_\ell = s, a_\ell = 1) \Boldv_{s_\ell'}}_{\widehat{\mathbf{W}}_{s,a}^{\Boldv}(t)} - \mathbf{W}_{s,a}^{\Boldv} \right| &\leq \frac{1}{N_t(s,a)}   \Big(    1.44\|\Boldv\|_\infty\sqrt{\Boldq_{s,a} t \left(    \ln \ln( 2t) + \ln\frac{10.4}{\delta} \right)} + \notag\\
    &0.81\| \Boldv\|_\infty \left(  1.4\ln \ln(2t) + \ln \frac{10.2}{\delta} \right)  \Big) \label{equation::initial_bound_closeness}
\end{align}

As a consequence of Lemma~\ref{lemma::upper_lower_bounds_N_t_s_a} we know that with probability at least $1-\delta$ for all $t$ such that $ \ln\ln(2t) + \ln\frac{5.2}{\delta} \leq  \frac{t\beta}{6} \leq  \frac{t\Boldq_{s,a}}{6}$:
\begin{equation*}
    N_t(s,a) \geq \frac{t\Boldq_{s,a}}{4} 
\end{equation*}
Plugging this into Equation~\ref{equation::initial_bound_closeness} and applying a union bound over all $s,a \in \CalS \times \mathcal{A}$ yields that for all $t$ such that $ \ln\ln(2t) + \ln\frac{5.2}{\delta} \leq  \frac{t\beta}{6} \leq  \frac{t\Boldq_{s,a}}{6}$ and with probability $1-2|S||A|\delta$ for all $s,a \in \CalS $ simultaneously:
\begin{align*}
    |  \mathbf{W}_{s,a}^{\Boldv}- \widehat{\mathbf{W}}_{s,a}^{\Boldv}(t) | &\leq \frac{4}{t\Boldq_{s,a}}   \Big(    1.44\|\Boldv\|_\infty\sqrt{ t \Boldq_{s,a}  \ln \ln( 2t) + t\ln\frac{10.4}{\delta}} + 0.81\| \Boldv'\|_\infty \left(  1.4\ln \ln(2t) + \ln \frac{10.4}{\delta} \right)  \Big)\\
    &\leq 8 \| \Boldv\|_\infty \left(\sqrt{ \frac{\ln \ln(2t) + \ln \frac{10.4}{\delta}}{t \Boldq_{s,a}} } + \frac{ \ln \ln(2t) + \ln \frac{10.4}{\delta}}{t \Boldq_{s,a}} \right) \\
    &\leq 8 \| \Boldv\|_\infty \left(\sqrt{ \frac{\ln \ln(2t) + \ln \frac{10.4}{\delta}}{t \beta} } + \frac{ \ln \ln(2t) + \ln \frac{10.4}{\delta}}{t \beta} \right).
\end{align*} 
The result follows.
\end{proof}

We can now derive a concentration result for $\widehat{\mathbf{A}}_{s,a}^{\Boldv}(t) = \Boldr_{s,a}- \Boldv_s + \gamma \widehat{\mathbf{W}}_{s,a}^\Boldv(t)$, the advantage estimator resulting from $\widehat{\mathbf{W}}_{s,a}^{\Boldv}(t)$:
\begin{corollary}\label{corollary::sample_complexity_advantage}
Let $\delta \in (0,1)$. With probability at least $1-(2|S||A|)\delta$ for all $t\in \mathbb{N}$ such that $\ln\ln(2t) + \ln\frac{5.2}{\delta} \leq  \frac{t\beta}{6}$ and for all $(s,a) \in \CalS $ simultaneously:
\begin{align*}
    | \mathbf{A}_{s,a}^{\Boldv}- \widehat{\mathbf{A}}_{s,a}^{\Boldv}(t) | \leq 8\gamma \| \Boldv\|_\infty \left(\sqrt{ \frac{\ln \ln(2t) + \ln \frac{10.4}{\delta}}{t \beta} } + \frac{ \ln \ln(2t) + \ln \frac{10.4}{\delta}}{t \beta} \right) .
\end{align*}
And therefore:
\begin{align*}
    | \mathbf{A}_{s,a}^{\Boldv}- \widehat{\mathbf{A}}_{s,a}^{\Boldv}(t) | \leq 16\gamma \| \Boldv\|_\infty \sqrt{ \frac{\ln \ln(2t) + \ln \frac{10.4}{\delta}}{t \beta} } 
\end{align*}

\end{corollary}

 \subsection{Estimating the Gradients}

\begin{lemma}\label{lemma::exponential_approximate}
If $\xi \in \mathbb{R}$ such that $|\xi| \leq \epsilon < 1$, and $y \in \mathbb{R}$, then:
\begin{equation*}
 \exp\left( y  \right)(1-\epsilon) \leq     \exp\left( y + \xi \right) \leq   \exp\left( y  \right)(1+2\epsilon) 
\end{equation*}
\end{lemma}

\begin{proof}
Notice that for $\epsilon \in (0,1)$:
\begin{equation*}
   \exp( \epsilon  ) \leq 1+2\epsilon, \quad \text{ and } 1-\epsilon \leq \exp( -\epsilon ).
\end{equation*}
The result follows by noting that:
\begin{equation*}
\exp(y) \exp(-|\xi|) \leq    \exp(y + \xi) \leq \exp(y) \exp(|\xi|). 
\end{equation*}
\end{proof}

A simple consequence of Lemma~\ref{lemma::exponential_approximate} is the following:

\begin{lemma}\label{lemma::approximation_result_plugin_B}
Let $\epsilon \in (0,1/2)$. If $\mathbf{C}, \widehat{\mathbf{C}} \in \mathbb{R}^{|\mathcal{S} | \times |\mathcal{A}| } $ and $\widehat{\Boldb}, \Boldb \in \mathbb{R}_{+}^{|\mathcal{S}|\times |\mathcal{A}|}$ are two vectors satisfying:
\begin{equation*}
    \| \widehat{\mathbf{C}} - \mathbf{C} \|_\infty \leq \epsilon, \qquad | \widehat{\Boldb}_{s,a} -\Boldb_{s,a}  | \leq \epsilon \Boldb_{s,a}. 
\end{equation*}
For all $s,a \in \mathcal{S}\times \mathcal{A}$ define $\mathbf{B}_{s,a} = \frac{\exp(\mathbf{C}_{s,a}) }{\mathbf{Z}}$ and $\widehat{\mathbf{B}}_{s,a} = \frac{\exp(\widehat{\mathbf{C}}_{s,a} )}{\widehat{\mathbf{Z}}}$ where $\mathbf{Z} = \sum_{s,a} \exp( \mathbf{C}_{s,a})\Boldb_{s,a}$ and $\widehat{\mathbf{Z}} = \sum_{s,a} \exp( \widehat{\mathbf{C}}_{s,a})\widehat{\Boldb}_{s,a}$:
\begin{equation*}
     \left|      \widehat{\mathbf{B}}_{s,a}- \mathbf{B}_{s,a}    \right| \leq 38\epsilon  \mathbf{B}_{s,a}\leq 38\epsilon.
\end{equation*}
\end{lemma}

\begin{proof}

Let's define an intermediate $\tilde{\mathbf{B}}_{s,a} = \frac{\exp\left( \mathbf{C}_{s,a} \right) \widehat{\Boldb}_{s,a} }{\tilde{\mathbf{Z}} }$ where $\tilde{\mathbf{Z}} = \sum_{s,a} \exp\left( \mathbf{C}_{s,a}\right)\widehat{\Boldb}_{s,a}$. By Lemma~\ref{lemma::exponential_approximate} we can conclude that for any $s,a \in \mathcal{S}\times\mathcal{A}$:
\begin{equation*}
    \widetilde{\mathbf{B}}_{s,a} \frac{ 1-\epsilon}{1+2\epsilon} \leq \widehat{\mathbf{B}}_{s,a} \leq \frac{1+2\epsilon}{1-\epsilon} \widetilde{\mathbf{B}}_{s,a}%
\end{equation*}

And therefore:
\begin{equation*}
    \widehat{\mathbf{B}}_{s,a}, \widetilde{\mathbf{B}}_{s,a} \in \left[ \widetilde{\mathbf{B}}_{s,a}\frac{1-\epsilon}{1+2\epsilon}, \frac{1+2\epsilon}{1-\epsilon} \widetilde{\mathbf{B}}_{s,a}          \right]
 \end{equation*}
 Which in turn implies that:
 \begin{equation*}
     \left|      \widehat{\mathbf{B}}_{s,a}- \widetilde{\mathbf{B}}_{s,a}    \right| \leq \left( \frac{ 1+2\epsilon }{ 1-\epsilon}-\frac{ 1-\epsilon }{ 1+2\epsilon} \right)  \widetilde{\mathbf{B}}_{s,a}  \leq 15\epsilon \widetilde{\mathbf{B}}_{s,a}.
 \end{equation*}

We now bound $| \widetilde{\mathbf{B}}_{s,a} - \mathbf{B}_{s,a} |$. By assumption for all $s,a \in \mathcal{S}\times \mathcal{A}$, it follows that $\widehat{\Boldb}_{s,a}(1-\epsilon) \leq \Boldb_{s,a} \leq \widehat{\Boldb}_{s,a} (1+\epsilon)$ and therefore:

\begin{equation*}
   \frac{   \mathbf{B}_{s,a}}{1+\epsilon} \leq \widetilde{\mathbf{B}}_{s,a} \leq   \frac{ \mathbf{B}_{s,a}}{1-\epsilon}
\end{equation*}

And therefore:
\begin{equation*}
    \widetilde{\mathbf{B}}_{s,a}, \mathbf{B}_{s,a} \in \left[  \frac{\mathbf{B}_{s,a}}{1+\epsilon}, \frac{\mathbf{B}_{s,a}}{1-\epsilon}  \right].
\end{equation*}

Hence:

\begin{equation*}
    \left| \widetilde{\mathbf{B}}_{s,a} - \mathbf{B}_{s,a} \right| \leq \left( \frac{1}{1-\epsilon} - \frac{1}{1+\epsilon} \right) \mathbf{B}_{s,a}\leq \frac{8}{3}\epsilon \mathbf{B}_{s,a}.  
\end{equation*}

And therefore:

\begin{align*}
    | \widehat{\mathbf{B}}_{s,a} - \mathbf{B}_{s,a} | \leq | \widehat{\mathbf{B}}_{s,a} - \widetilde{\mathbf{B}}_{s,a} | + | \widetilde{\mathbf{B}}_{s,a} - \mathbf{B}_{s,a} | \leq 15\epsilon \widetilde{\mathbf{B}}_{s,a} + \frac{8}{3}\epsilon \mathbf{B}_{s,a} \leq \left(15\epsilon(1+\frac{8}{3}\epsilon) + \frac{8}{3}\epsilon\right) \mathbf{B}_{s,a} \leq  38\epsilon\mathbf{B}_{s,a}.
\end{align*}

The result follows.

\end{proof}

If we set $\mathbf{C} = \eta \mathbf{A}^{\Boldv} , \widehat{\mathbf{C}} = \eta \widehat{\mathbf{A}}^{\Boldv} $ we obtain the following corollary of Lemma~\ref{lemma::approximation_result_plugin_B}:
\begin{corollary}\label{corollary::approximation_result_plugin_B_1}
Let $\epsilon \in (0, 1/2)$. If $\widehat{\mathbf{A}}^{\Boldv}$ and $\widehat{\Boldq}$ satisfies:
\begin{equation*}
    \| \widehat{\mathbf{A}}^{\Boldv} - \mathbf{A}^{\Boldv}\|_\infty \leq \epsilon, \quad \text{ and } \quad |\widehat{\Boldq}_{s,a} - \Boldq_{s,a} | \leq \epsilon \Boldq_{s,a}
\end{equation*}
Then:
\begin{equation*}
    \left|      \widehat{\mathbf{B}}_{s,a}^{\Boldv}- \mathbf{B}_{s,a}^{\Boldv}     \right|  \leq 111\eta\epsilon  \mathbf{B}_{s,a}^{\Boldv} \leq 111\eta\epsilon.
\end{equation*}
\end{corollary}

We can combine the sample complexity results of Corollary~\ref{corollary::sample_complexity_advantage} and the approximation results of Corollary~\ref{corollary::approximation_result_plugin_B_1} and Lemma~\ref{lemma::upper_lower_bounds_N_t_s_a} to obtain:

\begin{corollary}\label{corollary::approximation_B_t}
If $\delta, \xi \in (0,1)$, with probability at least $1-(4|\mathcal{S}||\mathcal{A}|\delta)$ for all $t$ such that:
\begin{equation*}
\frac{t}{\ln\ln(2t) } \geq \frac{120(\ln\frac{10.4}{\delta} +1)}{\beta \xi^2}\max\left( 480 \eta^2 \gamma^2 \| \Boldv \|_\infty^2  , 1 \right)    
\end{equation*}
 then for all $(s,a) \in \mathcal{S}\times \mathcal{A}$ simultaneously:
\begin{equation*}
     \left|      \widehat{\mathbf{B}}_{s,a}^{\Boldv}(t)- \mathbf{B}_{s,a}^{\Boldv}     \right| \leq \xi \mathbf{B}_{s,a}^\Boldv \leq \frac{\xi}{\beta}, \quad \text{ and }\quad \widehat{\mathbf{B}}^\Boldv_{s,a} \leq \mathbf{B}_{s,a}^\Boldv ( 1+\frac{\xi}{\beta}) \leq \frac{1}{\beta}(1+ \frac{\xi}{\beta}).
\end{equation*}

\end{corollary}

\subsection{Biased Stochastic Gradients}\label{section::biased_stochastic_gradients_appendix}

Notice that:

\begin{align*}
    \left( \nabla_\Boldv J_D(\Boldv) \right)_s &= (1 - \gamma)  \BoldMu_s + \gamma \sum_{s',a} \frac{ \exp\left( \eta  \mathbf{A}^{\Boldv}_{s',a}   \right)\Boldq_{s',a} }{\mathbf{Z}} P_a(s|s') -\sum_a \frac{ \exp\left( \eta  \mathbf{A}^{\Boldv}_{s,a}   \right)\Boldq_{s,a} }{\mathbf{Z}} \\
    &= (1 - \gamma)  \BoldMu_s + \gamma\mathbb{E}_{(s', a) \sim \Boldq , s'' \sim P_{a}(\cdot |s')} \left[ \mathbf{B}_{s',a}^\Boldv \mathbf{1}(s'' = s) \right] - \mathbb{E}_{(s',a) \sim \Boldq}\left[ \mathbf{B}_{s,a}^\Boldv \mathbf{1}(s' = s) \right]\\
    &= (1 - \gamma)  \BoldMu_s +  \mathbb{E}_{(s', a) \sim \Boldq , s'' \sim P_{a}(\cdot |s')} \left[ \mathbf{B}_{s',a}^\Boldv \left( \gamma \mathbf{1}(s'' = s) - \mathbf{1}(s' = s)\right) \right],  
\end{align*}

We now proceed to bound the bias of this estimator and prove a more fine grained version of Lemma~\ref{lemma::biased_gradient_guarantee}.

\begin{lemma}\label{lemma::biased_gradient_guarantee_appendix}
Let $\delta, \xi \in (0,1)$. With probability at least $1-\delta$ for all $t \in \mathbb{N}$ such that $$\frac{t}{\ln\ln(2t) } \geq \frac{120(\ln\frac{41.6|\mathcal{S}||\mathcal{A}|}{\delta} +1)}{\beta \xi^2}\max\left( 480 \eta^2 \gamma^2 \| \Boldv \|_\infty^2  , 1 \right)$$
the plugin estimator $\widehat{\nabla}_{\Boldv} J_D(\Boldv)$ satisfies:

\begin{align}
   \max_{u\in\{1,2,\infty\}} \left\|\widehat{\nabla}_\Boldv J_D(\Boldv) - \mathbb{E}_{s_{t+1}, a_{t+1}, s_{t+1}' }\left[    \widehat{\nabla}_\Boldv J_D(\Boldv) \Big| \widehat{\mathbf{B}}^{\Boldv}(t) \right]\right\|_u &\leq  \frac{4}{\beta}(1+\frac{\xi}{\beta}) \label{equation::bound_1_conditional}  \\
    \max_{u \in \{1,2, \infty\}}\left\| \mathbb{E}\left[ \widehat{\nabla}_{\Boldv} J_D(\Boldv)  \right] - \nabla_\Boldv J_D(\Boldv) \right\|_u &\leq 2(1+\gamma) \xi (1+\frac{\xi}{\beta}), \label{equation::bound_2_conditional}\\
    \mathbb{E}\left[ \left\| \widehat{\nabla}_\Boldv J_D(\Boldv) - \mathbb{E}_{s_{t+1}, a_{t+1}, s_{t+1}' }[    \widehat{\nabla}_\Boldv J_D(\Boldv) \Big| \widehat{\mathbf{B}}^{\Boldv}(t) ] \right\|_2^2 \Big|  \widehat{\mathbf{B}}^\Boldv(t)    \right] &\leq (1+\gamma^2)(1+4\xi)\frac{1}\beta(1+\frac{\xi}{\beta})\label{equation::bound_3_conditional}         %
\end{align}
\end{lemma}

\begin{proof}
As a consequence of Corollary~\ref{corollary::approximation_B_t}, we can conclude that for all $t$ satisfying the assumptions of the Lemma and with probability at leat $1-\delta$ simultaneously for all $(s,a) \in \mathcal{S}\times \mathcal{A}$:

\begin{equation}\label{equation::upper_bound_difference}
     \left|      \widehat{\mathbf{B}}_{s,a}^{\Boldv}(t)- \mathbf{B}_{s,a}^{\Boldv}     \right| \leq \xi \mathbf{B}_{s,a}^\Boldv(1+ \frac{\xi}{\beta}), \quad \text{ and } \quad \widehat{\mathbf{B}}^\Boldv_{s,a} \leq \mathbf{B}_{s,a}^\Boldv ( 1+\frac{\xi}{\beta}) \leq \frac{1}{\beta}(1+ \frac{\xi}{\beta}).
\end{equation}

Let's start by bounding the first term. Notice that $ \widehat{\nabla}_\Boldv J_D(\Boldv) -(1-\gamma) \boldsymbol{\mu}$ has at most $2$ nonzero entries and therefore:

\begin{align*}
     \max_{u \in \{ 1,2,\infty\}} \|\widehat{\nabla}_\Boldv J_D(\Boldv) -(1-\gamma) \boldsymbol{\mu} \|_u \leq \frac{2}{\beta}(1+\frac{\xi}{\beta}).
\end{align*}

Therefore for all $u \in \{ 1,2, \infty \}$:
\begin{equation*}
    \left\| \mathbb{E}_{s_{t+1}, a_{t+1}, s_{t+1}' }\left[ \widehat{\nabla}_\Boldv J_D(\Boldv) -(1-\gamma) \boldsymbol{\mu}  \Big| \widehat{\mathbf{B}}^\Boldv(t)\right]\right\|_u \leq \mathbb{E}_{s_{t+1}, a_{t+1}, s_{t+1}' }\left[ \| \widehat{\nabla}_\Boldv J_D(\Boldv) -(1-\gamma) \boldsymbol{\mu}  \|_u \Big| \widehat{\mathbf{B}}^\Boldv(t)\right] \leq \frac{2}{\beta}(1+\frac{\xi}{\beta}).
\end{equation*}

\begin{align*}
   \left\|  \widehat{\nabla}_\Boldv J_D(\Boldv) - \mathbb{E}\left[    \widehat{\nabla}_\Boldv J_D(\Boldv) \Big| \widehat{\mathbf{B}}^{\Boldv}(t) \right] \right \|_u &\leq   \left\|  \widehat{\nabla}_\Boldv J_D(\Boldv) - (1-\gamma) \boldsymbol{\mu} \right\|_u+\left\| \mathbb{E}\left[    \widehat{\nabla}_\Boldv J_D(\Boldv) \Big| \widehat{\mathbf{B}}^{\Boldv}(t) \right] - (1-\gamma) \boldsymbol{\mu}\right \|_u\\
   &\leq \frac{4}{\beta}(1+\frac{\xi}{\beta}) 
\end{align*}

Furthermore, notice that the following estimator of $\nabla_\Boldv J_D(\Boldv)$ is unbiased:

\begin{equation*}
 \left(\widetilde{\nabla}_{\Boldv} J_D(\Boldv) \right)_s = (1-\gamma)\boldsymbol{\mu}_s + \mathbf{B}^\Boldv_{s_{t+1}, a_{t+1}}(t)\left(\gamma   \mathbf{1}(s_{t+1}' = s) -  \mathbf{1}(s_{t+1} = s)\right). 
\end{equation*}
We conclude that for all $s\in\mathcal{S}$:
\begin{align*}
     \left(\widehat{\nabla}_{\Boldv} J_D(\Boldv) \right)_s - \left(  \widetilde{\nabla}_{\Boldv} J_D(\Boldv) \right)_s  &= \left(\gamma \mathbf{1}(s_{t+1}' = s) - \mathbf{1}(s_{t+1} = s)\right)\left(\widehat{\mathbf{B}}^\Boldv_{s_{t+1}, a_{t+1}}(t) - \mathbf{B}^\Boldv_{s_{t+1}, a_{t+1}}(t)\right)
\end{align*}

Consequently $\widehat{\nabla}_{\Boldv} J_D(\Boldv) -  \widetilde{\nabla}_{\Boldv} J_D(\Boldv)$ has at most $2$ nonzero entries. Now observe that any nonzero entry $s$ satisfies:
\begin{align*}
 \left|   \mathbb{E}\left[ \left(\widehat{\nabla}_{\Boldv} J_D(\Boldv) \right)_s \right]- \left(  \nabla_{\Boldv} J_D(\Boldv) \right)_s \right| &= \left| \mathbb{E}_{s_{t+1}, a_{t+1} \sim \Boldq}\left[ \left(\widehat{\nabla}_{\Boldv} J_D(\Boldv) \right)_s - \left(  \widetilde{\nabla}_{\Boldv} J_D(\Boldv) \right)_s \right]\right|\\
    &\leq \mathbb{E}_{s_{t+1}, a_{t+1} \sim \Boldq}\left[   \left|\gamma \mathbf{1}(s_{t+1}' = s) - \mathbf{1}(s_{t+1} = s)\right|\left|\widehat{\mathbf{B}}^\Boldv_{s_{t+1}, a_{t+1}}(t) - \mathbf{B}^\Boldv_{s_{t+1}, a_{t+1}}(t)\right|    \right] \\
    &\stackrel{(i)}{\leq} \mathbb{E}_{s_{t+1}, a_{t+1} \sim \Boldq}[\left(\gamma \mathbf{1}(s_{t+1}' = s) + \mathbf{1}(s_{t+1} = s) \right) \xi \mathbf{B}_{s_{t+1}, a_{t+1}}^\Boldv(1+ \frac{\xi}{\beta}) ] \\
    &\leq (1+\gamma) \xi (1+\frac{\xi}{\beta}) \mathbb{E}_{s_{t+1}, a_{t+1} \sim \Boldq}[  \mathbf{B}_{s_{t+1}, a_{t+1}}   ] \\
    &= (1+\gamma) \xi (1+\frac{\xi}{\beta})
\end{align*}
Inequality $(i)$ holds by the triangle inequality and Equation~\ref{equation::upper_bound_difference} and because $\mathbf{B}^\Boldv_{s,a} \geq 0$. This finishes the proof of the first result. Since $\widehat{\nabla}_{\Boldv} J_D(\Boldv) -  \widetilde{\nabla}_{\Boldv} J_D(\Boldv)$ has at most $2$ nonzero entries for all $u \in \{1,2,\infty\}$:

\begin{equation*}
    \left\|   \mathbb{E}\left[ \left(\widehat{\nabla}_{\Boldv} J_D(\Boldv) \right)_s \right]- \left(  \nabla_{\Boldv} J_D(\Boldv) \right)_s \right\|_u \leq 2(1+\gamma) \xi (1+\frac{\xi}{\beta})
\end{equation*}
The second inequality follows.

Recall that for any $s$:
\begin{equation*}
    \left(\widehat{\nabla}_{\Boldv} J_D(\Boldv) \right)_s = (1-\gamma)\boldsymbol{\mu}_s +  \widehat{\mathbf{B}}^{\Boldv}_{s_{t+1}, a_{t+1}(t) }\left(\gamma \mathbf{1}(s_{t+1}' = s) -  \mathbf{1}(s_{t+1} = s) \right). 
\end{equation*}
Observe that:
\begin{align*}
\mathbb{E}\left[ \left\| \widehat{\nabla}_\Boldv J_D(\Boldv) - \mathbb{E}[    \widehat{\nabla}_\Boldv J_D(\Boldv) \Big| \widehat{\mathbf{B}}^{\Boldv}(t) ] \right\|_2^2 \Big|  \widehat{\mathbf{B}}^\Boldv(t)    \right] &\leq \mathbb{E}\left[ \left\| \widehat{\nabla}_\Boldv J_D(\Boldv)  \right\|_2^2 \Big|  \widehat{\mathbf{B}}^\Boldv(t)    \right] \\
&= \sum_{s', a} \left( \widehat{\mathbf{B}}^\Boldv_{s',a}(t) \right)^2 \gamma^2 \Boldq_{s',a} P_a(s| s') +\\
&\sum_a \left( \widehat{\mathbf{B}}_{s,a}^\Boldv(t)\right)^2 \Boldq_{s,a} \left(1-2\gamma \right)P_a(s|s)\\
&\leq (1+\gamma^2)\mathbb{E}_{(s',a)\sim \widehat{\Boldq}(t) \widehat{\mathbf{B}}^\Boldv(t)}\left[   \widehat{\mathbf{B}}^\Boldv_{s', a}(t) \frac{\Boldq_{s',a}}{\widehat{\Boldq}_{s',a}} \right] \\
&\stackrel{(i)}{\leq} (1+\gamma^2)(1+4\xi)\frac{1}\beta(1+\frac{\xi}{\beta}). 
\end{align*}

Inequality $(i)$ follows because $\widehat{\mathbf{B}}_{s,a}\Boldq_{s,a} \leq \frac{\Boldq_{s,a} }{\widehat{\Boldq}_{s,a} } \leq (1+4\xi)$ and because by Corollary~\ref{corollary::approximation_B_t} we have that $\widehat{\mathbf{B}}_{s,a}^\Boldv \leq \frac{1}{\beta}(1+\frac{\xi}{\beta})$. 

The result follows.
\end{proof}

Combining the guarantees of Lemma~~\ref{lemma::helper_projected_sgd} and~\ref{lemma::biased_gradient_guarantee} for Algorithm~\ref{algorithm::biased_gradient_descent} applied to the objective function $J_D$:

\begin{lemma}\label{lemma::SGD_result_appendix}
Let $\xi_t= \min(\sqrt{\frac{c'}{t}}, \beta)$ for all $t$ where  $c' =   2(|\mathcal{S}|+1)^2 \eta^2 D^2 + \frac{320}{\beta^2}  + 240 $ and $D =  \frac{1}{1-\gamma}\left( 1+ \frac{ \log\frac{|\mathcal{S}||\mathcal{A}|}{\beta \rho}}{\eta} \right)$. If $n(t)$ is such that:
\begin{equation}\label{equation::lower_bound_n_t}
    \frac{n(t)}{\ln\ln(2 n(t))} \geq \frac{120 \left( \ln\frac{83.2|\mathcal{S}||\mathcal{A}| t^2}{\delta} + 1 \right)}{\beta \xi_t^2} \max\left( 280 \eta^2 \gamma^2 \| \Boldv_t \|^2_\infty, 1\right) 
\end{equation}
And $\tau_t = \frac{c}{\sqrt{t}}$ where $c = \frac{D}{2\sqrt{c'}}$ then for all $t \geq 1$ we have that with probability at least $1-2\delta$ and simulataneously for all $T \in \mathbb{N}$ :
\begin{equation*}
 J_D\left(\frac{1}{T}\sum_{t=1}^T \Boldv_t\right) \leq J_D(\Boldv_\star) + \frac{36D}{\sqrt{T}}\max\left(\left( |\mathcal{S}|+1\right)\eta D, \frac{18 + 16\sqrt{\ln\ln(2T) + \ln\frac{5.2}{\delta}}}{\beta}, 16 \right) %
\end{equation*}
\end{lemma}

\begin{proof}
We will make use of Lemmas~\ref{lemma::biased_gradient_guarantee} and~\ref{lemma::helper_projected_sgd}. We identify $\boldsymbol{\epsilon}_t = \widehat{\nabla}_\Boldv J_D(\Boldv_t) - \mathbb{E}\left[    \widehat{\nabla}_\Boldv J_D(\Boldv_t) \Big| \widehat{\mathbf{B}}^{\Boldv_t}(n(t)) \right]$ and $\Boldb_t =  \nabla_\Boldv J_D(\Boldv_t) - \mathbb{E}\left[ \widehat{\nabla}_{\Boldv_t} J_D(\Boldv_t) \Big| \widehat{\mathbf{B}}^{\Boldv_t}(n(t)) \right]$. As a consequence of Cauchy-Schwartz and Lemma~\ref{lemma::biased_gradient_guarantee} we see that if $n(t)$ is such that:

\begin{equation*}
    \frac{n(t)}{\ln\ln(2 n(t))} \geq \frac{120 \left( \ln\frac{83.2|\mathcal{S}||\mathcal{A}| t^2}{\delta} + 1 \right)}{\beta \xi_t^2} \max\left( 280 \eta^2 \gamma^2 \| \Boldv_t \|^2_\infty, 1\right) 
\end{equation*}

Then for all $t$ with probability at least $1-\frac{\delta}{2t^2}$ the bounds in Equations~\ref{equation::bound_1_conditional},~\ref{equation::bound_2_conditional}, and~\ref{equation::bound_3_conditional} in Lemma~\ref{lemma::biased_gradient_guarantee} hold and therefore:
\begin{equation*}
    \left| \langle \boldsymbol{\epsilon}_t, \Boldv_t - \Boldv_\star \rangle \right|\leq \| \Boldv_t - \Boldv_\star\|_\infty \| \boldsymbol{\epsilon}_t \|_1 \leq   \frac{1}{1-\gamma}\left( 1+ \frac{ \log\frac{|\mathcal{S}||\mathcal{A}|}{\beta \rho}}{\eta} \right)  \frac{4}{\beta}(1+\frac{\xi_t}{\beta}) \stackrel{(i)}{\leq} \underbrace{\frac{1}{1-\gamma}\left( 1+ \frac{ \log\frac{|\mathcal{S}||\mathcal{A}|}{\beta \rho}}{\eta} \right)  \frac{8}{\beta}}_{:=U_1}.
\end{equation*}

Where inequality $(i)$ holds by the assumption $\xi_t \leq \beta$. Notice that $X_t =\langle \boldsymbol{\epsilon}_t, \Boldv_t -\Boldv_\star\rangle   $ is a martingale difference sequence. A simple application of Lemma~\ref{lem:uniform_emp_bernstein} yields that with probability at least $1-\delta$ for all $t\in\mathbb{N}$: 

\begin{equation}\label{equation::bounding_the_noise_terms}
    -\sum_{t=1}^T \langle \boldsymbol{\epsilon}_t, \Boldx_t -\Boldx_\star\rangle \leq  2U_1 \sqrt{t\left( \ln \frac{2t^2}{\delta} \right)} %
\end{equation}

Similarly observe that for all $t$ with probability at least $1-\frac{\delta}{2t^2}$, since the bounds in Equations~\ref{equation::bound_1_conditional},~\ref{equation::bound_2_conditional}, and~\ref{equation::bound_3_conditional} in Lemma~\ref{lemma::biased_gradient_guarantee} hold, %

\begin{equation}\label{equation::bounding_the_bias}
    \| \Boldb_t \|_1 = \left\|  \nabla_{\Boldv} J_D(\Boldv_t) - \mathbb{E}\left[ \widehat{\nabla}_{\Boldv_t} J_D(\Boldv_t) \Big | \widehat{\mathbf{B}}^{\Boldv_t}(n(t))      \right] \right\|_1 \leq   2(1+\gamma) \xi_t (1+\frac{\xi_t}{\beta})
\end{equation}

Notice that similarly and for all $t$ with probability at least $1-\frac{\delta }{2t^2}$, since the bounds in Equations~\ref{equation::bound_1_conditional},~\ref{equation::bound_2_conditional}, and~\ref{equation::bound_3_conditional} in Lemma~\ref{lemma::biased_gradient_guarantee} hold:

\begin{equation*}
  \| \boldsymbol{\epsilon}_t \|_2^2 \leq   \frac{16}{\beta^2}\left(1+\frac{\xi_t}{\beta}\right)^2,\quad \text{ and } \quad \| \Boldb_t \|_2^2 \leq 4(1+\gamma)^2 \xi_t^2 (1+\frac{\xi_t}{\beta})^2
\end{equation*}

Finally we show a bound on the $l_2$ norm of the gradient of $J_D$. Since $\Boldv_\star \in \mathcal{D} = \left\{ \Boldv \text{ s.t. } \| \Boldv \|_\infty \leq \frac{1}{1-\gamma}\left( 1+ \frac{ \log\frac{|\mathcal{S}||\mathcal{A}|}{\beta \rho}}{\eta} \right) \right\}$. Recall that by Lemma~\ref{lemma::RL_smoothness_dual}, we have that $J_D$ is $(|\mathcal{S}| +1)\eta$-smooth in the $\| \cdot \|_\infty$ norm. Therefore by Lemma~\ref{lemma::supporting_lemma_bound_gradient}:

\begin{equation*}
    \| \nabla J_D(\Boldv_t) \|_1\leq (|\mathcal{S}|+1)\frac{\eta}{1-\gamma}\left( 1 + \frac{ \log \frac{|\mathcal{S}||\mathcal{A}|}{\beta \rho}}{\eta} \right)  
\end{equation*}

Since $\| \nabla J_D(\Boldv_t) \|_2 \leq \| \nabla J_D(\Boldv_t) \|_1$ this in turn implies that:

\begin{equation*}
    \| \nabla J_D(\Boldv_t) \|^2_2\leq (|\mathcal{S}|+1)^2\frac{\eta^2}{(1-\gamma)^2}\left( 1 + \frac{ \log \frac{|\mathcal{S}||\mathcal{A}|}{\beta \rho}}{\eta} \right)^2.  
\end{equation*}

We now invoke the guarantees of Lemma~\ref{lemma::helper_projected_sgd} to show that with probability $1-2\delta$ and simultaneously for all $T \in \mathbb{N}$:

\begin{align*}
    \sum_{t=1}^T J_D(\Boldv_t) - J_D(\Boldv_\star) &\leq \sum_{t=1}^T \frac{ \|  \Boldv_t - \Boldv_\star \|^2 - \| \Boldv_{t+1} - \Boldv_\star\|^2 }{ 2\tau_t } +\\
    &\tau_t \left( 2(|\mathcal{S}|+1)^2\frac{\eta^2}{(1-\gamma)^2}\left( 1 + \frac{ \log \frac{|\mathcal{S}||\mathcal{A}|}{\beta \rho}}{\eta} \right)^2 +  \frac{80}{\beta^2}\left(1+\frac{\xi_t}{\beta}\right)^2 +  20(1+\gamma)^2 \xi_t^2 (1+\frac{\xi_t}{\beta})^2 \right) + \\
    &2(1+\gamma) \xi_t (1+\frac{\xi_t}{\beta}) \times \frac{1}{1-\gamma}\left( 1+ \frac{ \log\frac{|\mathcal{S}||\mathcal{A}|}{\beta \rho}}{\eta} \right) + 2 U_1 \sqrt{T\left( \ln \frac{2t^2}{\delta} \right)} \\
    &\stackrel{(i)}{\leq} \sum_{t=1}^T \frac{ \|  \Boldv_t - \Boldv_\star \|^2 - \| \Boldv_{t+1} - \Boldv_\star\|^2 }{ 2\tau_t } + \tau_t \left( 2(|\mathcal{S}|+1)^2 \eta^2 D^2 + \frac{320}{\beta^2}  + 240 \right) +  8 D \xi_t + \\
    &2U_1 \sqrt{T\left(\ln \frac{2t^2}{\delta} \right)} %
\end{align*}

Recall that $U_1 = \frac{1}{1-\gamma}\left( 1+ \frac{ \log\frac{|\mathcal{S}||\mathcal{A}|}{\beta \rho}}{\eta} \right)  \frac{8}{\beta}= \frac{8D}{\beta}$ and where $D = \frac{1}{1-\gamma}\left( 1+ \frac{ \log\frac{|\mathcal{S}||\mathcal{A}|}{\beta \rho}}{\eta} \right)$. Inequality $(i)$ holds because $\xi_t \leq \beta$ and because $\gamma \leq 1$. Let $\tau_t = \frac{c}{\sqrt{t}}$ for some constant to be specified later and let's analyze the terms in the sum above that depend on these $\tau_t$ values:

\begin{align*}
    \sum_{t=1}^T \frac{ \|  \Boldv_t - \Boldv_\star \|^2 - \| \Boldv_{t+1} - \Boldv_\star\|^2 }{ 2\tau_t }  &= -\frac{ \| \Boldv_{T+1}-\Boldv_\star\|^2}{2\tau_{T}} + \frac{1}{2c}\sum_{t=1}^T \| \Boldv_t - \Boldv_\star \|^2\left( \sqrt{t} - \sqrt{t-1} \right) \\
    &\leq \frac{D^2}{2c} \sqrt{T}
\end{align*}

The second term can be bounded as:

\begin{align*}
    \sum_{t=1}^T \tau_t c' = cc' \sum_{t=1}^T \frac{1}{\sqrt{t}} \leq cc' 2 \sqrt{T}
\end{align*}

Where $c' =   2(|\mathcal{S}|+1)^2 \eta^2 D^2 + \frac{320}{\beta^2}  + 240 $. Therefore under this assumption we obtain:

\begin{align*}
      \sum_{t=1}^T J_D(\Boldv_t) - J_D(\Boldv_\star) &\leq \frac{D^2}{2c} \sqrt{T} +  cc' 2 \sqrt{T} + 8D \left(\sum_{t=1}^T \xi_t \right) +   2 U_1 \sqrt{T\left( \ln \frac{2t^2}{\delta} \right)}. %
\end{align*}

The minimizing choice for $c$ equals $c =\frac{ D}{2\sqrt{c'}}$. And in this case:

\begin{equation*}
      \sum_{t=1}^T J_D(\Boldv_t) - J_D(\Boldv_\star) \leq 2D\sqrt{c'T} + 8D \left(\sum_{t=1}^T \xi_t \right) + 2 U_1 \sqrt{T\left( \ln \frac{2t^2}{\delta} \right)} %
\end{equation*}

If we set $\xi_t =\min(\sqrt{\frac{c'}{t}}, \beta)$ we get:

\begin{align*}
      \sum_{t=1}^T J_D(\Boldv_t) - J_D(\Boldv_\star) &\leq 18D\sqrt{c'T} + 2 U_1 \sqrt{T\left( \ln \frac{2t^2}{\delta} \right)} \\
      &\stackrel{(i)}{\leq} 36D\max\left(\left( |\mathcal{S}|+1\right)\eta D, \frac{18}{\beta}, 16 \right)\sqrt{T} + 2 U_1 \sqrt{T\left( \ln \frac{2t^2}{\delta} \right)} \\
      &\leq 36D\max\left(\left( |\mathcal{S}|+1\right)\eta D, \frac{18 + 16\sqrt{\ln\ln(2T) + \ln\frac{5.2}{\delta}}}{\beta}, 16 \right)\sqrt{T} %
\end{align*}

Inequality $(i)$ holds because $\sqrt{c'} \leq 2\max\left(\left( |\mathcal{S}|+1\right)\eta D, \frac{18}{\beta}, 16 \right)$. 

We conclude that:

\begin{align*}
    J_D\left(\frac{1}{T}\sum_{t=1}^T \Boldv_t\right) &\stackrel{(i)}{\leq}  \frac{1}{T} \sum_{t=1}^T J_D(\Boldv_t) \\
    &\leq J_D(\Boldv_\star) + \frac{36D}{\sqrt{T}}\max\left(\left( |\mathcal{S}|+1\right)\eta D, \frac{18 + 16\sqrt{\ln\ln(2T) + \ln\frac{5.2}{\delta}}}{\beta}, 16 \right) %
\end{align*}

Inequality $(i)$ holds by convexity of $J_D$. The result follows.
\end{proof}

We are ready to present the proof of Lemma~\ref{lemma::SGD_result_simplified} which corresponds to a simplified version of Lemma~\ref{lemma::SGD_result_appendix}.

\subsection{Proof of Lemma~\ref{lemma::SGD_result_simplified}}

\SGDresultsimplified*

\begin{proof}
First note that the $c' $ of Lemma~\ref{lemma::SGD_result_appendix} satisfies $c' = \max\left( 2\left( |\mathcal{S}|+1\right)^2 \eta^2 D^2, \frac{320}{\beta^2}, 240  \right)$ and therefore:

\begin{equation*}
    c' \leq 8\max\left(8|\mathcal{S}|^2 \eta^2 D^2, \frac{320}{\beta} \right)
\end{equation*}

Thus $\sqrt{c'} = \max( 8 |\mathcal{S}|\eta D, \frac{31}{\beta}) = 8 |\mathcal{S}|\eta D$ (the last equality holds because $\eta \geq \frac{4}{\beta}$) and therefore:
\begin{equation*}
    \xi_t = \min( \frac{8|\mathcal{S}|\eta D}{\sqrt{t}}, \beta) = \frac{8|\mathcal{S}|\eta D}{\sqrt{t}}
\end{equation*}
The last equality holds because $t \geq  \frac{64|\mathcal{S}|^2\eta^2 D^2}{\beta}$. 

Then the condition in Equation~\ref{equation::lower_bound_n_t} of Lemma~\ref{lemma::SGD_result_appendix} is satisfies whenever:

\begin{equation}\label{equation::inequatliy_n_t}
    \frac{n(t)}{\ln\ln(2n(t))} \geq \frac{120 t\times 280 \eta^2 D^2 \left( \ln \frac{100 |\mathcal{S}||\mathcal{A}|t^2}{\delta} + 1\right)}{\beta 64 |\mathcal{S}|^2 \eta^2 D^2 } = \frac{525 t  \left( \ln \frac{100 |\mathcal{S}||\mathcal{A}|t^2}{\delta} + 1\right)}{\beta  |\mathcal{S}|^2  }
\end{equation}

And therefore if we set $n(t) = \frac{525 t  \left( \ln \frac{100 |\mathcal{S}||\mathcal{A}|t^2}{\delta}+ 1 \right)^3  }{\beta |\mathcal{S}|^2} \geq \frac{525 t \ln\ln(2t) \left( \ln \frac{100 |\mathcal{S}||\mathcal{A}|t^2}{\delta}+ 1 \right)  }{\beta |\mathcal{S}|^2} \ln( \frac{2t^2}{\delta}) $ we see that with probability at least $1-3\delta$ and simultaneously for all $t \in \mathbb{N}$:

 \begin{align*}
 J_D\left(\frac{1}{t}\sum_{\ell=1}^t \Boldv_\ell\right) &\leq J_D(\Boldv_\star) + \frac{36D}{\sqrt{t}}\max\left(\left( |\mathcal{S}|+1\right)\eta D, \frac{18 + 16\sqrt{\ln\ln(2t) + \ln\frac{5.2}{\delta}}}{\beta}, 16 \right) \\ %
 &= J_D(\Boldv_\star) + \frac{72D^2|\mathcal{S}|\eta}{\sqrt{t}}\left( 5 + 4\sqrt{\ln\ln(2t) + \ln \frac{5.2}{\delta} } \right)
 \end{align*}

The last inequality holds since $\eta \geq \frac{4}{\beta} $. This implies that using a budget of $n(t)$ samples where $n(t)$ satisfies Inequality~\ref{equation::inequatliy_n_t} we can take $t$ gradient steps. 

\end{proof}

\section{Extended Results for Tsallis Entropy Regularizers}\label{section::extended_results_tsallis}

For $\alpha > 1$ recall the Tsallis entropy between distributions $\Boldq, \BoldLambda$ equals:
\begin{align*}
     D_\alpha^{\mathcal{T}}(\BoldLambda\parallel \Boldq) &= \frac{1}{\alpha-1}\left(\mathbb{E}_{(s,a) \sim \Boldq} \left[\left( \frac{\BoldLambda_{s,a}}{\Boldq_{s,a}}\right)^\alpha  -1\right]\right) \\
     &= \frac{1}{\alpha-1}\left(\mathbb{E}_{(s,a) \sim \BoldLambda} \left[\left( \frac{\BoldLambda_{s,a}}{\Boldq_{s,a}}\right)^{\alpha-1}  -1\right]\right)
\end{align*}

Let $F(\BoldLambda) = \frac{1}{\eta}  D_\alpha^{\mathcal{T}}(\BoldLambda\parallel \Boldq) $. The Fenchel Dual of a Tsallis Entropy satisfies:

\begin{equation*}
    F^*(\Boldu) = \left \langle \BoldLambda(\Boldu), \Boldu - \frac{\left(\Boldu + x_* \mathbf{1} \right)}{\alpha} \BoldLambda(\Boldu)^{\alpha-1} + \frac{1}{\eta(\alpha-1)}\mathbf{1}   \right \rangle 
\end{equation*}
Where $\BoldLambda(\Boldu) = (\eta\Boldu + \eta x_* \mathbf{1})^{1/(\alpha-1)} \left( \frac{\alpha-1}{\alpha}\right)^{1/(\alpha-1)}\Boldq$ and where $x_* \in \mathbb{R}$ such that $\sum_{s,a} \BoldLambda_{s,a}(\Boldu) = 1$ and $\BoldLambda_{s,a}(\Boldu) \geq 0$ for all $s,a \in \mathcal{S} \times \mathcal{A}$.  This implies that:

\begin{equation*}
    J_D^{\mathcal{T}, \alpha}(\Boldv) = (1-\gamma)\sum_s \Boldv_s \BoldMu_s + \left \langle \BoldLambda(\mathbf{A}^\Boldv), \mathbf{A}^\Boldv - \frac{\left(\mathbf{A}^\Boldv + x_* \mathbf{1} \right)}{\alpha} \BoldLambda(\mathbf{A}^\Boldv)^{\alpha-1} + \frac{1}{\eta(\alpha-1)}\mathbf{1}   \right \rangle  
\end{equation*}

\subsubsection{Strong Convexity of Tsallis Entropy}

In this section we show that whenever $\alpha \in (1,2]$, the Tsallis entropy is a strongly convex function of $\BoldLambda$ in the $\| \cdot \|_2$ norm,

\begin{lemma}
If $\alpha \in (1,2]$, the function $F(\BoldLambda) = \frac{1}{\eta} D_{\alpha}^\mathcal{T}(\BoldLambda \parallel \Boldq)$ is $\frac{\alpha}{\eta}$-strongly convex in the $\| \cdot \|_2$ norm. 
\end{lemma}
\begin{proof}
It is easy to see that $\nabla^2_{\BoldLambda} D_\alpha^{\mathcal{T}}(\BoldLambda \parallel \Boldq)$ is a diagonal matrix satisfying:
\begin{equation*}
   \left[ \nabla^2_{\BoldLambda} D_\alpha^{\mathcal{T}}(\BoldLambda \parallel \Boldq) \right]_{s,a}= \frac{\alpha \BoldLambda_{s,a}^{\alpha-2}}{\eta \Boldq_{s,a}^{\alpha-1}}. 
\end{equation*}

Whenever $\alpha \leq 2$, and noting that $\Boldq \in [0,1]$ we conclude that any of these terms must be lower bounded by $\frac{\alpha}{\eta}$. The result follows.

\end{proof}

\subsection{Tsallis entropy version of Lemma~\ref{lemma::bounding_primal_value_candidate_solution}}

\begin{lemma}\label{lemma::bounding_primal_value_candidate_solution_tsallis}
Let $\tilde{\Boldv}\in\mathbb{R}^{|\mathcal{S}|}$ be arbitrary and let $\tilde{\BoldLambda}$ be its corresponding candidate primal variable (i.e. $\tilde{\BoldLambda} = \BoldLambda( \mathbf{A}^\Boldv)$). If $\| \nabla_{\Boldv} J_D(\tilde{\Boldv}) \|_1 \leq \epsilon$ and Assumptions~\ref{ass:uniform} and~\ref{assumption::lower_bound_q} hold then whenever $|\mathcal{S}| \geq 2$:
\begin{equation*}
    J^{\mathcal{T}, \alpha}_P(\BoldLambda^{\tilde{\pi}}) \geq J^{\mathcal{T}, \alpha}_P(\BoldLambda_{\eta}^*) - \epsilon \left( \frac{1+c}{1-\gamma} + \| \tilde{\Boldv} \|_\infty \right) 
\end{equation*}

Where $c = \frac{1}{\eta (\alpha-1)} \frac{1}{\beta^{\alpha-1}} \left( \max(\alpha-1, \frac{2}{\rho^{\alpha-1}} ) + 2  \right)$ and $\BoldLambda_\eta^\star$ is the $J_P$ optimum.
 \end{lemma}

\begin{proof}
For any $\BoldLambda$ and $\Boldv$ let the lagrangian $J_L(\BoldLambda, \Boldv)$ be defined as,
\begin{equation*}
J_L( \BoldLambda, \Boldv) = (1-\gamma) \langle \BoldMu, \Boldv \rangle + \left\langle \BoldLambda, \mathbf{A}^{\Boldv} - \frac{1}{\eta(\alpha-1)}\left( \left(\frac{\BoldLambda}{\Boldq}\right)^{\alpha-1} - 1 \right)  \right\rangle
\end{equation*}

Note that $J_D(\widetilde{\Boldv}) = J_L(\widetilde{\BoldLambda}, \widetilde{\Boldv})$ and that in fact $J_L$ is linear in $\bar{\Boldv}$; \ie, 
$$J_L(\widetilde{\BoldLambda}, \bar{\Boldv}) = J_L(\widetilde{\BoldLambda}, \widetilde{\Boldv}) + \langle \nabla_\Boldv J_L(\widetilde{\BoldLambda}, \widetilde{\Boldv}),\bar{\Boldv} - \widetilde{\Boldv}     \rangle.$$

Using Holder's inequality we have:
\begin{equation*}
J_L(\widetilde{\BoldLambda}, \bar{\Boldv} ) \geq J_L(\widetilde{\BoldLambda}, \widetilde{\Boldv}) - \|\nabla_\Boldv J_L(\widetilde{\BoldLambda}, \widetilde{\Boldv})\|_1\cdot \| \bar{\Boldv} - \widetilde{\Boldv}  \|_\infty = J_D(\widetilde{\Boldv}) - \|\nabla_\Boldv J_L(\widetilde{\BoldLambda}, \widetilde{\Boldv})\|_1\cdot \| \bar{\Boldv} - \widetilde{\Boldv}  \|_\infty. 
\end{equation*}

Let $\BoldLambda_\star$ be the candidate primal solution to the optimal dual solution $\Boldv_\star = \argmin_{\Boldv} J_D(\Boldv)$. By weak duality we have that $J_D(\widetilde{\Boldv}) \geq J_P(\BoldLambda^\star) = J_D(\Boldv_\star)$, and since by assumption $\|\nabla_\Boldv J_L(\widetilde{\BoldLambda}, \widetilde{\Boldv})\|_1 \leq \epsilon$:
\begin{equation}\label{equation::lowr_bounding_J_lambda_star_tsallis}
    J_L(\widetilde{\BoldLambda}, \bar{\Boldv} ) \geq J_P(\BoldLambda^\star) -\epsilon \| \bar{\Boldv} - \widetilde{\Boldv}  \|_\infty. 
\end{equation}
In order to use this inequality to lower bound the value of $J_P(\BoldLambda^{\widetilde{\pi}})$, we will need to choose an appropriate $\bar{\Boldv}$ such that the LHS reduces to $J_P(\BoldLambda^{\widetilde{\pi}})$ while keeping the $\ell_\infty$ norm on the RHS small. Thus we consider setting $\bar{\Boldv}$ as:
\begin{equation*}
    \bar{\Boldv}_s = \mathbb{E}_{a,s' \sim \widetilde{\pi} \times \Trans }\left[ \Boldz_s + \Boldr_{s,a}        - \frac{1}{\eta(\alpha-1)}\left( \left(\frac{\BoldLambda^{\widetilde{\pi}}_{s,a}}{\Boldq_{s,a}}\right)^{\alpha-1} - 1 \right) + \gamma \bar{\Boldv}_{s'} \right]
\end{equation*}
Where $\Boldz \in \mathbb{R}^{|S|}$ is some function to be determined later. It is clear that an appropriate $\Boldz$ exists as long as $\Boldz, \Boldr, \frac{1}{\eta(\alpha-1)}\left( \left(\frac{\BoldLambda^{\widetilde{\pi}}_{s,a}}{\Boldq_{s,a}}\right) - 1 \right)^{\alpha-1}$ are uniformly bounded. Furthermore:
\begin{equation}\label{equation::z_infinity_bound_tsallis}
    \| \bar{\Boldv} \|_\infty \leq \frac{\max_{s,a} \left| \Boldz_s + \Boldr_{s,a} - \frac{1}{\eta(\alpha-1)}\left( \left(\frac{\BoldLambda_{s,a}^{\widetilde{\pi}}}{\Boldq_{s,a}} \right)^{\alpha-1} - 1 \right) \right|}{1-\gamma} \leq \frac{\| \Boldz \|_\infty + \| \Boldr \|_{\infty}  + \frac{1}{\eta( \alpha-1)}\left\| \left(  \frac{\BoldLambda_{s,a}^{\widetilde{\pi}}}{\Boldq_{s,a}}   \right)^{\alpha-1} - 1\right\|_{\infty}  }{1-\gamma}
\end{equation}

We proceed to bound the norm of $\left\| \left(  \frac{\BoldLambda_{s,a}^{\widetilde{\pi}}}{\Boldq_{s,a}}   \right)^{\alpha-1} - 1\right\|_{\infty}$. Observe that by Assumptions~\ref{assumption::lower_bound_q} and~\ref{ass:uniform}, for all states $s,a \in \mathcal{S}\times \mathcal{A}$, the ratio $|\frac{\BoldLambda_{s,a}^{\widetilde{\pi}}}{\Boldq_{s,a}}  | \leq  \frac{1}{\beta}$ and therefore:
\begin{equation*}
    \left\| \left(  \frac{\BoldLambda_{s,a}^{\widetilde{\pi}}}{\Boldq_{s,a}}   \right)^{\alpha-1} - 1\right\|_{\infty} \leq 1 + \frac{1}{\beta^{\alpha-1}}
\end{equation*}

Notice the following relationships hold:
\begin{small}
\begin{align}
\left\langle \widetilde{\BoldLambda}, \mathbf{A}^{\bar{\Boldv}} - \frac{1}{\eta(\alpha-1)}\left( \left(\frac{\widetilde{\BoldLambda}}{\Boldq}\right)^{\alpha-1} - 1 \right)  \right\rangle &=    \sum_{s} \widetilde{\BoldLambda}_{s} \left(\mathbb{E}_{a, s' \sim \widetilde{\pi} \times \BoldP } \left[  \Boldr_{s,a} + \gamma \bar{\Boldv}_{s'} - \bar{\Boldv}_s - \frac{1}{\eta(\alpha-1)} \left( \left( \frac{\widetilde{\BoldLambda}_{s,a}}{\Boldq_{s,a}}\right)^{\alpha-1}-1 \right)   \right] \right) \notag\\
&=    \sum_{s} \widetilde{\BoldLambda}_{s} \left(\mathbb{E}_{a, s' \sim \widetilde{\pi} \times \BoldP } \left[ \frac{1}{\eta(\alpha-1)}\left(\left( \frac{\BoldLambda_{s,a}^{\widetilde{\pi}}}{\Boldq_{s,a}}\right)^{\alpha-1} -1 \right)  - \frac{1}{\eta(\alpha-1)} \left( \left( \frac{\widetilde{\BoldLambda}_{s,a}}{\Boldq_{s,a}}\right)^{\alpha-1}-1 \right) -\Boldz_s  \right] \right)\notag\\
&=  \sum_{s} \widetilde{\BoldLambda}_{s} \left(\mathbb{E}_{a, s' \sim \widetilde{\pi} \times \BoldP } \left[ \frac{1}{\eta(\alpha-1)}\left( \frac{\BoldLambda_{s,a}^{\widetilde{\pi}}}{\Boldq_{s,a}}\right)^{\alpha-1}    - \frac{1}{\eta (\alpha-1)} \left( \frac{\widetilde{\BoldLambda}_{s,a}}{\Boldq_{s,a}}\right)^{\alpha-1}  -\Boldz_s  \right] \right)\notag\\
&= \sum_{s} \widetilde{\BoldLambda}_{s} \left( \frac{1}{\eta(\alpha-1)}\left( \left(\frac{\BoldLambda_s^{\widetilde{\pi}}}{\Boldq_s}\right)^{\alpha-1}  - \left(\frac{\widetilde{\BoldLambda}_s}{\Boldq_s} \right)^{\alpha-1}\right)\left[ \sum_a \frac{\widetilde{\pi}^\alpha(a|s)}{\Boldq_{a|s}^{\alpha-1}}\right]   - \Boldz_s\right)\label{equation::support_equation_1_tsallis}  
\end{align}
\end{small}

Where $\widetilde{\BoldLambda}_{s} = \sum_{a} \widetilde{\BoldLambda}_{s,a}$ and $\BoldLambda^{\widetilde{\pi}}_s = \sum_a \BoldLambda^{\widetilde{\pi}}_{s,a}$. Note that by definition:

\begin{equation}
    (1-\gamma) \langle \BoldMu, \bar{\Boldv} \rangle = \left\langle \BoldLambda^{\widetilde{\pi}}, \Boldz + \Boldr - \frac{1}{\eta(\alpha-1)} \left( \left( \frac{\BoldLambda^{\widetilde{\pi}}}{\Boldq}\right)^{\alpha-1}-1 \right)  \right\rangle = J_P(\BoldLambda^{\widetilde{\pi}}) + \langle \BoldLambda^{\widetilde{\pi}}, \Boldz \rangle. \label{equation::support_equation_2_tsallis}
\end{equation}

Let's expand the definition of $J_L( \widetilde{\BoldLambda}, \bar{\Boldv} )$ using Equations~\ref{equation::support_equation_1} and \ref{equation::support_equation_2}:
\begin{align*}
    J_L( \widetilde{\BoldLambda}, \bar{\Boldv} ) &= (1-\gamma) \langle \BoldMu, \bar{\Boldv} \rangle + \left\langle \widetilde{\BoldLambda}, \mathbf{A}^{\bar{\Boldv}}- \frac{1}{\eta(\alpha-1)}\left( \left(\frac{\widetilde{\BoldLambda}}{\Boldq}\right)^{\alpha-1} - 1 \right)  \right\rangle\\
    &= J_P(\BoldLambda^{\widetilde{\pi}}) + \langle \BoldLambda^{\widetilde{\pi}}, \Boldz \rangle +  \sum_{s} \widetilde{\BoldLambda}_{s} \left( \frac{1}{\eta(\alpha-1)}\left( \left(\frac{\BoldLambda_s^{\widetilde{\pi}}}{\Boldq_s}\right)^{\alpha-1}  - \left(\frac{\widetilde{\BoldLambda}_s}{\Boldq_s} \right)^{\alpha-1}\right)\left[ \sum_a \frac{\widetilde{\pi}^\alpha(a|s)}{\Boldq_{a|s}^{\alpha-1}}\right]   - \Boldz_s\right)\\
    &= J_P(\BoldLambda^{\widetilde{\pi}})  + \sum_s \left( \Boldz_s(\BoldLambda^{\widetilde{\pi}}_s - \widetilde{\BoldLambda}_s ) +  \frac{\widetilde{\BoldLambda}_s}{\eta(\alpha-1)}\left( \left(\frac{\BoldLambda_s^{\widetilde{\pi}}}{\Boldq_s}\right)^{\alpha-1}  - \left(\frac{\widetilde{\BoldLambda}_s}{\Boldq_s} \right)^{\alpha-1}\right)\left[ \sum_a \frac{\widetilde{\pi}^\alpha(a|s)}{\Boldq_{a|s}^{\alpha-1}}\right] \right)
\end{align*}

Since we want this expression to equal $J_P( \BoldLambda^{\widetilde{\pi}} )$, we need to choose $\Boldz$ such that:
\begin{align*}
    \Boldz_s &= \frac{\frac{1}{\eta(\alpha-1)}\left( \left(\frac{\BoldLambda_s^{\widetilde{\pi}}}{\Boldq_s}\right)^{\alpha-1}  - \left(\frac{\widetilde{\BoldLambda}_s}{\Boldq_s} \right)^{\alpha-1}\right)\left[ \sum_a \frac{\widetilde{\pi}^\alpha(a|s)}{\Boldq_{a|s}^{\alpha-1}}\right]}{1-\frac{\BoldLambda_s^{\widetilde{\pi}}}{\widetilde{\BoldLambda}_s}}
    \end{align*}

Observe that $ \Boldz_s = \frac{\frac{1}{\eta(\alpha-1)}\left( \left({\BoldLambda_s^{\widetilde{\pi}}}\right)^{\alpha-1}  - \left(\widetilde{\BoldLambda}_s \right)^{\alpha-1}\right)\left[ \sum_a \frac{\widetilde{\pi}^\alpha(a|s)}{\Boldq_{s,a}^{\alpha-1}}\right]}{1-\frac{\BoldLambda_s^{\widetilde{\pi}}}{\widetilde{\BoldLambda}_s}}$ and therefore, since for all $s$ and when $\alpha \geq 1$ by Assumption~\ref{assumption::lower_bound_q} we have that $\sum_a  \frac{\widetilde{\pi}^\alpha(a|s)}{\Boldq_{s,a}^{\alpha-1}} \leq \frac{1}{\beta^{\alpha-1}}$,
\begin{align*}
   \left|  \Boldz_s \right| \leq \frac{1}{\eta (\alpha-1)} \frac{1}{\beta^{\alpha-1}} \frac{\left| \left(\BoldLambda_s^{\widetilde{\pi}}\right)^{\alpha-1} - \widetilde{\BoldLambda}_s^{\alpha-1} \right| }{\left|1-\frac{\BoldLambda_s^{\widetilde{\pi}}}{\widetilde{\BoldLambda}_s}\right|}
\end{align*}

Let $\frac{\BoldLambda_s^{\widetilde{\pi}}}{\widetilde{\BoldLambda}_s} = \frac{1}{\theta}$ where $\theta \in [0, \frac{1}{\rho}]$. Then,
\begin{equation*}
\left| \Boldz_s \right| \leq \frac{1}{\eta (\alpha-1)\beta^{\alpha-1}} \BoldLambda^{\widetilde{\pi}}_s \frac{|1-\theta^{\alpha-1} |}{\left|1 - \frac{1}{\theta}\right|}
 \end{equation*}

It is easy to see that when $\alpha \geq 0$ the function $f(\theta) = \frac{1-\theta^{\alpha-1} }{1 - \frac{1}{\theta}} = \frac{\theta - \theta^{\alpha}}{\theta-1}$ is decreasing in the interval $(0, 1]$ and increasing afterwards. Furthermore, by L'Hopital's rule, $f(1) = 1-\alpha$ and $f(\frac{1}{\rho}) = \frac{\frac{1}{\rho^{\alpha}} - \frac{1}{\rho}}{\frac{1}{\rho}-1} \leq \frac{2}{\rho^{\alpha-1}}$ since $\rho \leq \frac{1}{2}$. This implies,
\begin{equation*}
  \left|  \Boldz_s \right| \leq \frac{1}{\eta (\alpha-1)} \frac{1}{\beta^{\alpha-1}} \max(\alpha-1, \frac{2}{\rho^{\alpha-1}}).
\end{equation*}

And therefore Equation~\ref{equation::z_infinity_bound_tsallis} implies:

\begin{equation*}
    \|\bar{\Boldv}\|_\infty \leq \frac{ \frac{1}{\eta (\alpha-1)} \frac{1}{\beta^{\alpha-1}} \max(\alpha-1, \frac{2}{\rho^{\alpha-1}}) + 1  + \frac{1}{\eta( \alpha-1)} \left(\frac{1}{\beta^{\alpha-1}}  +1\right)   }{1-\gamma} = \frac{ \frac{1}{\eta (\alpha-1)} \frac{1}{\beta^{\alpha-1}} \left( \max(\alpha-1, \frac{2}{\rho^{\alpha-1}} ) + 2  \right) +1     }{1-\gamma}
\end{equation*}
Putting these together we obtain the following version of equation~\ref{equation::lowr_bounding_J_lambda_star_tsallis}:
\begin{equation*}
      J_L(\widetilde{\BoldLambda}, \bar{\Boldv} ) \geq J_P(\BoldLambda^\star) -\epsilon \left( \frac{ \frac{1}{\eta (\alpha-1)} \frac{1}{\beta^{\alpha-1}} \left( \max(\alpha-1, \frac{2}{\rho^{\alpha-1}} ) + 2  \right) +1     }{1-\gamma} + \| \widetilde{\Boldv } \|_\infty   \right)
\end{equation*}

\end{proof}

\subsection{Extension of Lemma~\ref{lemma::dual_variables_bound_ofir} to Tsallis Entropy}

\begin{restatable}{lemma}{dualvariablesboundofirtsallis}
\label{lemma::dual_variables_bound_ofir_tsallis}
Under Assumptions~\ref{assumption::bounded_rewards},  \ref{assumption::lower_bound_q} and~\ref{ass:uniform}, the optimal dual variables are bounded as
\begin{equation}\label{equation::max_radius_tsallis}
\| \Boldv^*  \|_\infty \le \frac{1}{1-\gamma} \left( 1 +  \frac{2}{\eta(\alpha-1) \beta^{\alpha-1}} \right) = D_{\mathcal{D}, \alpha}.
\end{equation}
\end{restatable}

\begin{proof}
Recall the Lagrangian form,
\begin{equation*}
    \min_{\Boldv}, \max_{\BoldLambda_{s,a} \in \Delta_{S\times A}}~ J_L(\BoldLambda, \Boldv) :=  (1-\gamma) \langle \Boldv, \BoldMu \rangle + \left\langle \BoldLambda, \mathbf{A}^{\Boldv} - \frac{1}{\eta(\alpha-1)}\left( \left(\frac{\BoldLambda_{s,a}}{\Boldq_{s,a}}\right)^{\alpha-1} - 1 \right)\right\rangle.     
\end{equation*}
The KKT conditions of $\BoldLambda^*,\Boldv^*$ imply that for any $s,a$, either (1) $\BoldLambda^*_{s,a} = 0$ and $\frac{\partial}{\partial\BoldLambda_{s,a}}J_L(\BoldLambda^*,v^*) \le 0$ or (2) $\frac{\partial}{\partial\BoldLambda_{s,a}}J_L(\BoldLambda^*,\Boldv^*) = 0$. The partial derivative of $J_L$ is given by,
\begin{equation}
    \frac{\partial}{\partial\BoldLambda_{s,a}}J_L(\BoldLambda^*,\Boldv^*) = \Boldr_{s,a}  + \gamma\sum_{s'} P_{a}(s'|s) \Boldv^*_{s'} - \Boldv^*_{s}-\frac{\alpha}{\eta(\alpha-1)}\left(  \frac{\BoldLambda^*_{s,a} }{\Boldq_{s,a}}\right)^{\alpha-1} + \frac{1}{\eta (\alpha-1)} . 
\end{equation}
Thus, for any $s,a$, either
\begin{equation}
    \BoldLambda^*_{s,a} = 0 ~\text{and}~ \Boldv^*_{s} \ge \Boldr_{s,a} -\frac{\alpha}{\eta(\alpha-1)}\left(  \frac{\BoldLambda^*_{s,a} }{\Boldq_{s,a}}\right)^{\alpha-1} + \frac{1}{\eta (\alpha-1)}  + \gamma\sum_{s'} P_{a}(s'|s) \Boldv^*_{s'}, 
\end{equation}
or,
\begin{equation}
    \BoldLambda^*_{s,a} > 0 ~\text{and}~ \Boldv^*_{s} = \Boldr_{s,a} -\frac{\alpha}{\eta(\alpha-1)}\left(  \frac{\BoldLambda^*_{s,a} }{\Boldq_{s,a}}\right)^{\alpha-1} + \frac{1}{\eta (\alpha-1)} + \gamma\sum_{s'} P_{a}(s'|s) \Boldv^*_{s'}.
\end{equation}
Recall that $\BoldLambda^*$ is the discounted state-action visitations of some policy $\pi_\star$; \ie, $\BoldLambda^*_{s,a} = \BoldLambda^{\pi_\star}_s \cdot \pi_\star(a|s)$ for some $\pi_\star$. Note that by Assumption~\ref{ass:uniform}, any policy $\pi$ has $\BoldLambda^{\pi_\star}_{s} > 0$ for all $s$. Accordingly, the KKT conditions imply,
\begin{equation}
    \pi_\star(a|s) = 0 ~\text{and}~ \Boldv^*_{s} \ge \Boldr_{s,a} -\frac{\alpha}{\eta(\alpha-1)}\left(  \frac{\BoldLambda^*_{s,a} }{\Boldq_{s,a}}\right)^{\alpha-1} + \frac{1}{\eta (\alpha-1)}  + \gamma\sum_{s'} P_{a}(s'|s) \Boldv^*_{s'}, 
\end{equation}
or,
\begin{equation}
    \pi_\star(a|s) > 0 ~\text{and}~ \Boldv^*_{s} = \Boldr_{s,a} -\frac{\alpha}{\eta(\alpha-1)}\left(  \frac{\BoldLambda^*_{s,a} }{\Boldq_{s,a}}\right)^{\alpha-1} + \frac{1}{\eta (\alpha-1)}  + \gamma\sum_{s'} P_{a}(s'|s) \Boldv^*_{s'}.
\end{equation}
Equivalently,
\begin{align}
    \Boldv^*_{s} &= \E_{a\sim\pi_\star(s)}\left[\Boldr_{s,a} -\frac{\alpha}{\eta(\alpha-1)}\left(  \frac{\BoldLambda^*_{s,a} }{\Boldq_{s,a}}\right)^{\alpha-1} + \frac{1}{\eta (\alpha-1)} + \gamma\sum_{s'} P_{a}(s'|s) \Boldv^*_{s'}\right] \\
\end{align}
We may express these conditions as a Bellman recurrence for $\Boldv^*_s$ %
and the solution to these Bellman equations is bounded when $\Boldr_{s,a} -\frac{\alpha}{\eta(\alpha-1)}\left(  \frac{\BoldLambda^*_{s,a} }{\Boldq_{s,a}}\right)^{\alpha-1} + \frac{1}{\eta (\alpha-1)}$ is bounded~\citep{puterman2014markov}.  And indeed, by Assumptions~\ref{assumption::lower_bound_q} and~\ref{assumption::bounded_rewards}, $\left|\Boldr_{s,a} -\frac{\alpha}{\eta(\alpha-1)}\left(  \frac{\BoldLambda^*_{s,a} }{\Boldq_{s,a}}\right)^{\alpha-1} + \frac{1}{\eta (\alpha-1)}\right| \leq 1 + \frac{1}{\eta (\alpha-1)}  + \frac{1}{\eta(\alpha-1) \beta^{\alpha-1}}$ 
We may thus bound the solution as,
\begin{equation}
    \|\Boldv^*\|_\infty \le \frac{1}{1-\gamma} \left( 1 +  \frac{2}{\eta(\alpha-1) \beta^{\alpha-1}} \right).
\end{equation}
\end{proof}

\subsection{Gradient descent results for the Tsallis Entropy}

\begin{remark}
Throughout this section we make the assumption that $\alpha \in (1,2]$.
\end{remark}%

We start by characterizing the smoothness properties of $J_D^{\mathcal{T}, \alpha}(\Boldv)$, the dual function of the Tsallis regularized LP. 

\begin{lemma}
If $\alpha \in (1, 2]$ the dual function $J^{\mathcal{T}, \alpha}_D(\Boldv)$ is $\frac{\eta |\CalS||\CalA|}{\alpha}$-smooth in the $\| \cdot \|_2$ norm. 
\end{lemma}

\begin{proof}
Recall that \ref{equation::primal_visitation_regularized} can be written as \ref{equation::regularized_LP}:
\begin{align*}
    \max_{\BoldLambda \in \mathcal{D}} \langle \Boldr, \BoldLambda \rangle - F(\BoldLambda)  \\
    \text{s.t. } \mathbf{E}\BoldLambda = b. \notag
\end{align*}

Where the regularizer ($   F(\BoldLambda) := \frac{1}{\eta}D_\alpha^{\mathcal{T}}(\BoldLambda\parallel \Boldq) $) is $\frac{\alpha}{\eta}-\| \cdot \|_2$ strongly convex. In this problem $\Boldr$ corresponds to the reward vector, the vector $\Boldb = (1-\gamma) \BoldMu \in \mathbb{R}^{|\mathcal{S}|}$ and matrix $\mathbf{E} \in \mathbb{R}^{|\mathcal{S}| \times |\mathcal{S}|\times |\mathcal{A}|}$ takes the form:
\begin{equation*}
    \mathbf{E}[s, s',a] = \begin{cases}
                \gamma \mathbf{P}_a(s | s')     & \text{if } s \neq s'\\
                  1-\gamma \mathbf{P}_{a}(s|s)  & \text{o.w.}
                \end{cases}
\end{equation*}
Therefore (since  $\| \mathbf{E} \|_{2, 2}$ is simply the Frobenius norm of matrix $\mathbf{E}$),
\begin{equation*}
    \| \mathbf{E} \|_{2, 2} \leq 2|\CalS|| \CalA| 
\end{equation*}

The result follows as a corollary of Lemma~\ref{lemma::equivalence_smoothness_strong_convexity}. 
\end{proof}

Throughout this section we use the notation $ D_{\mathcal{T}, \alpha}$ to refer to $\| \Boldv^*  \|_\infty \le \frac{1}{1-\gamma} \left( 1 +  \frac{2}{\eta(\alpha-1) \beta^{\alpha-1}} \right)$. We are ready to prove convergence guarantees for Algorithm~\ref{algorithm::accelerated_gradient_descent} when applied to the objective $J^{\mathcal{T}, \alpha}_D$. 

\begin{lemma}\label{lemma::grad_descent_guarantee_tsallis} Let Assumptions~\ref{assumption::bounded_rewards},~\ref{assumption::lower_bound_q} and \ref{ass:uniform} hold. Let $\mathcal{D}_{ \mathcal{T}, \alpha}= \left\{ \Boldv \text{ s.t. } \| \Boldv \|_\infty \le D_{\mathcal{T}, \alpha}\right\}$, and define the distance generating function to be $w(\Boldx) = \| \Boldx \|_2^2$. After $T$ steps of Algorithm~\ref{algorithm::accelerated_gradient_descent}, the objective function $J^{\mathcal{T}, \alpha}_D$ evaluated at the iterate $\Boldv_T = y_T$ satisfies:
\vspace{-.3cm}
\begin{equation*}
   J^{\mathcal{T}, \alpha}_D(\Boldv_T ) - J^{\mathcal{T}, \alpha}_D(\Boldv^*)\leq 4\eta \frac{|\CalS|^2|\CalA|}{\alpha}\frac{\left( 1 + c' \right)^2}{(1-\gamma)^2T^2}.
\end{equation*}
Where $c'  =\frac{2}{\eta (\alpha-1)\beta^{\alpha-1}}$.
\end{lemma}

\begin{proof}
This results follows simply by invoking the guarantees of Theorem \ref{theorem::accelerated_gradient_descent}, making use of the fact that $J^{\mathcal{T}, \alpha}_D$ is $\frac{\eta | \CalS||\CalA|}{\alpha}-$smooth as proven by Lemma~\ref{lemma::RL_smoothness_dual}, observing that as a consequence of Lemma~\ref{lemma::dual_variables_bound_ofir_tsallis}, $\Boldv^\star \in \mathcal{D}_{\mathcal{T}, \alpha}$ and using the inequality $\| \Boldx \|_2^2 \leq |\CalS|\| \Boldx\|_\infty^2$ for $\Boldx \in \mathbb{R}^{| \CalS|}$. 
\end{proof}

Lemma~\ref{lemma::grad_descent_guarantee_tsallis} can be easily turned into the following guarantee regarding the dual function value of the final iterate:%
\begin{corollary}\label{corollary::lower_bound_T_accelerated_gradient_tsallis}
Let $\epsilon > 0$. If Algorithm~\ref{algorithm::accelerated_gradient_descent} is ran for at least $T$ rounds
\begin{equation*}
    T \geq 2\eta^{1/2} (|\mathcal{S}||\CalA|^{1/2})\frac{\left( 1 + c' \right)}{\alpha^{1/2} (1-\gamma) \sqrt{\epsilon}}
\end{equation*}
then $\Boldv_{T}$ is an $\epsilon-$optimal solution for the dual objective $J^{\mathcal{T}, \alpha}_D$. %
\end{corollary}

If $T$ satisfies the conditions of Corollary~\ref{corollary::lower_bound_T_accelerated_gradient_tsallis} a simple use of Lemma~\ref{lemma::bounding_gradients} allows us to bound the $\| \cdot \|_2$ norm of the dual function's gradient at $\Boldv_{T}$:
\begin{equation*}
    \| \nabla J_D(\Boldv_{T}) \|_2 \leq \sqrt{  \frac{2|\CalS||\CalA|\eta \epsilon}{\alpha}}
\end{equation*}

If we denote as $\pi_{T}$  to be the policy induced by $\BoldLambda^{\Boldv_{T}}$, and $\BoldLambda_{\eta}^\star$ is the candidate dual solution corresponding to $\Boldv^\star$. A simple application of Lemma~\ref{lemma::bounding_primal_value_candidate_solution_tsallis} yields:
\begin{equation*}
    J_P( \BoldLambda^{\pi_{T}} ) \geq J_P(\BoldLambda_\eta^\star) -\frac{1}{1-\gamma} \left( 2+c+c'  \right)\sqrt{  \frac{2|\CalS||\CalA|\eta \epsilon}{\alpha}}
\end{equation*}

Where $c = \frac{1}{\eta (\alpha-1)} \frac{1}{\beta^{\alpha-1}} \left( \max(\alpha-1, \frac{2}{\rho^{\alpha-1}} ) + 2  \right)$, $c' =\frac{2}{\eta (\alpha-1)\beta^{\alpha-1}}$ and $\BoldLambda_\eta^\star$ is the $J_P$ optimum.

This leads us to the main result of this section:
\begin{corollary}\label{corollary::regularized_result_tsallis}
Let $\alpha \in (1, d]$. For any $\xi  > 0$. If $T \geq 4\eta|\CalS|^{3/2}|\CalA|^{1/2} \frac{\left( 2 + c+ c' \right)^2 }{\alpha (1-\gamma)^2 \xi} $ then:
\begin{equation*}
     J_P( \BoldLambda^{\pi_{T}} )\geq J_P(\BoldLambda^{\star}_\eta)  -\xi.
\end{equation*}
\end{corollary}

Thus Algorithm~\ref{algorithm::accelerated_gradient_descent} achieves an $\mathcal{O}(1/(1-\gamma)^2\epsilon)$ rate of convergence to an $\epsilon-$optimal regularized policy. We now proceed to show that an appropriate choice for $\eta$ can be leveraged to obtain an $\epsilon-$optimal policy.%

\begin{restatable}{theorem}{mainacceleratedresulttsallis}\label{theorem::main_accelerated_result_tsallis}
For any $\epsilon > 0$, let $\eta = \frac{2}{(\alpha-1)\epsilon\beta^\alpha}$. If $T \geq 8|\CalS|^{3/2}|\CalA|^{1/2} \frac{\left( 2 + c+ c' \right)^2 }{(\alpha-1)\alpha (1-\gamma)^2 \beta^{\alpha}\epsilon^2} $, then $\pi_T$ is an $\epsilon-$optimal policy. 
\end{restatable}

\begin{proof}
As a consequence of Corollary~\ref{corollary::regularized_result_tsallis}, we can conclude that:
\begin{equation*}
    J_P(\BoldLambda^{\pi_T}) \geq J_P(\BoldLambda^{\star, \eta}) - \frac{\epsilon}{2}.
\end{equation*}
Where $\BoldLambda_\eta^\star$ is the regularized optimum. Recall that:
\begin{equation*}
    J_P(\BoldLambda) = \sum_{s,a} \BoldLambda_{s,a} \Boldr_{s,a} -  \frac{1}{(\alpha-1)\eta}\left(\mathbb{E}_{(s,a) \sim \Boldq} \left[\left( \frac{\BoldLambda_{s,a}}{\Boldq_{s,a}}\right)^\alpha  -1\right]\right). 
\end{equation*}
Since $\BoldLambda^{\star, \eta}$ is the maximizer of the regularized objective, it satisfies $J_P(\BoldLambda^{\star, \eta}) \geq J_P(\BoldLambda^*)$ where $\BoldLambda^\star$ is the visitation frequency of the optimal policy corresponding to the unregularized objective. We can conclude that:
\begin{align*}
    \sum_{s,a} \BoldLambda_{s,a}^{\pi_T} \Boldr_{s,a} &\geq \sum_{s,a} \BoldLambda^\star_{s,a} \Boldr_{s,a} + \frac{1}{(\alpha-1)\eta}\left(  \sum_{s,a} \Boldq_{s,a} \left( \left( \frac{ \BoldLambda^{\pi_T}_{s,a}}{\Boldq_{s,a}}\right)^{\alpha}  - 1\right) -  \sum_{s,a} \Boldq_{s,a} \left( \left( \frac{ \BoldLambda^\star_{s,a}}{\Boldq_{s,a}}\right)^{\alpha}  - 1\right) \right) -\frac{\epsilon}{2} \\
    &= \sum_{s,a} \BoldLambda^\star_{s,a} \Boldr_{s,a} + \frac{1}{(\alpha-1)\eta}\left(  \sum_{s,a} \Boldq_{s,a}  \left( \frac{ \BoldLambda^{\pi_T}_{s,a}}{\Boldq_{s,a}}\right)^{\alpha}   -  \sum_{s,a} \Boldq_{s,a}  \left( \frac{ \BoldLambda^\star_{s,a}}{\Boldq_{s,a}}\right)^{\alpha}  \right)  - \frac{\epsilon}{2}\\
    &\geq  \sum_{s,a} \BoldLambda^\star_{s,a} \Boldr_{s,a} - \frac{1}{(\alpha-1)\eta}\left(\frac{1}{\beta} \right)^{\alpha} - \frac{\epsilon}{2}
\end{align*}
And therefore if $\eta = \frac{2}{(\alpha-1)\epsilon\beta^\alpha}$, we can conclude that:
\begin{equation*}
    \sum_{s,a} \BoldLambda_{s,a}^{\pi_T} \Boldr_{s,a} \geq \sum_{s,a} \BoldLambda_{s,a}^{\star} \Boldr_{s,a} - \epsilon.
\end{equation*}

\end{proof}

\end{document}